\documentclass{article}

\PassOptionsToPackage{numbers, compress}{natbib}
\usepackage[main,final]{neurips_2026}
\usepackage{amsmath,amssymb,amsthm}
\usepackage{mathtools}
\usepackage{tikz}
\usetikzlibrary{positioning}
\usepackage{booktabs}
\usepackage{graphicx}
\usepackage{subcaption} 
\usepackage{amsthm}
\usepackage{xcolor}
\usepackage{amssymb}
\usepackage{booktabs}
 \usepackage{makecell}
\usepackage{wrapfig}
\usepackage{enumitem}

\theoremstyle{plain}
\newtheorem{theorem}{Theorem}
\newtheorem{lemma}[theorem]{Lemma}
\newtheorem{corollary}[theorem]{Corollary}
\newtheorem{proposition}[theorem]{Proposition}

\theoremstyle{definition}
\newtheorem{definition}[theorem]{Definition}

\theoremstyle{remark}
\newtheorem{remark}[theorem]{\textbf{Remark}}

\usepackage{amsmath,amsfonts,bm}

\def\eqref#1{equation~\ref{#1}}
\def\1{\bm{1}}

\DeclareMathAlphabet{\mathsfit}{\encodingdefault}{\sfdefault}{m}{sl}
\SetMathAlphabet{\mathsfit}{bold}{\encodingdefault}{\sfdefault}{bx}{n}

\usepackage{hyperref}
\usepackage{url}

\usepackage{hyperref}
\usepackage{url}

\usepackage[utf8]{inputenc} % allow utf-8 input
\usepackage[T1]{fontenc}    % use 8-bit T1 fonts
\usepackage{hyperref}       % hyperlinks
\usepackage{url}            % simple URL typesetting
\usepackage{booktabs}       % professional-quality tables
\usepackage{amsfonts}       % blackboard math symbols
\usepackage{nicefrac}       % compact symbols for 1/2, etc.
\usepackage{microtype}      % microtypography
\usepackage{xcolor}         % colors

\title{Constraint-Aware Influence Estimation in Deep Constrained Learning}

\author{%
  Xin Wang \\
  Department of Civil \& Environmental Engineering\\
  University of Washington\\
  Seattle, WA 98195, USA\\
  \texttt{xinw22@uw.edu}
  \And
  R. Tyrrell Rockafellar \\
  Department of Mathematics\\
  University of Washington\\
  Seattle, WA 98195, USA\\
  \texttt{rtr@math.washington.edu}
  \And
  Xuegang (Jeff) Ban\thanks{Corresponding author.} \\
  Department of Civil \& Environmental Engineering\\
  University of Washington\\
  Seattle, WA 98195, USA\\
  \texttt{banx@uw.edu}
}

\begin{document}

\maketitle

\begin{abstract}
 Constrained learning has been increasingly applied to various domains to ensure explicit feasibility requirements due to fairness, safety, robustness, regularization, and physics or logic constraints. Understanding how training samples influence the solution (e.g., learned parameters) of constrained learning is crucial for interpretability and robustness. The classical influence function may become unreliable in constrained settings: data perturbations can reshape both the objective and the feasible region, leading to estimates that violate feasibility. In response, we propose the Directional Influence Function (DIF), a new estimator that explicitly incorporates the constraints into influence estimation. DIF formulates the optimality conditions of constrained learning as a variational inequality (VI) and analyzes how perturbing training data affects this VI. We validate DIF in constrained linear regression and demonstrate that it recovers leave-one-out retraining results, whereas IF and penalty-based IF exhibit significant bias. We further apply DIF to fairness-constrained CNNs, where DIF accurately predicts test loss changes under data removal and aligns closely with actual retraining. Our results establish DIF as an efficient and reliable tool for data attribution in constrained learning.
\end{abstract}

\section{Introduction}
Understanding how individual training samples influence model predictions, also known as data attribution~\citep{pruthi2020estimating,feldman2020neural,lin2024diffusion}, is fundamental for interpreting model behavior, model correction, and data error debugging. Although retraining the model after removing a sample and comparing the resulting change in model solution (learned parameters) provides ground-truth data influence, this approach is computationally expensive for modern deep learning models. A widely adopted alternative is to approximate sample influence by studying how the solution responds to small data perturbations. As a representative method, the influence function (IF), originally developed in robust statistics~\citep{hampel1974influence} and later adapted to machine learning~\citep{koh2017understanding,feldman2020neural,zhang2024timeinf}, instantiates this idea by differentiating the solution with respect to (w.r.t.) small data perturbations. While these influence estimators work for unconstrained learning problems, they often fail when applied to \textit{constrained learning}. The latter refers to learning tasks constrained by domain knowledge, such as physics-informed neural networks (PINNs), or those constrained by embedding requirements, including safety, robustness, and fairness~\citep{hounie2023resilient,li2023double,garcia2015comprehensive,xia2025fairtp}.  The empirical formulation of constrained learning is given by~\citep{chamon2020probably}:
% \begin{equation}
% \label{CL}
% \begin{aligned}
% \min _\theta & \frac{1}{N_0} \sum_{i=1}^{N_0}\ell_0\left(z_i^{(0)}, \theta\right)  \\
% \text { s.t. } & \frac{1}{N_j} \sum_{i=1}^{N_j} \ell_j\left(z_i^{(j)}, \theta\right)\leq \tau_j, \quad j=1, \ldots, m.
% \end{aligned}
% \tag{CL}
% \end{equation} 
\begin{equation}
\label{CL}
\begin{aligned}
\bar{\theta} :=& \arg\min_{\theta \in \mathbb{R}^d} \quad  \frac{1}{N_0} \sum_{i=1}^{N_0} \ell_0\left(z_i^{(0)}, \theta\right) \\
\text{s.t.} \quad & \frac{1}{N_j} \sum_{i=1}^{N_j} \ell_j\left(z_i^{(j)}, \theta\right) \leq \tau_j, \quad j=1, \ldots, m,
\end{aligned}
% \tag{CL}
\end{equation}

where $\theta \in \mathbb{R}^d$ denotes the model parameters to be optimized, and $\bar{\theta}$ is the optimal solution. The function $\ell_0\left(z_i^{(0)}, \theta\right)$ represents the primary loss evaluated on the training data point $z_i^{(0)}$. Each constraint $j \in\{1, \ldots, m\}$ is represented by an auxiliary loss function $\ell_j\left(z_i^{(j)}, \theta\right)$ evaluated over a separate dataset $\left\{z_i^{(j)}\right\}_{i=1}^{N_j}$. The constant $\tau_j$ specifies the upper bound allowed for the $j$th constraint. The influence estimators aim to quantify the change in solution, denoted by $\Delta\theta$, when specific training data points participating in the objective function or constraints (e.g., $z^{(1)}_1$) are removed. Current IF-based approaches fail in constrained learning for two primary reasons. First, constrained problems~(\ref{CL}) require that the perturbed solution, $\bar{\theta}+\Delta\theta$, remains within the feasible region, while IF methods ignore the constraints, not to mention that the feasible region itself may also be altered by data perturbation. Existing IF cannot guarantee that the estimated solution change is feasible. Second, IF primarily relies on the gradient of the optimal solution $\bar{\theta}$ w.r.t. data perturbations. In constrained learning, the solution can be only directionally differentiable. The full derivatives do not necessarily exist, rendering IF estimators invalid. This issue arises because the solution updates must adhere to the feasible region when constraints are active, often requiring projections to maintain feasibility. Additionally, data removal can change the active constraint set, leading to sudden shifts in the optimal solution. See the problem~(\ref{eq:toy}) in Section \ref{sec:example} for a concrete example. 

To overcome the above limitations, this paper introduces the \textit{Directional Influence Function} (DIF), an influence function designed for constrained learning, to estimate the impact of training points in either the loss function or the constraints on the model solution. DIF estimates the change in model solution through a directional derivative approach, aiming to address the following question: 
\begin{quote}
    \textit{How does the solution of constrained learning change when data points are removed from either the objective loss function or the constraints?}
\end{quote}
The main contributions of the paper are:
(1) We formalize data attribution for constrained learning by casting the optimality conditions as a variational inequality (VI) and performing local sensitivity analysis of this VI; 
(2) DIF quantifies the effect of data perturbations on model solutions via directional derivatives, thereby addressing the non-smoothness of solution changes induced by constraints; and 
(3) DIF can be efficiently computed by solving a quadratic program (QP); when all constraints are inactive (i.e., their KKT multipliers are zero), DIF reduces to the classical IF.

The remainder of this paper is organized as follows.
Section~\ref{sec:2} formally defines our problem and illustrates the failure of IF in a constrained learning setting through a linear regression example.
Section~\ref{sec:3} introduces the proposed DIF.
Section~\ref{sec:4} analyzes the impact of data perturbations on the VI-formulated optimality conditions, introduces a QP approach to compute DIF, and presents our main theoretical results.
Section~\ref{sec:5} evaluates DIF on several learning tasks, followed by concluding remarks in Section \ref{sec:6}. Throughout this paper, we assume that all functions involved are twice continuously differentiable ($C^2$). A detailed comparison with prior work is also presented in Appendix~\ref{related work}.

\section{Problem Formulation}\label{sec:2}
Let $Z^{(0)}=\{z_i^{(0)}\}_{i=1}^{N_0}$ denote the objective dataset and 
$Z^{(j)}=\{z_i^{(j)}\}_{i=1}^{N_j}$ denote the dataset for constraint $j$. We begin by examining how the solution $\bar{\theta}$ changes when a selected subset $Z^r \subset \bigcup_{j=0}^m Z^{(j)}
$ is removed. This removal is modeled via a perturbation formulation, where each point in $Z^r$ is assigned a small weight $\varepsilon_k: \varepsilon_0$ for points contributing to the objective, and $\varepsilon_j, j \in\{1, \ldots, m\}$ for the points contributing to the $j$-th constraint. This formulation, referred to as \textit{Perturbed Constrained Learning} (PCL), is defined as follows\footnote{To simplify our discussion, this constrained learning formulation does not explicitly include equality constraints, which can be considered straightforwardly.}:
\begin{equation}
\label{perturb CL}
\begin{aligned}
\min _\theta \quad & \frac{1}{N_0} \sum_{i=1}^{N_0} \ell_0\left(z_i^{(0)}, \theta\right) + \varepsilon_0 \sum_{z_i^{(0)}\in Z^r}\ell_0\left(z_{i}^{(0)}, \theta\right) \\
\text {s.t.} \quad & \frac{1}{N_j} \sum_{i=1}^{N_j} \ell_j\left(z_i^{(j)}, \theta\right) +  \varepsilon_j \sum_{z_i^{(j)}\in Z^r} \ell_j \left(z^{(j)}_{i}, \theta\right) \leq \tau_j, j=1, \ldots, m.
\end{aligned}
% \tag{PCL}
\end{equation}
% Removing data points in $Z^r$ is equivalent to shifting $\varepsilon_k$ from $\bar{\varepsilon}_k=0$ to $\hat{\varepsilon}_k=-\frac{1}{N_k}$.

Let $\varepsilon=\left[\varepsilon_0, \ldots, \varepsilon_m\right]$ denote the perturbation vector. When $\varepsilon=\mathbf{0}$, problem~(\ref{perturb CL}) reduces to the problem~(\ref{CL}). Shifting $\varepsilon$ from $\bar{\varepsilon}=[0,0, \ldots, 0]$ to $\hat{\varepsilon}=$ $\left[-\frac{1}{N_0},-\frac{1}{N_1}, \ldots,-\frac{1}{N_m}\right]$ corresponds to removing $Z^r$ from the constrained learning~(\ref{CL}). Treat $\theta$ as a function of $\varepsilon$, denoted by $\theta(\varepsilon)$, and define the solution change as $\Delta \theta=\theta(\hat{\varepsilon})-\theta(\bar{\varepsilon})$. By definition, $\theta(\bar{\varepsilon})=\bar{\theta}$ is the solution to problem~(\ref{CL}).

\subsection{Failure of IF in Constrained Learning}\label{sec:example}
\textbf{Toy example}. We use a $\ell_2$-constrained least-squares example to demonstrate the failure of IF in constrained learning, as shown below. Here $\theta=\left[\theta_1, \theta_2\right]$ are the model parameters. 
% The given data points are:
% $\left(x_1, x_2, x_3\right)=([1,0],[1,0],[0,1]), \quad\left[y_1, y_2, y_3\right]=[1,0,1] .
% $
\begin{equation}
\label{eq:toy}
\begin{aligned}
    \min_{\theta} \quad & \frac{1}{3} (\theta^\top x_1 - y_1)^2 + \frac{1}{3} (\theta^\top x_2 - y_2)^2 + \frac{1}{3} (\theta^\top x_3 - y_3)^2 + \varepsilon \cdot (\theta^\top x_2 - y_2)^2\\
    % \quad & =\frac{1}{3} (\theta_1 - 1)^2 + \frac{1}{3} (\theta_1 - 0)^2 + \frac{1}{3} (\theta_2 - 1)^2 + \varepsilon \cdot \theta_1^2 \\
    \text{s.t.} \quad & \|\theta\|^2_2 \leq 0.5.
\end{aligned}
\end{equation}
The data are as follows:\([x_1,x_2,x_3]=[(1,0);\,(1,0);\,(0,1)],\quad [y_1,y_2,y_3]=(1,0,1/2) \).
The explicit solution of this problem is:
\begin{equation}
 \theta(\varepsilon)=\operatorname{Proj}_{\|\cdot\|^2_2 \leq 0.5}\left[\left(\frac{1}{2+3\varepsilon}, 0.5\right)\right]
\end{equation}
Removing $(x_2,y_2)$ is equivalent to shifting $\varepsilon$ from $0$ to $-\frac{1}{3}$. IF defines the impact caused by data removal using the derivative of $\theta(\varepsilon)$:
\begin{equation*}
IF(z_2)=\left.\frac{d {\theta}(\varepsilon)}{d \varepsilon}\right|_{\varepsilon=0}=-H_{\bar{\theta}}^{-1} \nabla_\theta \ell_0(x_2,y_2, \bar{\theta}),
\end{equation*}
where $z_2=(x_2,y_2)$, $H_{\bar{\theta}}=\frac{1}{3}\sum_{i=1}^3 \nabla_\theta^2 \ell_0(x_i,y_i, {\bar\theta})$ is the Hessian, and $\ell_0(x_2,y_2, {\bar\theta})=(\bar\theta^\top x_2 - y_2)^2$.
IF then estimates the change $\Delta \theta$ ignoring the constraint by:
\begin{equation*}   
\theta(-\frac{1}{3})-\theta(0)\approx-\frac{1}{3}\left.\frac{d {\theta}(\varepsilon)}{d \varepsilon}\right|_{\varepsilon=0}.\end{equation*}

\vspace{-10pt}

\begin{figure}[!htbp]
    \centering
    \includegraphics[width=0.8\linewidth]{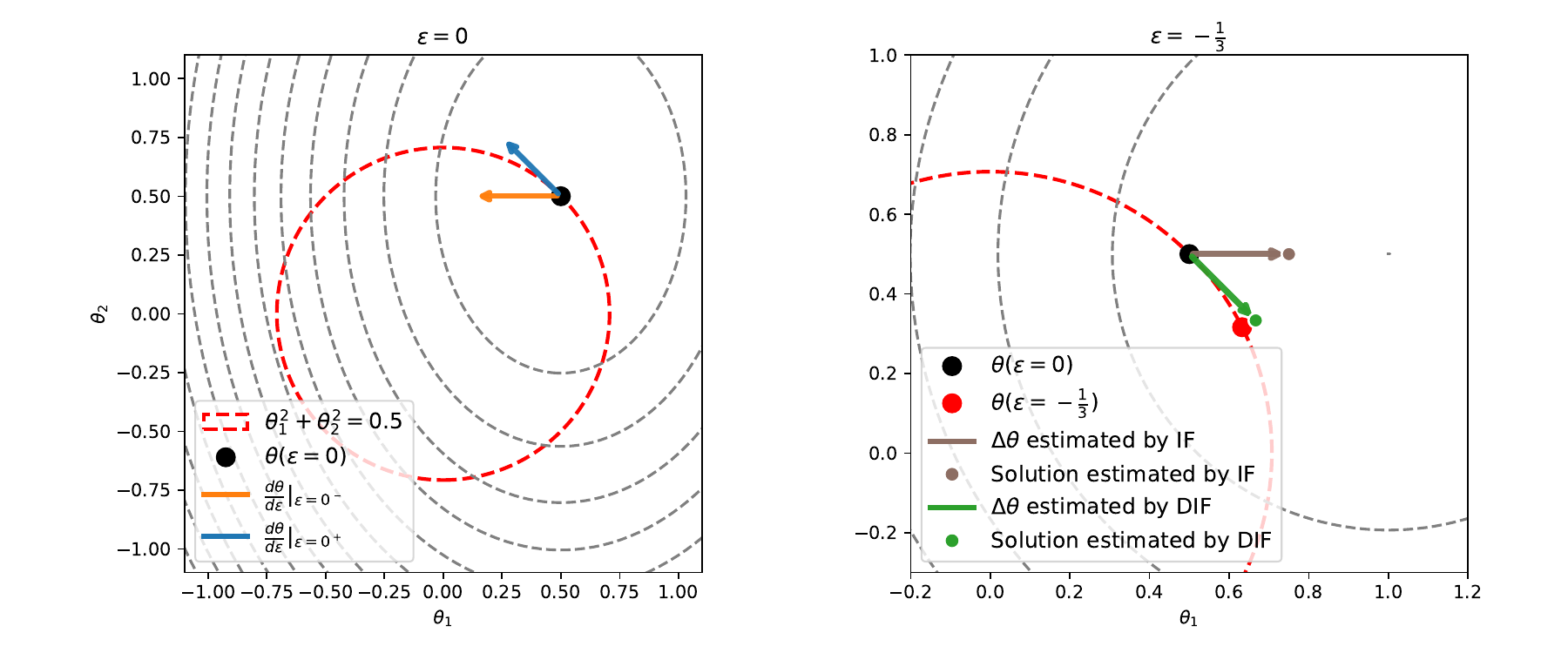}
    \caption{Loss landscape of the toy example when $\varepsilon=0$ and $\varepsilon=-\frac{1}{3}$. The left plot illustrates the non-existence of the derivative of $\theta(\varepsilon)$ at $\varepsilon=0$ as $\left.\frac{d \theta(\varepsilon)}{d \varepsilon}\right|_{\varepsilon=0^+}\neq \left.\frac{d \theta(\varepsilon)}{d \varepsilon}\right|_{\varepsilon=0^-}$. The right plot shows that the IF estimation of $\Delta\theta$ 
    fails to account for the feasible region, whereas DIF accurately predicts the solution change direction and closely matches the ground truth.}
    \label{fig:toy example}
    \vspace{-5pt}
\end{figure}

However, as shown in the left panel of Figure~\ref{fig:toy example}, $\left.\frac{d {\theta}(\varepsilon)}{d \varepsilon}\right|_{\varepsilon=0}$ is not well defined because the left derivative and right derivative are inconsistent. When $\varepsilon>0$, the point $\left(\frac{1}{2+3\varepsilon}, 0.5\right)$ satisfies $\|\cdot\|_2^2 < 0.5$ and thus lies in the interior of the feasible region. When $\varepsilon<0$, $\|\theta(\varepsilon)\|_2^2$ exceeds $0.5$, the solution is projected onto the $\ell_2$-boundary, resulting in two different one-sided derivatives at $\varepsilon=0$. Moreover, the right panel of Figure~\ref{fig:toy example} compares the ground truth $\Delta\theta$ with the change estimated by IF. When $\varepsilon$ shifts from 0 to $-\frac{1}{3}$, $\theta$ moves along the boundary of the feasible region, whereas the estimated change by the IF significantly deviates from the ground truth due to its failure to consider the constraint. In contrast, our estimator DIF (see Definition~\ref{def:DIF}), which accounts for the geometry of the feasible region, closely matches the ground truth $\Delta \theta$ along the boundary. The derivation of DIF for this toy example is provided in Appendix~\ref{appendixoftoy}.

\section{Directional Influence Function (DIF)}\label{sec:3}
To address the non-smoothness of the solution map $\theta(\varepsilon)$ w.r.t.  perturbation $\varepsilon$ and the inability of IF to handle constrained scenarios, this section proposes an alternative approach called the Directional Influence Function (DIF). We start with the definition of the classical IF.

\begin{definition}[Influence Function]
The classical \emph{IF} applies only to unconstrained problems. IF is defined as the derivative of $\theta$ w.r.t. the $\varepsilon_0$, i.e.,
\begin{equation}
    IF(\bar{\theta};Z^r)=\left.\frac{d \theta(\varepsilon_0)}{d \varepsilon_0}\right|_{\varepsilon_0=0}.
\end{equation}
\end{definition}

% \begin{definition}[Directional Influence Function]\label{def:DIF}

\begin{definition}[Directional Influence Function]\label{def:DIF}
\emph{DIF} of the data subset \( Z^r\) is defined as the limiting directional derivative of the solution map \( \theta(\varepsilon): \mathbb{R}^{m+1} \rightarrow \mathbb{R}^d \) at \( \bar{\varepsilon} = \mathbf{0} \), if the limit exists\footnote{This generalized definition accounts for potential non-smoothness of $\theta(\varepsilon)$ by allowing the direction to vary in a neighborhood of $\Delta \bar{\varepsilon}$, as is common in nonsmooth analysis\citep{clarke1990optimization,rockafellar1998variational}.}
\begin{align}
    \mathrm{DIF}\left(\bar{\theta}; Z^r\right)
    &:= D\theta(\bar{\varepsilon};\Delta \bar{\varepsilon})=\lim_{\substack{t \searrow 0^+\\ \Delta {\varepsilon} \to \Delta \bar{\varepsilon}}} \frac{\theta(\bar{\varepsilon} + t \Delta \bar{\varepsilon}) - \theta(\bar{\varepsilon})}{t},
\end{align}
where \( \bar{\theta} = \theta(\bar{\varepsilon}) \) denotes the model solution before perturbation. $ \Delta \bar{\varepsilon} = \hat\varepsilon-\bar\varepsilon= \left[-\frac{1}{N_0}, -\frac{1}{N_1}, \ldots, -\frac{1}{N_m}\right]$ is the perturbation direction, where $ \Delta \bar{\varepsilon}$ corresponds to removing the data subset $Z^r$ from the constrained learning~(\ref{CL}).
\end{definition}

A function $f: X \rightarrow Y$ is called positively homogeneous if $
f(\alpha x)=\alpha f(x) \quad \forall \alpha \geq 0, \forall x \in X.
$
\begin{corollary}\label{corollary2}
DIF is positively homogeneous,i.e.,$$\alpha D\theta(\bar{\varepsilon};\Delta\bar{\varepsilon})=D\theta(\bar{\varepsilon};\alpha\Delta\bar{\varepsilon}),\quad\alpha\in \mathbb{R}^+$$
\end{corollary}

% \begin{proof}
% See Appendix~\ref{app:positivehomo} for the full proof.
% \end{proof}

When $\varepsilon$ is perturbed from $\bar{\varepsilon}$ by a finite $\Delta \bar{\varepsilon}$, IF estimates $\Delta \theta$ via a first-order Taylor expansion. IF pertains to unconstrained learning and ignores constraints; thus
\begin{equation}\label{IF estimator}
    \Delta \theta \approx \left. \frac{d \theta(\varepsilon_0)}{d \varepsilon_0} \right|_{\varepsilon = \bar{\varepsilon}} \cdot \Delta \bar{\varepsilon}_0.
\end{equation}
This approximation, as demonstrated in the previous section, fails in the presence of constraints. Instead, DIF utilizes a directional derivative approach:
\begin{equation}\label{DIF estimator}
\begin{aligned}
    \Delta \theta 
    &\approx D\theta\left(\bar{\varepsilon}; \frac{\Delta \bar{\varepsilon}}{\|\Delta \bar{\varepsilon}\|}\right) \cdot \|\Delta \bar{\varepsilon}\| = D\theta\left(\bar{\varepsilon}; \Delta \bar{\varepsilon}\right).
\end{aligned}
\end{equation}

where $D\theta(\bar{\varepsilon};\frac{\Delta \bar{\varepsilon}}{\|\Delta \bar{\varepsilon}\|})$ is the directional derivative, and the second equation follows from  Corollary~\ref{corollary2}.
In what follows, we will demonstrate how to compute 
$D\theta\left(\bar{\varepsilon};\Delta \bar{\varepsilon}\right)$ by solving a linearized VI system. We use $\Delta\hat {\theta}$ to denote the DIF-based estimate of $\Delta\theta$.

\section{Deriving DIF via Sensitivity Analysis of Optimality Conditions}\label{sec:4}

In this section, we show how DIF naturally arises from the sensitivity analysis of the optimality conditions of the constrained learning problem. 
Specifically, we first formulate the optimality system as a VI, then linearize this VI to obtain an auxiliary VI, and finally convert it into a QP whose solution yields the desired directional change in the model solution.

% This section is to estimate the change in $\theta$ when $\varepsilon$ is perturbed from $\bar{\varepsilon}$ by a finite $\Delta \bar{\varepsilon}$.
% While the IF estimates $\Delta \theta$ via a first-order Taylor expansion:
% \begin{equation}\label{IF estimator}
%     \Delta \theta \approx \left. \frac{d \theta(\varepsilon)}{d \varepsilon} \right|_{\varepsilon = \bar{\varepsilon}} \cdot \Delta \bar{\varepsilon}.
% \end{equation}
% This approximation, as demonstrated in the previous section, fails in the presence of constraints. Instead, DIF utilizes a directional derivative approach:
% \begin{equation}\label{DIF estimator}
% \begin{aligned}
%     \Delta \theta 
%     &\approx D\theta\left(\bar{\varepsilon}; \frac{\Delta \bar{\varepsilon}}{\|\Delta \bar{\varepsilon}\|}\right) \cdot \|\Delta \bar{\varepsilon}\| = D\theta\left(\bar{\varepsilon}; \Delta \bar{\varepsilon}\right) 
%     \stackrel{\text { def }}{=} \operatorname{DIF}\left(\bar{\theta};Z^r\right)
% \end{aligned}
% \end{equation}

% where $D\theta(\bar{\varepsilon};\frac{\Delta \bar{\varepsilon}}{\|\Delta \bar{\varepsilon}\|})$ is the directional derivative, and the second equation follows from  Corollary~\ref{corollary2}.
% In what follows, we will demonstrate how to compute 
% $D\theta\left(\bar{\varepsilon};\Delta \bar{\varepsilon}\right)$ by solving a linearized variational inequality system. We use $\Delta\hat {\theta}$ to denote the estimation of $\Delta\theta$.
Constrained learning is typically solved using Lagrangian dual approaches, such as the augmented Lagrangian method, primal-dual methods~\citep{chamon2020probably,ahmed2022pylon,nandwani2019primal}. To quantify the impact of data downweighting on the solution of PCL problem~(\ref{perturb CL}), we introduce the perturbation $\varepsilon$ into the Lagrangian function, enabling us to analyze how the perturbation affects the optimality system. The Lagrangian function of PCL problem~(\ref{perturb CL}) is:  
\begin{equation}
\begin{aligned}
&{L}(\varepsilon,\theta, \lambda)=  \frac{1}{N_0} \sum_{i=1}^{N_0} \ell_0\left(z_i^{(0)}, \theta\right)+\varepsilon_0 \sum_{z_i^{(0)}\in Z^r} \ell_0\left(z_{i}^{(0)}, \theta\right) \\
& +\sum_{j=1}^m \lambda_j\left[\frac{1}{N_j} \sum_{i=1}^{N_j} \ell_j\left(z_i^{(j)}, \theta\right)+\varepsilon_j \sum_{z_i^{(j)}\in Z^r}\ell_j\left(z_{i}^{(j)}, \theta\right)-\tau_j\right],
\end{aligned}
\end{equation}
where $\lambda=[\lambda_1,\lambda_2,\dots,\lambda_m], \lambda_j\geq 0 $ denotes the dual variables. Let $(\theta(\varepsilon),\lambda(\varepsilon))$ be the primal--dual solution for a given $\varepsilon$. From Section~\ref{sec:2}, we use $(\bar\theta,\bar\lambda)=(\theta(\bar\varepsilon),\lambda(\bar\varepsilon))$ to denote the optimal solution of the constrained learning problem~(\ref{CL}). 

At $\theta=\bar{\theta}$, we partition the constraints into the active set $I_{\text {Active }}$ and the inactive set $I_{\text {Inactive}}$, where $I_{\text {Active }}=I_{\text{Binding}}\cup I_{\text{Non-binding}}$.

\begin{align*}
I_{\text{Active}} &:= \left\{ j \,\middle|\, 
\frac{1}{N_j}\sum_{i=1}^{N_j}\ell_j(z_i^{(j)},\bar\theta)=\tau_j \right\},
&
I_{\text{Inactive}} &:= \left\{ j \,\middle|\, 
\frac{1}{N_j}\sum_{i=1}^{N_j}\ell_j(z_i^{(j)},\bar\theta)<\tau_j \right\}, \\[3pt]
I_{\text{Binding}} &:= \{ j\in I_{\text{Active}} \mid \bar\lambda_j>0 \},
&
I_{\text{Non-binding}} &:= \{ j\in I_{\text{Active}} \mid \bar\lambda_j=0 \}.
\end{align*}

Varying $\varepsilon$ from $\bar\varepsilon$ to $\hat\varepsilon$ corresponds to removing $Z^r$. Next, we study the sensitivity of the optimality conditions with respect to $\varepsilon$.  We begin by examining the optimality condition of PCL~(\ref{perturb CL}).

\subsection{Optimality Condition of PCL Problem ~(\ref{perturb CL})}
The optimality condition of PCL ~(\ref{perturb CL}) can be described by the Karush-Kuhn-Tucker (KKT) conditions. However, directly analyzing the KKT condition of (\ref{perturb CL}) turns out to be difficult for understanding how the solution changes with data perturbations. Here, we adopt a VI formulation to characterize the optimality conditions. VI is a more general form of nonlinear optimization, which allows us to analyze solution perturbations through the lens of generalized differential calculus~\citep{rockafellar1998variational}.  The optimality condition of PCL~(\ref{perturb CL}) is:
\begin{equation}\label{VI}
-\nabla_\theta L(\varepsilon, \theta, \lambda) \in N_{\mathbb{R}^d}(\theta), \quad \nabla_\lambda L(\varepsilon, \theta, \lambda) \in N_{\mathbb{R}_{+}^m}(\lambda),
\end{equation}
where $N_{\mathbb{R}^d}(\theta),N_{\mathbb{R}_{+}^m}(\lambda)$ denote the Normal Cone.
\begin{definition}[Normal Cone;\citet{dontchev2009implicit}]\label{def:normal cone}
Let $C \subseteq \mathbb{R}^d$ be a closed convex set. The normal cone to $C$ at a point $x \in C$ is defined as:
\[
N_C(x) = \{v \in \mathbb{R}^d : \langle v, y - x \rangle \leq 0, \ \forall y \in C \}.
\]
\end{definition}

%  \begin{theorem}\label{thm:5}
% Assume the following two conditions hold:
% \begin{itemize}
%     \item[1.] The gradients $\frac{1}{N_j} \sum_{i=1}^{N_j} \nabla_\theta \ell_j(z_i^{(j)}, \bar{\theta})$ associated with the active constraints are linearly independent.
    
%     \item[2.] For any $\Delta\theta \neq 0$ such that $\Delta\theta \perp \frac{1}{N_j} \sum_{i=1}^{N_j} \nabla_\theta \ell_j(z_i^{(j)}, \bar{\theta})$ for all $j \in I_{\text{Binding}}$, it holds that $\langle \Delta \theta, \nabla^2_{\theta\theta} L(\bar{\varepsilon}, \bar{\theta}, \bar{\lambda}) \Delta \theta \rangle > 0$.
% \end{itemize}
% Then the DIF exists.
% \end{theorem}
\subsection{Linear Approximation of Variational Inequality}
The classical derivation of IF relies on the Taylor expansion of the first-order optimality condition, which does not naturally handle constraints. In constrained learning, a natural choice for the optimality condition is the KKT system. Yet, estimating how $\varepsilon$ affects the KKT system is challenging due to its mixed system of equations and inequalities.

DIF adopts an alternative: representing the optimality condition as a VI (see formulation~(\ref{VI})) and applying a first-order approximation directly to the VI.
The VI (\ref{VI}) can be compactly written as:
\begin{align}\label{new VI}
&f(\varepsilon, \theta, \lambda) + N_E(\theta, \lambda) \ni \mathbf{0}, \\ 
&\text{where}\quad 
f(\varepsilon, \theta, \lambda) = 
\begin{pmatrix}
\nabla_\theta L(\varepsilon, \theta, \lambda) \\
- \nabla_\lambda L(\varepsilon, \theta, \lambda)
\end{pmatrix}, E = \mathbb{R}^d \times \mathbb{R}_+^m \notag
\end{align}
By definition, the solution $(\bar{\varepsilon},\bar{\theta},\bar{\lambda})$ satisfies the equation~(\ref{new VI}). When $\varepsilon$ is perturbed by $\Delta \bar{\varepsilon}$, estimating $\Delta \theta$ and $\Delta \lambda$ amounts to finding $\Delta \theta$ and $\Delta \lambda$ such that 
\begin{equation}\label{perturb VI}
    f(\bar{\varepsilon}+\Delta \bar{\varepsilon}, \bar{\theta}+\Delta \theta, \bar{\lambda}+\Delta \lambda) + N_E(\bar{\theta}+\Delta \theta, \bar{\lambda}+\Delta \lambda) \ni \mathbf{0}
\end{equation}
The linear approximation of (\ref{perturb VI}) is 
\begin{align}
&f(\bar{\varepsilon}, \bar{\theta}, \bar{\lambda}) 
+ \nabla_{\varepsilon} f(\bar{\varepsilon}, \bar{\theta}, \bar{\lambda})\, \Delta \bar{\varepsilon}
+ \nabla_{(\theta, \lambda)} f(\bar{\varepsilon}, \bar{\theta}, \bar{\lambda})\, \left[\begin{array}{l}\Delta \hat{\theta} \\ \Delta \hat{\lambda}\end{array}\right] \nonumber \\
&+ N_{E}((\bar{\theta}, \bar{\lambda})+ (\Delta \hat\theta, \Delta \hat\lambda)) \ni \mathbf{0}. \label{vi-linearized}
\end{align}
Here, $(\Delta \hat{\theta}, \Delta \hat{\lambda})$, which satisfies~(\ref{vi-linearized}), serves as an approximation of $(\Delta \theta, \Delta \lambda)$. Denote 

\begin{equation*}
\begin{aligned}
& A:=\nabla_{(\theta, \lambda)} f(\bar{\varepsilon}, \bar{\theta}, \bar{\lambda})=\left[\begin{array}{l}
\nabla_\theta^2 L(\bar{\varepsilon}, \bar{\theta}, \bar{\lambda}), \nabla_{\theta \lambda} L(\bar{\varepsilon}, \bar{\theta}, \bar{\lambda}) \\
-\nabla_{\theta\lambda} L(\bar{\varepsilon}, \bar{\theta}, \bar{\lambda}),-\nabla^2_{\lambda}L(\bar{\varepsilon}, \bar{\theta}, \bar{\lambda})
\end{array}\right],\\
& \mu:=\nabla_{\varepsilon} f(\bar{\varepsilon}, \bar{\theta}, \bar{\lambda})=\left[\begin{array}{l}
\nabla_{\theta \varepsilon} L(\bar{\varepsilon}, \bar{\theta}, \bar{\lambda}) \\
-\nabla_{\lambda \varepsilon} L(\bar{\varepsilon}, \bar{\theta}, \bar{\lambda})
\end{array}\right],\\ &\Delta \hat\eta=\left[\begin{array}{l}\Delta \hat{\theta} \\ \Delta \hat{\lambda}\end{array}\right].
\end{aligned}
\end{equation*}
The linearized VI~(\ref{vi-linearized}) can be expressed in the form:
\begin{equation}\label{compressed LVI}
  f(\bar{\varepsilon}, \bar{\theta}, \bar{\lambda})+\mu\Delta \bar{\varepsilon}+A\Delta \hat\eta+N_E((\bar{\theta},\bar{\lambda})+\Delta \hat\eta) \ni 0.
\end{equation}
Here $\Delta\bar\varepsilon$ is given. Note that all terms in the linearized VI~(\ref{compressed LVI}), except for $\Delta \hat\eta$—which is the variable we aim to estimate—can be precomputed from the known solution $\bar{\theta}$ and $\bar{\lambda}$.  The linearized VI~(\ref{compressed LVI}) can be simplified into the following Auxiliary VI~(\ref{aux vi}).
\begin{proposition}[Auxiliary VI]\label{pro:aux vi}
   The solution $\Delta \hat\eta$ to the linearized VI~(\ref{compressed LVI}) satisfies 
\begin{equation}\label{aux vi}
  \mu\Delta \bar{\varepsilon}+A\Delta \hat\eta+N_K(\Delta \hat\eta) \ni 0,
\end{equation}
where $K=\mathbb{R}^d \times D$. $D$ is a space defined as
\begin{equation*}
D := \left\{ \Delta \hat\lambda \in \mathbb{R}^m \;\middle|\;
\begin{aligned}
&\Delta\hat \lambda_j \in \mathbb{R} \quad \text{for }  j \in I_{\text{Binding}}, \\
&\Delta\hat \lambda_j \geq 0 \quad \text{for } j \in I_{\text{Non-binding}}, \\
&\Delta \hat\lambda_j = 0 \quad \text{for } j \in I_{\text{Inactive.}}
\end{aligned}
\right\}.
\end{equation*}
\end{proposition}
% \begin{proof}
%     See Appendix~\ref{proof of propostion6}
% \end{proof}
Define (\ref{aux vi}) as the \textit{Auxiliary VI}, which admits the same solution as the linearized VI~(\ref{compressed LVI}), but with a simpler form.

\begin{proposition}\label{proposition7}
Denote $\Delta \hat\eta=[\Delta \hat{\theta} ; \Delta \hat{\lambda}]$ the solution to the Auxiliary VI~(\ref{aux vi}). Then, $\Delta \hat{\theta}$ exactly recovers the directional derivative of the solution mapping $\theta(\varepsilon)$ at $\bar \varepsilon$ along $\Delta \bar{\varepsilon}$, i.e.,
\begin{equation} 
\Delta \hat{\theta}=D \theta(\bar{\varepsilon} ; \Delta \bar{\varepsilon}).
\end{equation}
\end{proposition}
% \begin{proof}
%     See Appendix~\ref{proofof7}.
% \end{proof}
\begin{remark}
By Proposition~\ref{proposition7}, computing DIF reduces to solving the Auxiliary VI~(\ref{aux vi}).
\end{remark}
The following theorem establishes the existence of DIF.
\begin{theorem}\label{thm:5}
Assume the following regularity conditions hold:
\begin{enumerate}[label=\textbf{Assumption \arabic*.},  align=left]
    \item (LICQ: Linear Independence Constraint Qualification) 
    The gradients $\frac{1}{N_j} \sum_{i=1}^{N_j} \nabla_\theta \ell_j(z_i^{(j)}, \bar{\theta})$ for the active constraints ($j \in I_{\text{Activate}}$) are linearly independent.

    \item (SOSC: Second-Order Sufficient Condition) 
    For any $\Delta\theta \neq 0$ such that $\Delta\theta \perp \frac{1}{N_j} \sum_{i=1}^{N_j} \nabla_\theta \ell_j(z_i^{(j)}, \bar{\theta})$ for all $j \in I_{\text{Binding}}$, it holds that $\langle \Delta \theta, \nabla^2_{\theta\theta} L(\bar{\varepsilon}, \bar{\theta}, \bar{\lambda}) \Delta \theta \rangle > 0$.
\end{enumerate}
Then the directional derivative $D\theta(\bar{\varepsilon};\Delta\bar{\varepsilon})$ exists.
\end{theorem}
% \begin{proof}
%     See Appendix~\ref{proof of thm5}.
% \end{proof}

% \begin{theorem}\label{thm:5}
% Assume the following two conditions hold:
% \begin{itemize}
%     \item[1.] The gradients $\frac{1}{N_j} \sum_{i=1}^{N_j} \nabla_\theta \ell_j(z_i^{(j)}, \bar{\theta})$ associated with the active constraints are linearly independent.
    
%     \item[2.] For any $\Delta\theta \neq 0$ such that $\Delta\theta \perp \frac{1}{N_j} \sum_{i=1}^{N_j} \nabla_\theta \ell_j(z_i^{(j)}, \bar{\theta})$ for all $j \in I_{\text{Binding}}$, it holds that $\langle \Delta \theta, \nabla^2_{\theta\theta} L(\bar{\varepsilon}, \bar{\theta}, \bar{\lambda}) \Delta \theta \rangle > 0$.
% \end{itemize}
% Then the  directional derivative $D \theta(\bar{\varepsilon} ; \Delta \bar{\varepsilon})
% $ exists
% \end{theorem}

\begin{proposition}
\label{thm:DIF-error}
$\Delta\theta=\theta(\bar\varepsilon+\Delta \bar\varepsilon)-\theta(\bar\varepsilon)$ is the ground truth and $\Delta\hat\theta$ is the estimation. Under the assumption of Theorem~\ref{thm:5} , there exists a constant $M>0$ such that, for all perturbations $\Delta \bar{\varepsilon}$ sufficiently small, the following holds:
\begin{equation}
\|\Delta\theta - \Delta\hat\theta\|
\le 
M\|\Delta\bar\varepsilon\|
\end{equation}
\end{proposition}
% \begin{proof}
%     See Appendix~\ref{proof of DIF error}.
% \end{proof}

Note that LICQ in Theorem \ref{thm:5} is the regular constraint qualification, while SOSC essentially requires some type of "local convexity" (i.e., at the current solution) of the Lagrangian function over all feasible directions 
 that do not violate the binding constraints. They are the conditions that the proposed DIF method can apply to analyze data attributions of a learning model. This local convexity may exist even when the model is non-convex in general, such as deep learning models. Furthermore, Proposition \ref{thm:DIF-error} states that, under the two assumptions, solution changes over data perturbations are Lipschitz continuous. This ensures that when constructing the auxiliary problem (\ref{QP}) in the next subsection, the constraint set is stable, i.e., no sudden change will occur when infinitesimally small data perturbations are imposed.

% \begin{lemma}\label{lem:gap-theta}
% There exists a neighborhood of $0$ such that
% \begin{equation}\label{eq:gap-small-o}
% \big\|\Delta\theta-\Delta\hat\theta\big\| \;=\; o\!\left(\,\|\Delta\bar\varepsilon\|\,\right)
% \qquad\text{as } \Delta\bar\varepsilon\to 0 .
% \end{equation}
% Moreover, if $f$ is $C^2$ in a neighborhood of $(\bar\varepsilon,\bar\eta)$, then
% \begin{equation}\label{eq:gap-big-O}
% \big\|\Delta\theta-\Delta\hat\theta\big\|
% \;\le\; c\,\|\Delta\bar\varepsilon\|^2
% \quad\text{for all sufficiently small }\|\Delta\bar\varepsilon\|,
% \end{equation}
% for some constant $c>0$ independent of $\Delta\bar\varepsilon$.
% \end{lemma}

\subsection{Computing DIF via Quadratic Programming}

Proposition~\ref{proposition7} implies that computing the DIF reduces to solving the Auxiliary VI~(\ref{aux vi}). In this section, we construct a QP~(\ref{QP}) whose optimality is exactly the Auxiliary VI~(\ref{aux vi}). By solving this QP, we can compute DIF and estimate the directional change in solution, i.e., $\Delta \theta$.
Figure~\ref{fig:flow} summarizes the derivation in this section, illustrating the connections between the PCL problem~(\ref{perturb CL}), VI~(\ref{new VI}), Auxiliary VI~(\ref{aux vi}), and QP~(\ref{QP}) formulations.

We provide the computational complexity analysis of the QP solver and a scalable implementation strategy in Appendix~\ref{app:qp-complexity}.
\begin{figure}[h]
\centering
\begin{tikzpicture}[
  box/.style={draw, minimum width=1.5cm, minimum height=0.5cm, align=center, font=\normalsize},
  node distance=1cm and 2.0cm,thick, 
  every node/.style={font=\small}
]

% Nodes
\node[box] (pcl) {PCL~(\ref{perturb CL})};
\node[box, below=of pcl] (vi) {VI~(\ref{new VI})};
\node[box, right=of vi] (auxvi) {Auxiliary VI~(\ref{aux vi})};
\node[box, above=of auxvi] (qp) {\shortstack{QP}~(\ref{QP})};

% Vertical arrows
\draw[dashed,thick,<->] (pcl) -- (vi) node[midway, left] {\shortstack{Optimality Condition}};
\draw[dashed,->] (pcl) -- (vi) node[midway, right] {\shortstack{Equivalent}};
\draw[dashed,thick,<->] (auxvi) -- (qp) node[midway, right] {\shortstack{Theorem~\ref{thm:Auxiliary problem}}};
% \draw[dashed,->] (qp) -- (auxvi) node[midway, left] {Theorem 2};
\draw[dashed,thick,<->] (auxvi) -- (qp) node[midway, left] {\shortstack{Equivalent}};
% Horizontal arrow
\draw[dashed,->] (vi) -- (auxvi)
  node[midway, above=1.2mm] {{Linearization}}
  node[midway, below=1.2mm] {{Proposition~\ref{pro:aux vi}}};

% DIF arrows with spacing and balance
\draw[->] (qp) -- (pcl)
  node[midway, above=1.2mm] {$\Delta \theta$}
  node[midway, below=1.2mm] {$\Delta \lambda$};

\end{tikzpicture}
\caption{The relationship between (\ref{perturb CL}) and (\ref{QP}).}
\label{fig:flow}
\end{figure}

We define the following QP:
\begin{equation}
\label{QP}
\begin{aligned}
\underset{\omega}{\min}\quad &
L(\bar{\varepsilon}, \bar{\theta}, \bar{\lambda}) 
+ \big\langle \nabla_{\theta \varepsilon} L(\bar{\varepsilon}, \bar{\theta}, \bar{\lambda})  \Delta \bar\varepsilon,\; \omega \big\rangle 
+ \frac{1}{2}\big\langle \omega,\; \nabla_{\theta\theta}^2 L(\bar{\varepsilon}, \bar{\theta}, \bar{\lambda})\, \omega \big\rangle\\[4pt]
\text{s.t.}\quad &
\frac{1}{N_j}\sum_{i=1}^{N_j} \ell_j(z_i^{(j)}, \bar{\theta})
+ \frac{1}{N_j}\sum_{i=1}^{N_j} \nabla_\theta \ell_j(z_i^{(j)}, \bar{\theta})\!\cdot\!\omega
+ \sum_{z_i^{(j)}\in Z^r} \Delta \bar\varepsilon_j \ell_j(z_{i}^{(j)}, \bar{\theta})
- \tau_j
\begin{cases}
= 0, & j \in I_{\text{binding}},\\[3pt]
\le 0, & j \in I_{\text{non-binding}},\\[3pt]
\text{free}, & j \in I_{\text{Inactive}}.
\end{cases}
\end{aligned}
\end{equation}

\begin{theorem}[Auxiliary problem]\label{thm:Auxiliary problem}
Let $\left(\omega^{\star}, \zeta^{\star}\right)$ denote the optimal primal-dual solution to the QP~(\ref{QP}).  Then $(\omega^\star,\zeta^\star)$ is also a solution to the auxiliary VI~(\ref{aux vi}). In particular, the QP~(\ref{QP}) and the auxiliary VI~(\ref{aux vi}) admit the same solution pair under substitution $(\omega^*, \zeta^*)=(\Delta \hat{\theta}, \Delta \hat{\lambda})$.
\end{theorem}
% \begin{proof}
%     See Appendix~\ref{proofof8}. 
% \end{proof}

\begin{remark}
By Proposition~\ref{proposition7}, $\Delta\hat\theta = D\theta(\bar{\varepsilon};\,\Delta\bar{\varepsilon})$.
Hence $\omega^\star$ is identical to $\operatorname{DIF}(\bar{\theta};\,Z^r)$.
\end{remark}
Finally, we establish the relationship between the IF and the DIF.
\begin{proposition}\label{prop:dif-if}
If no constraints are active at the solution $\bar{\theta}$, i.e., $I_{\text{Active}}=\emptyset$,  
then the DIF coincides with the classical IF.
\end{proposition}
% \begin{proof}
%     See Appendix~\ref{proofofdif-if}.
% \end{proof}

\section{Validation}\label{sec:5}
\subsection{DIF Validating via Constrained Linear Regression}
This section evaluates DIF on constrained linear regression. 
We generate $n=1000$ samples with $d=5$ features:
\[
X_{ij}\sim\mathcal{N}(0,1),\qquad 
\theta^*\sim\mathcal{N}(0,I_d),\qquad 
y=X\theta^*+\varepsilon,\quad 
\varepsilon\sim\mathcal{N}(0,0.1^2I_n).
\]
The model is trained by solving
\[
\hat{\theta}
=
\arg\min_{\theta\in\mathbb{R}^5}
\frac{1}{2n}\|X\theta-y\|_2^2
\quad
\mathrm{s.t.}\quad
A_{\mathrm{eq}}\theta=b_{\mathrm{eq}},\quad
A_{\mathrm{ineq}}\theta\le b_{\mathrm{ineq}},
\]
where
\[
A_{\mathrm{eq}}=[1,1,1,1,1],\quad b_{\mathrm{eq}}=-4.0,\quad
A_{\mathrm{ineq}}=
\begin{bmatrix}
1&1&0&0&0\\
0&0&-1&0&1
\end{bmatrix},
\quad
b_{\mathrm{ineq}}=
\begin{bmatrix}
1.5\\
1.5
\end{bmatrix}.
\]

We conduct 100 single-point removal trials and compare the estimated solution change $\Delta\theta$ with leave-one-out (LOO) retraining. 
We compare DIF with two baselines: classical IF, which ignores constraints, and penalty-IF, which applies IF to a softly penalized surrogate. 
DIF is computed by solving the QP in~(\ref{QP}); full derivations for all estimators are given in Appendix~\ref{appendix:constrained linear regression}. 
All optimization problems are solved with \texttt{CVXPY}. 
As shown in Figure~\ref{fig:a}, DIF nearly coincides with LOO, while IF and penalty-IF exhibit clear biases.
\begin{figure}[htbp]
  \centering
  \begin{subfigure}[b]{0.35\linewidth}
    \centering
    \includegraphics[width=\linewidth]{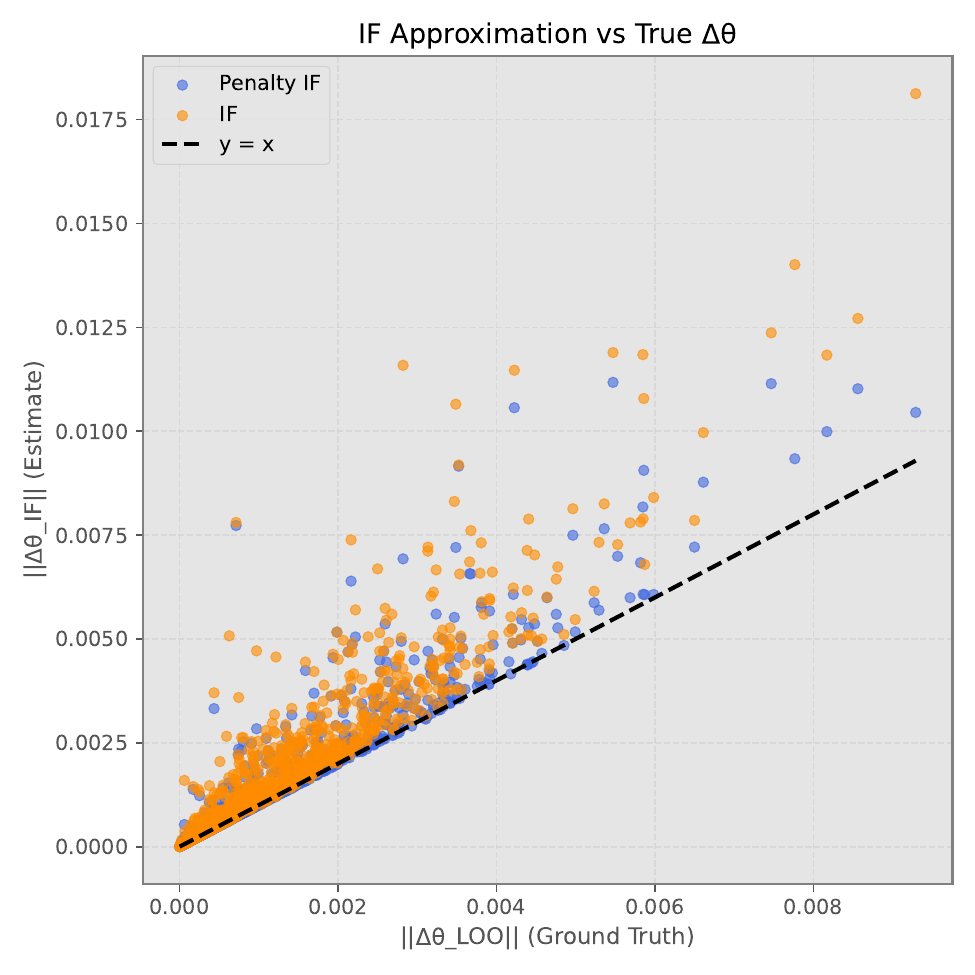}
    \caption{}
  \end{subfigure}
  \hspace{2mm} % 
  \begin{subfigure}[b]{0.35\linewidth}
    \centering
    \includegraphics[width=\linewidth]{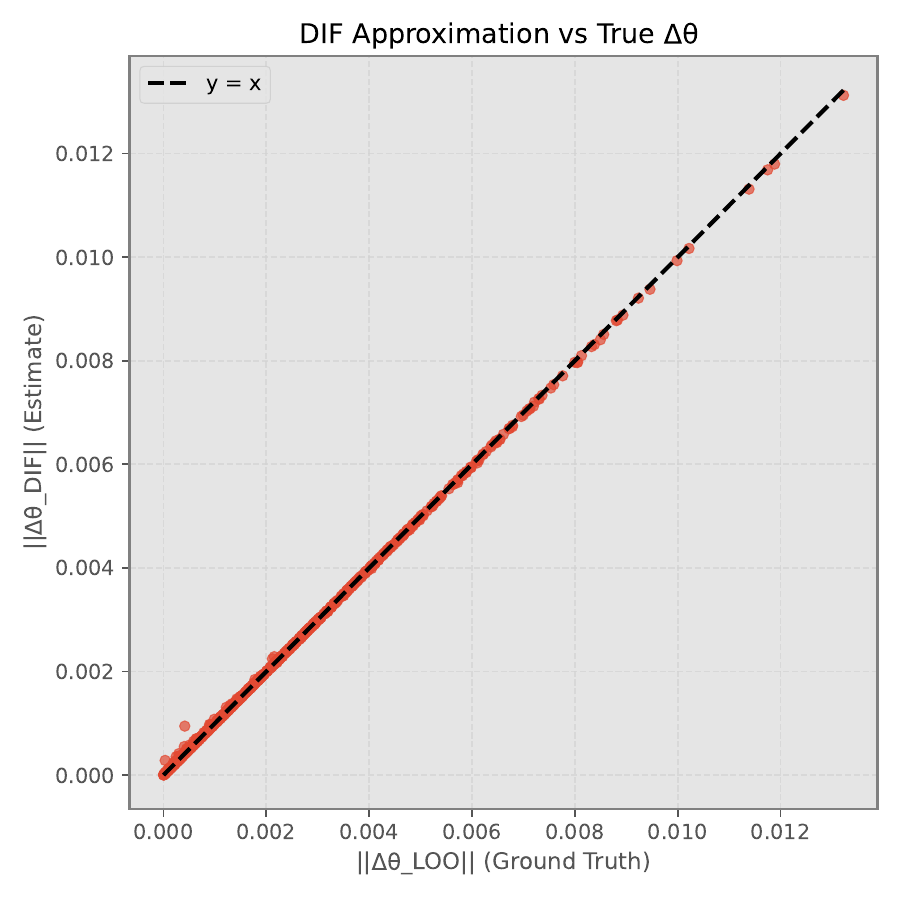}
    \caption{}
  \end{subfigure}

  \caption{Estimated versus LOO solution changes on constrained linear regression. DIF aligns closely with the $y=x$ line, whereas IF and penalty-IF show noticeable biases.}
  \label{fig:a}
\end{figure}

\subsection{DIF Validation via Constrained CNN}\label{DIFviladation}

\paragraph{Model} We adopt the following constrained learning formulation~\citep{shen2022agnostic}:
\begin{equation}
\label{eq:constrained-cnn}
\min_{\theta \in \Theta} \; \overline{R}(f_\theta) 
\quad \text{s.t.} \quad 
R_i(f_\theta) - \overline{R}(f_\theta) - \tau \leq 0, 
\;\; i = 1, \dots, m,
\end{equation}
where $R_i(f_\theta) = \mathbb{E}_{(x,y)\sim \mathcal{D}_i}[\ell(f_\theta(x),y)]$ is the risk of the client (or group) $i$, and $\overline{R}(f_\theta) = \frac{1}{C}\sum_{i=1}^C R_i(f_\theta)$ is the average risk. The constraints ensure that no group’s risk exceeds the global average risk by more than $\tau$, thereby limiting the performance disparity across groups.

We use a network with seven convolutional layers and $\tanh (\cdot)$ non-linearities, modeled after the all convolutional network of \citet{springenberg2014striving}. We train it on the MNIST training set~\citep{lecun1998mnist}. We solve this problem using a primal--dual method, jointly updating the model parameters $\theta$ and the Lagrange multipliers. 

\paragraph{Heterogeneous Data Partitioning.} We create non-IID group-wise data partitions following~\citep{shen2022agnostic}. We split the dataset into $C$ groups by allocating an $\alpha$-fraction of the data uniformly at random and distributing the remaining $(1-\alpha)$-fraction in a label-skewed way, where samples are sorted by class labels and assigned consecutively to groups. Unless otherwise specified, we set the number of groups to $m=3$ and $\alpha=0.5$. This setup creates label-imbalanced distributions across different groups.

\paragraph{Influence Estimation.}
We evaluate DIF on the top 100 influential training samples under the fairness constraint, removing one sample at a time. 
Since CNN retraining may move to a different local basin, we compare methods through their induced loss changes on a fixed misclassified test point rather than through parameter differences. 
Using $\theta'=\theta+\Delta\theta_{\mathrm{DIF}}$, we compute
$
\Delta \ell
=
R(f_{\theta'}(x_{\mathrm{test}}),y_{\mathrm{test}})
-
R(f_{\theta}(x_{\mathrm{test}}),y_{\mathrm{test}}).
$
Figure~\ref{fig:DIF} shows that DIF-predicted loss changes closely match leave-one-out retraining, with Pearson $r=0.90$, suggesting that the local regularity conditions in Theorem~\ref{thm:5} remain meaningful around the trained CNN solution.

\begin{wrapfigure}{r}{0.48\linewidth}
    \vspace{-6pt}
    \centering
    \includegraphics[width=\linewidth]{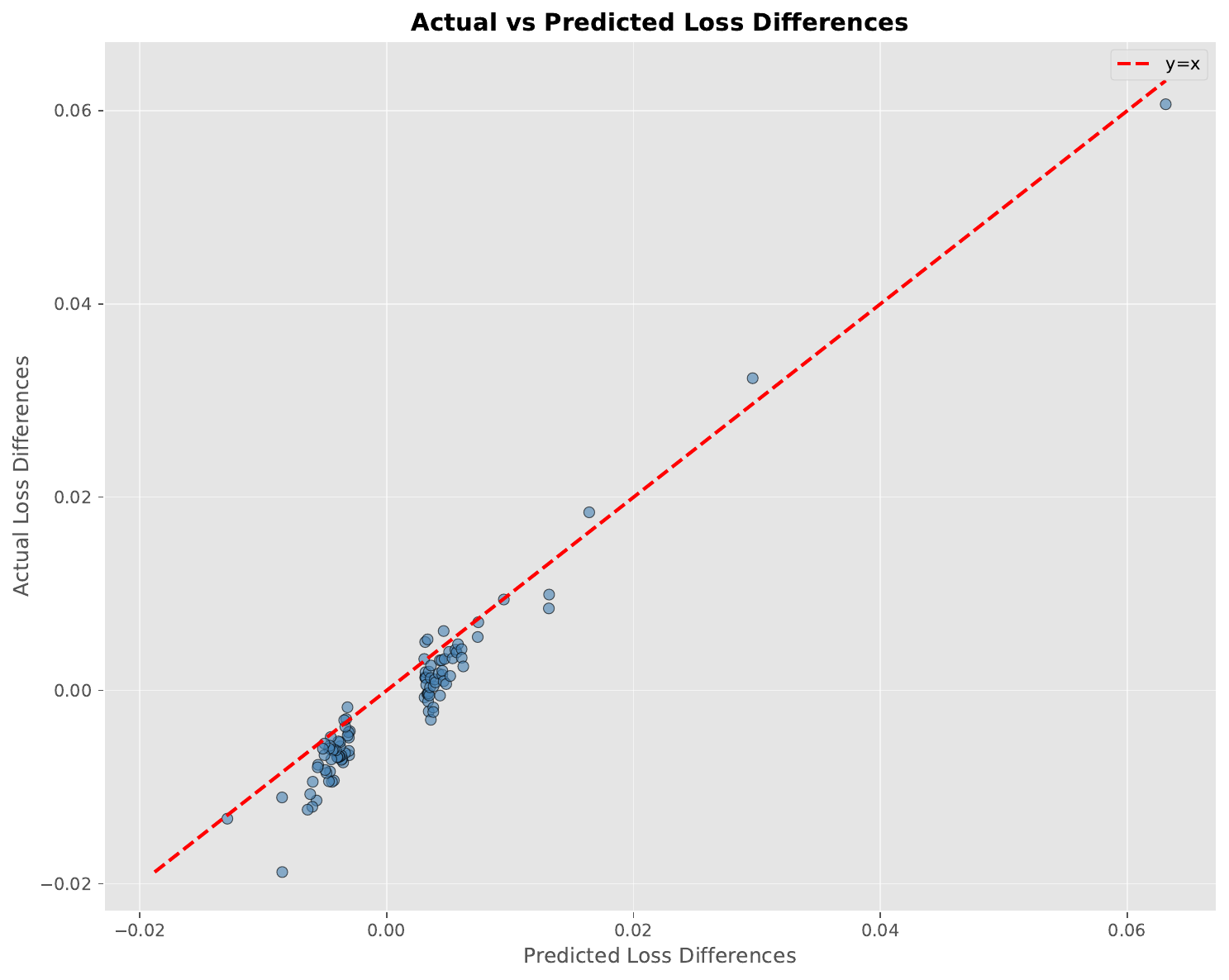}
    \vspace{-3mm}
    \caption{Actual vs. DIF-predicted loss changes on a misclassified test sample.}
    \label{fig:DIF}
    \vspace{-3mm}
\end{wrapfigure}

On the same samples, standard IF fails to converge reliably due to ill-conditioned Hessian-vector products under fairness constraints. 
Penalty-IF converges after using a penalized surrogate, but is sensitive to the penalty coefficient and reaches only Pearson $r\approx0.62$ after tuning. 
Thus, unlike DIF, penalty-IF does not accurately capture the KKT-based sensitivity structure. 
Additional numerical comparisons are provided in Appendix~\ref{app:penalty-if-results}.

% \begin{table}[ht]
% \centering
% \caption{Comparison of IF predicted influence, DIF result, and actual retraining loss change.}
% \label{tab:if_dif_retrain}
% \begin{tabular}{cccc}
% \toprule
% \textbf{Sample ID} & \textbf{Predicted (Penalty-IF)} & \textbf{Predicted (DIF)} & \textbf{Actual (Retrain)} \\
% \midrule
% 4221 &  0.005528  &  0.004979 &  0.003664 \\
% 215  &  0.000016  &  0.004249 &  0.003015 \\
% 1328 & -0.009318 & -0.003962& -0.003139 \\
% 4668 & 0.013700  &  0.003922    &  0.004601 \\
% 772  & 0.001171 &  0.003030 &  0.003509 \\
% 4473 &  0.005000  &  0.002985 &  0.002487 \\
% 4144 &-0.007275  & -0.002820 & -0.000704 \\
% 2828 &  0.002719 &  0.002775   &  0.002441 \\
% 3505 & 0.001118  &  0.002747 &  0.003592 \\
% 1163 & 0.001183 &  0.002359 &  0.002551 \\
% \bottomrule
% \end{tabular}
% \end{table}

\textbf{Counterfactual retraining.} 
We further evaluate counterfactual retraining with DIF, constrained TracIn~\citep{pruthi2020estimating}, and TRAK~\citep{park2023trak}. 
For each method, we compute per-sample influence scores, remove the most harmful \textbf{2\%} and \textbf{10\%} training samples, retrain the constrained CNN, and measure the loss change on the same misclassified test point. 
Constrained TracIn records Lagrangian-loss gradients along the primal-dual trajectory, while TRAK is applied to a squared-hinge penalized objective; details are provided in Appendix~\ref{app:counterfactual-details}. 
As shown in Figure~\ref{fig:conter}, DIF and constrained TracIn substantially outperform random removal, with DIF slightly better than constrained TracIn, whereas TRAK gives little improvement. 
This suggests that preserving constraint-adjusted gradient information is important for identifying harmful samples in constrained training.

\paragraph{Additional experiments.}
Two additional studies are included in the appendix. Appendix~\ref{AppendixPINN} presents a PINN-based machine unlearning case study, and Appendix~\ref{app:fairface_large_scale} reports a large-scale fairness-constrained ResNet-50 experiment on FairFace~\citep{karkkainen2021fairface}. These results further validate DIF in downstream and large-scale constrained learning settings.

\begin{figure}[t]

    \centering
    
    %------------ Left figure ------------
    \begin{subfigure}{0.49\linewidth}
        \centering
        \includegraphics[width=\linewidth]{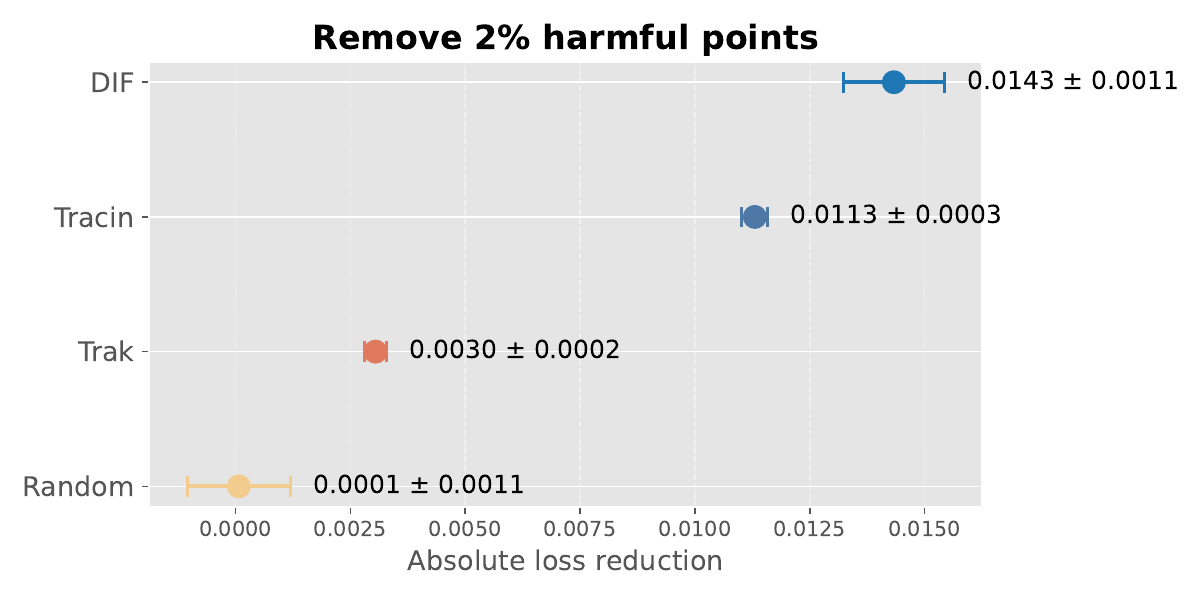}
        \caption{Remove 2\% harmful points}
        
    \end{subfigure}
    \hfill
    %------------ Right figure ------------
    \begin{subfigure}{0.49\linewidth}
        \centering
        \includegraphics[width=\linewidth]{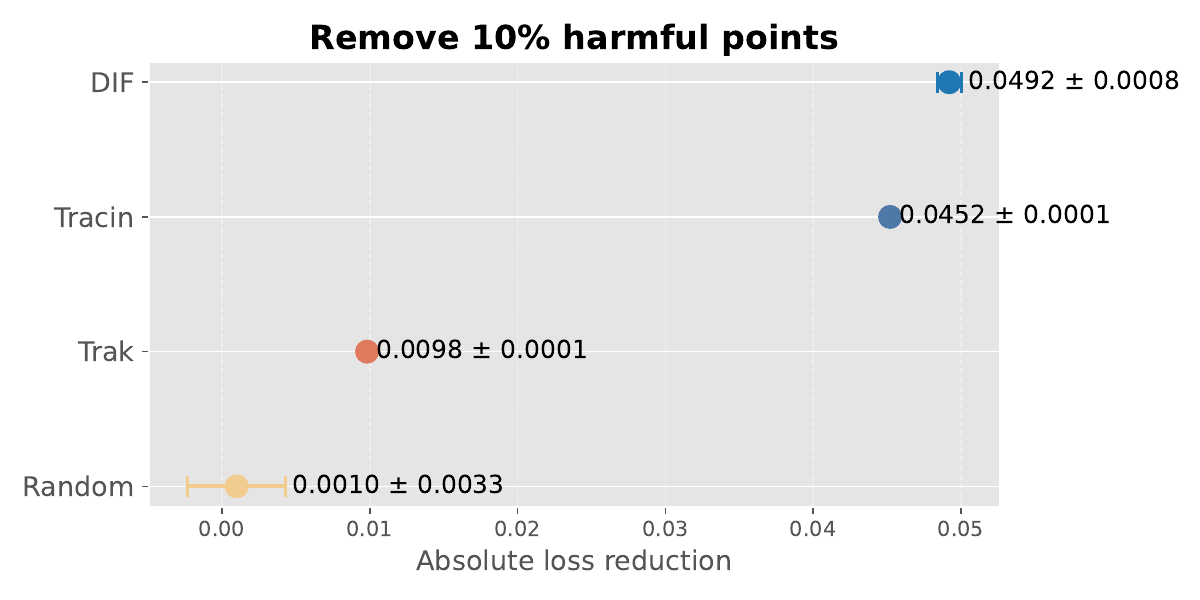}
        \caption{Remove 10\% harmful points}
        
    \end{subfigure}

    \vspace{0mm}
    \caption{Absolute loss reduction obtained by deleting the most harmful samples selected by different methods. The dot represents the mean loss reduction, and the horizontal bar is the standard deviation.}
    \label{fig:conter}
\end{figure}

\section{Conclusion and Limitations}\label{sec:6}

% This paper introduces DIF as a data attribution method for constrained learning. Classical IF may fail under constraints because it linearizes the unconstrained optimum, producing solution updates that may violate feasibility. In contrast, DIF respects feasibility by restricting solution changes to directions that remain valid under the currently active constraints.
% Methodologically, we develop a VI-based sensitivity analysis framework for constrained learning and use it to derive a first-order, feasibility-preserving estimator of solution changes. Empirically, on constrained linear regression and fairness-constrained CNNs, DIF closely tracks leave-one-out retraining, whereas IF and penalty-based IF exhibit substantial bias or infeasibility. Looking ahead, our DIF framework may be applied to data poisoning attacks, machine unlearning, and online or distribution-shifted training with evolving constraint sets. 

This paper introduces DIF as a constraint-aware data attribution method for constrained learning. 
Using VI-based sensitivity analysis, DIF derives a first-order, feasibility-preserving estimator of solution changes. 
Experiments on constrained linear regression and fairness-constrained CNNs show that DIF closely tracks leave-one-out retraining, whereas IF and penalty-based IF exhibit substantial bias or infeasibility. 
DIF may benefit applications where constraint violations are costly, such as machine unlearning, PINNs, robotics, and autonomous driving. 
As a limitation, DIF inherits the Hessian-related computational costs of classical IF. in large neural networks. However, existing IF acceleration techniques such as K-FAC and EK-FAC~\citep{grosse2023studying} can also be applied to speed up DIF.

\begingroup
\small
\bibliographystyle{plainnat}
\bibliography{example_paper}
\endgroup

\newpage
\appendix

\section{Technical appendices and supplementary material}\label{app:technical_appendix}

\begin{itemize}
  \item \hyperref[related work]{A.1\;Related Work}

  \item \hyperref[app:penalty-if-results]{A.2.\;Additional Penalty-IF Comparison Results}

  \item \hyperref[preliminary:cones]{A.3.\;Preliminaries: Cones and Variational Geometry}

  \item \hyperref[app:proofs]{A.4.\;Proofs and Auxiliary Results}
  \begin{itemize}
    \item \hyperref[app:auxlemmas]{A.4.1.\;Auxiliary Lemmas and Notation}
    \item \hyperref[proof of thm5]{A.4.2.\;Proof of Theorem~\ref{thm:5}}
    \item \hyperref[proof of propostion6]{A.4.3.\;Proof of Proposition~\ref{pro:aux vi}}
    \item \hyperref[app:positivehomo]{A.4.4.\;Proof of Corollary~\ref{corollary2}}
    \item \hyperref[proofof7]{A.4.5.\;Proof of Proposition~\ref{proposition7}}
    \item \hyperref[proof of DIF error]{A.4.6.\;Proof of Proposition~\ref{thm:DIF-error}}
    \item \hyperref[proofof8]{A.4.7.\;Proof of Proposition~\ref{thm:Auxiliary problem}}
    \item \hyperref[proofofdif-if]{A.4.8.\;Proof of Proposition~\ref{prop:dif-if}}
  \end{itemize}

  \item \hyperref[appendixoftoy]{A.5.\;Details in Toy Example}
  \begin{itemize}
    \item \hyperref[appendixoftoy]{A.5.1.\;Classical IF: Ignoring Constraints}
    \item \hyperref[appendixoftoy]{A.5.2.\;DIF}
  \end{itemize}

  \item \hyperref[appendix:constrained linear regression]{A.6.\;Details in Constrained Linear Regression}

  \item \hyperref[app:qp-complexity]{A.7.\;Computational Complexity and Scalable Implementation of the QP Solver}
  \begin{itemize}
    \item \hyperref[app:ipm]{A.7.1.\;Computational Complexity of the DIF QP}
    \item \hyperref[app:scalable-qp]{A.7.2.\;Scalable QP Solver with Structured Curvature Approximation}
  \end{itemize}

  \item \hyperref[moreexp]{A.8.\;More Experiments}
  \begin{itemize}
    \item \hyperref[AppendixPINN]{A.8.1.\;Machine Unlearning of Physics-Informed Neural Networks}
    \item \hyperref[app:fairface_large_scale]{A.8.2.\;Additional Large-Scale Validation on FairFace}
  \end{itemize}
   \item \hyperref[app:counterfactual-details]{A.9.\;Details of Counterfactual Retraining Evaluation}
\end{itemize}

\subsection{Related work}\label{related work}

IF was first introduced by \citet{hampel1974influence} to study the robustness of statistical estimators under distributional shifts. Decades later, \citet{koh2017understanding} adapted IF to machine learning to analyze the impact of individual training samples on model behavior. Since then, numerous extensions have emerged, including faster and more scalable inference methods \citep{schioppa2022scaling, guo2020fastif}, and extensions to domains such as time-series forecasting \citep{zhang2024timeinf} and generative modeling \citep{terashita2021influence, kong2021understanding, lin2024diffusion}. Beyond these, IF has been applied to practical tasks such as defending against data poisoning \citep{steinhardt2017certified}, improving interpretability \citep{han2020explaining}, and supporting machine unlearning \citep{wang2025machine, warnecke2021machine}. However, because IF fundamentally relies on the implicit function theorem \citep{krantz2002implicit}, most methods require strong convexity of the model. Recent empirical studies \citep{bae2022if} further highlight this limitation, showing that IF estimates can be fragile and often misaligned with their actual effects in practice \citep{basu2020influence, ghorbani2019interpretation}.

Constrained learning has gained increasing attention due to the need for machine learning models that are fair~\citep{quadrianto2019discovering,mehrabi2021survey}, robust~\citep{zhang2019theoretically}, and consistent with domain-specific priors~\citep{borghesi2020improving}. For instance, invariant learning aims to ensure stable performance under distribution shifts or data transformations~\cite{benton2020learning,immer2022invariance}, and PINNs incorporate traffic flow theory into transportation forecasting~\cite{shi2021physics}. These settings highlight a growing class of learning problems where solutions must satisfy explicit feasibility requirements beyond minimizing prediction~\citep{donini2018empirical,robey2021adversarial,hounie2023resilient,li2023double,shen2022agnostic}. Although adding a regularization or penalty term to the objective function is common practice in ML, this approach suffers from inherent limitations. First, the choice of trade-off parameters is highly non-trivial, often requiring heuristic tuning without guarantees of correctness. Second, the gradients of the prediction loss and the penalty term rarely align; as a result, gradient-based optimization must continually balance these two gradient directions, which may lead to suboptimal updates~\citep{hwang2024dual,liu2024config}.

Given the wide application of constrained learning, it is essential to develop data attribution methods that explicitly account for feasibility. Classical influence-function analyses have considered constraints, but only in limited settings.
\cite{ghosh2014influence_preprint} studied restricted minimum-divergence estimators and derived closed-form IF expressions under smooth equality constraints, while \cite{vasconcellos2009influence} extended Cook’s local-influence framework to perturbations constrained by homogeneous linear restrictions. In this paper, we extend the classical IF framework to general constrained learning by analyzing solution changes within the feasible region. Specifically, we reformulate the optimality conditions as a VI and study its sensitivity with respect to data weights. This formulation naturally admits both primal and dual sensitivities, allowing us to capture how perturbations affect not only the model parameters but also the associated Lagrangian multipliers. Compared to the KKT system, the VI method offers a more general analysis framework, together with a more compact geometric language for describing the coupled dynamics of primal and dual variables by characterizing solution changes within the tangent cone of the feasible region.
\subsection{Additional Penalty-IF Comparison Results}
\label{app:penalty-if-results}

Table~\ref{tab:if_dif_retrain_appendix} reports representative cases from the 100 influential samples used in the constrained CNN experiment. 
Compared with DIF, penalty-IF exhibits larger deviations from the actual leave-one-out retraining loss changes, which is consistent with its lower Pearson correlation reported in the main text.

\begin{table}[ht]
\centering
\caption{Representative comparison of penalty-IF, DIF, and actual retraining loss changes on the constrained CNN experiment.}
\label{tab:if_dif_retrain_appendix}
\begin{tabular}{cccc}
\toprule
\textbf{Sample ID} & \textbf{Predicted (Penalty-IF)} & \textbf{Predicted (DIF)} & \textbf{Actual (Retrain)} \\
\midrule
4221 &  0.005528  &  0.004979 &  0.003664 \\
215  &  0.000016  &  0.004249 &  0.003015 \\
1328 & -0.009318 & -0.003962& -0.003139 \\
4668 & 0.013700  &  0.003922 &  0.004601 \\
772  & 0.001171  &  0.003030 &  0.003509 \\
4473 & 0.005000  &  0.002985 &  0.002487 \\
4144 & -0.007275 & -0.002820 & -0.000704 \\
2828 & 0.002719  &  0.002775 &  0.002441 \\
3505 & 0.001118  &  0.002747 &  0.003592 \\
1163 & 0.001183  &  0.002359 &  0.002551 \\
\bottomrule
\end{tabular}
\end{table}
\subsection{Preliminary}\label{preliminary:cones}
\begin{definition}[Polar Cone; cf.~\cite{dontchev2009implicit}]
Let $K \subseteq \mathbb{R}^n$ be a closed convex cone.  
The \emph{polar cone} of $K$ is defined as
\begin{equation}
    K^* = \{\, y \in \mathbb{R}^n \;\mid\; \langle x, y \rangle \leq 0 \;\; \forall x \in K \,\}.
\end{equation}
\end{definition}

Moreover, the normal vectors to $K$ and $K^*$ satisfy
\begin{equation}\label{22}
    y \in N_K(x) 
    \;\;\Longleftrightarrow\;\; 
    x \in N_{K^*}(y) 
    \;\;\Longleftrightarrow\;\;
    x \in K,\; y \in K^*,\; \langle x, y \rangle = 0.
\end{equation}

\begin{definition}[Tangent Cone, cf.~\cite{rockafellar1998variational}]
For a set $C \subset \mathbb{R}^n$ (not necessarily convex) and a point $x \in C$, 
a vector $v$ is said to be tangent to $C$ at $x$ if
\begin{equation}
    \frac{1}{\tau^k}(x^k - x) \;\to\; v 
    \quad \text{for some } x^k \to x, \; x^k \in C, \; \tau^k \downarrow 0.
\end{equation}
The set of all such vectors $v$ is called the \emph{tangent cone} to $C$ at $x$ 
and is denoted by $T_C(x)$. 
For $x \notin C$, we take $T_C(x) = \emptyset$.
\end{definition}

\begin{definition}[Critical Cone, cf.~\cite{rockafellar1998variational}]
For a convex set $C$, any $x \in C$ and any $v \in N_C(x)$, 
the \emph{critical cone} to $C$ at $x$ for $v$ is
\begin{equation}
    K_C(x,v) = \{\, w \in T_C(x) \mid w \perp v \,\}.
\end{equation}
\end{definition}

\begin{definition}[Critical Subspaces]\label{def:critical subspaces}
For a convex set $C\subset\mathbb{R}^n$ and $(x,v)$ with $v\in N_C(x)$, 
let $K_C(x,v)$ denote the critical cone at $(x,v)$.  
Then the associated \emph{critical subspaces} are defined as
\begin{equation}
K_C^+(x,v) \;=\; K_C(x,v)-K_C(x,v)
\;=\; \{\,w-w' \mid w,w'\in K_C(x,v)\,\},
\end{equation}
\begin{equation}
K_C^-(x,v) \;=\; K_C(x,v)\cap[-K_C(x,v)]
\;=\; \{\,w\in K_C(x,v)\mid -w\in K_C(x,v)\,\}.
\end{equation}
Here $K_C^+(x,v)$ is the smallest linear subspace containing $K_C(x,v)$,
and $K_C^-(x,v)$ is the largest linear subspace contained in $K_C(x,v)$.
\end{definition}

\begin{lemma}[Reduction Lemma;cf.~2E.4~\cite{dontchev2009implicit}]\label{lem:reduction}
Let \(C \subset \mathbb{R}^n\) be a convex set and 
\(\bar{x}\in C\) with \(\bar{v}\in N_C(\bar{x})\). 
Define the critical cone \(K_C := K_C(\bar{x},\bar{v})\).  
Then, for all sufficiently small \(w,u \in \mathbb{R}^n\), we have the equivalence
\begin{equation}
\bar{v}+\Delta v \in N_C(\bar{x}+\Delta x)
\ \Longleftrightarrow\ 
\Delta v \in N_{K_C}(\Delta x).
\end{equation}
\end{lemma}

\subsection{Proofs}\label{app:proofs}
\subsubsection{Auxiliary Lemmas}\label{app:auxlemmas}

For clarity, we first introduce some notation and rewrite the generalized equation in a standard form. 

\paragraph{Notation.}
Let 
\(\eta = (\theta,\lambda)\) denote the parameter vector, 
\(\bar{\eta} = (\bar{\theta},\bar{\lambda})\) be its reference value, 
and 
\(\Delta \hat\eta = (\Delta\hat{\theta},\Delta\hat{\lambda})\) be the estimation to $\Delta \eta=(\Delta \theta, \Delta \lambda)$.

\paragraph{Solution mapping.}
We define the solution mapping as
\begin{equation}\label{eta25}
S(\varepsilon)
:= \{\,(\theta,\lambda)\mid f(\varepsilon,\theta,\lambda)+N_E(\theta,\lambda)\ni 0 \,\}
   = \{\,\eta\mid f(\varepsilon,\eta)+N_E(\eta)\ni 0\,\},
\end{equation}
where \(N_E(\eta)\) denotes the normal cone to \(E\) at \(\eta\).

\paragraph{Generalized equation form.}
To study the local behavior of \(S\), we rewrite the system as the generalized equation
\begin{equation}
G(\eta) := f(\bar{\varepsilon},\bar{\eta})
        + \nabla_\eta f(\bar{\varepsilon},\bar{\eta})(\eta-\bar{\eta})
        + N_E(\eta),
\end{equation}
with
\begin{equation}
\nabla_\eta f(\bar{\varepsilon},\bar{\eta})=\nabla_{(\theta,\lambda)} f(\bar{\theta},\bar{\lambda},\bar{x}) =: A.
\end{equation}
At the reference point, the optimality condition~(\ref{new VI}) ensures that
\begin{equation}
G(\bar{\eta})\ni 0.
\end{equation}

\paragraph{Linearization.}
We then consider the linearized generalized equation (the first-order approximation of \(G\)):
\begin{equation}\label{g0}
G_0(\Delta\eta)
:= \nabla_\eta f(\bar{\varepsilon},\bar{\eta})\Delta\eta + N_K(\Delta\eta)
= A\Delta\eta + N_K(\Delta\eta),
\qquad G_0(0)\ni 0,
\end{equation}
where \(K:=K_E(\bar{\eta},-f(\bar{\varepsilon},\bar{\eta}))\) is the critical cone. This coincides with \( K= \mathbb{R}^d \times D\) as introduced in Proposition~\ref{pro:aux vi}.

It is convenient to define the inverse-type mapping
\begin{equation}
\bar{s} := G_0^{-1}=(A+N_K)^{-1},
\end{equation}
which solves the linearized inclusion~(\ref{g0}).

Given perturbation \(u\Delta\bar{\varepsilon}\), 
\begin{equation}\label{barssolution}
\bar{s}(-u\Delta\bar{\varepsilon})
= (A+N_K)^{-1}(-u\Delta\bar{\varepsilon})
= \{\Delta \hat\eta \mid A\Delta\hat\eta + N_K(\Delta\hat\eta) + u\Delta\bar{\varepsilon} \ni 0\}.
\end{equation}
In other words, \(\bar{s}(-u\Delta\bar{\varepsilon})\) yields the solution 
to the auxiliary VI~(\ref{aux vi}).

\paragraph{Critical subspace of $E$.}
Following the definition~\ref{def:critical subspaces}, the critical subspaces of $E$ at $(\bar\theta,\bar\lambda,-f(\bar\varepsilon,\bar\theta,\bar \lambda))$ are defined as
\begin{align}
K_E^+(\bar\theta,\bar\lambda,-f(\bar\varepsilon,\bar\theta,\bar \lambda))
&=
\mathbb{R}^d \times K_{\mathbb{R}^m}^+(\bar\lambda, \nabla_\lambda L(\bar\varepsilon,\bar\theta,\bar \lambda)),\\[2mm]
K_E^-(\bar\theta,\bar\lambda,-f(\bar\varepsilon,\bar\theta,\bar \lambda))
&=
\mathbb{R}^d \times K_{\mathbb{R}^m}^-(\bar\lambda, \nabla_\lambda L(\bar\varepsilon,\bar\theta,\bar \lambda)).
\end{align}
\begin{equation}\label{K-}
\Delta\hat\eta=(\Delta\hat\theta,\Delta\hat\lambda)\ \text{satisfies}\ 
\Delta\hat\eta \in K_E^+
\ \Longleftrightarrow\ 
\Delta\hat\lambda_j = 0,\ \forall j\in I_{\text{Inactive}},\ \text{and } 
\Delta\hat\theta\in\mathbb{R}^d.
\end{equation}

\begin{lemma}\label{lamma:assumption of tm5}
Under the assumptions of Theorem~\ref{thm:5},  
for any $\Delta\hat\eta$ satisfying the linearized VI~(\ref{aux vi}),  
if

\begin{equation}
\Delta\hat\eta\in K_E^+,\quad 
\Delta\hat\eta\neq 0,\quad 
A\Delta\hat\eta\perp K_E^-,
\end{equation}

then
\[
\langle \Delta\hat\eta,\; A\Delta\hat\eta\rangle > 0.
\]
\end{lemma}

\begin{proof}
For any $\Delta\hat\eta$ satisfying the linearized VI~(\ref{aux vi}), 
it follows from the definition of the space $D$ that 
$\Delta\hat\lambda_j=0$ for all $j$ in the inactive set.
Therefore, any $\Delta\hat\eta$ satisfying VI~(15) also satisfies 
$\Delta\hat\eta\in K_E^+$.

\medskip
We next analyze the orthogonality condition 
\(A\Delta\hat\eta \perp K_E^-\).
By expanding \(A(\Delta\hat\theta,\Delta\hat\lambda)\) and splitting the \(\theta\)- and \(\lambda\)-parts, we obtain:
\begin{align*}
A\Delta\hat{\eta}\perp K_E^- 
&\ \Longleftrightarrow\ 
A(\Delta\hat\theta,\Delta\hat\lambda)\perp K_E^- \\[2mm]
&\ \Longleftrightarrow\ 
\bigl(H\Delta\hat\theta+\nabla_{\theta\lambda}L(\bar{\varepsilon}, \bar{\theta}, \bar{\lambda})\,\Delta\hat\lambda,\;
-\nabla_{\lambda\theta}L(\bar{\varepsilon}, \bar{\theta}, \bar{\lambda})\,\Delta\hat\theta \bigr)\perp K_E^- \\[2mm]
&\ \Longleftrightarrow\ 
H\Delta\hat\theta+\nabla_{\theta\lambda}L(\bar{\varepsilon}, \bar{\theta}, \bar{\lambda})\,\Delta\hat\lambda \perp \mathbb{R}^d,\quad 
-\nabla_{\lambda\theta}L(\bar{\varepsilon}, \bar{\theta}, \bar{\lambda})\,\Delta\hat\theta \perp 
K_{\mathbb{R}^m}^{-}\!\left(\bar{\lambda}, \nabla_\lambda L(\bar{\varepsilon}, \bar{\theta}, \bar{\lambda})\right) \\[2mm]
&\ \Longleftrightarrow\ 
H\Delta\hat\theta+\nabla_{\theta\lambda}L(\bar{\varepsilon}, \bar{\theta}, \bar{\lambda})\,\Delta\hat\lambda = 0,\quad 
-\nabla_{\lambda\theta}L(\bar{\varepsilon}, \bar{\theta}, \bar{\lambda})\,\Delta\hat\theta \perp 
K_{\mathbb{R}^m}^{-}\!\left(\bar{\lambda}, \nabla_\lambda L(\bar{\varepsilon}, \bar{\theta}, \bar{\lambda})\right).
\end{align*}

\medskip
Moreover, by the definition of the critical subspace, we have
\begin{equation}
\Delta \hat\lambda \in 
K_{\mathbb{R}^m}^{-}\!\left(\bar{\lambda}, \nabla_\lambda L(\bar{\varepsilon}, \bar{\theta}, \bar{\lambda})\right)
\ \Longleftrightarrow\ 
\Delta \hat{\lambda}_j = 0,\ \forall j \in I_{\text{non-binding}} \cup I_{\text{Inactive}},
\end{equation}
where \(I_{\text{Inactive}}\) and \(I_{\text{non-binding}}\) denote the sets of inactive and non-binding constraints, respectively.

\medskip
Consequently, the condition 
\[
-\nabla_{\lambda \theta} L(\bar{\varepsilon}, \bar{\theta}, \bar{\lambda})\, \Delta \hat{\theta} 
\ \perp\ 
K_{\mathbb{R}^m}^{-}\!\left(\bar{\lambda}, \nabla_\lambda L\right)
\]
only requires that
\[
\nabla_{\lambda_j \theta} L\, \Delta \hat{\theta} = 0,
\qquad \forall j \in I_{\text{binding}}.
\]
That is,
\begin{align}\label{AK}
A\Delta\hat{\eta}\perp K_E^- 
&\ \Longleftrightarrow\ 
H\Delta\hat\theta+\nabla_{\theta\lambda}L(\bar{\varepsilon}, \bar{\theta}, \bar{\lambda})\,\Delta\hat\lambda = 0,\quad 
\nabla_{\lambda_j \theta} L\, \Delta \hat{\theta} = 0,\ \forall j \in I_{\text{binding}}.
\end{align}

\medskip
Notice that
\begin{equation}
    \nabla_{\lambda_j \theta} L \,\Delta \hat{\theta}=0 
    \ \Longleftrightarrow\ 
    \Delta \hat{\theta} \perp 
    \frac{1}{N_j} \sum_{i=1}^{N_j} \nabla_\theta \ell_j\left(z_i^{(j)}, \bar{\theta}\right),
    \qquad \forall j \in I_{\text{binding}}.
\end{equation}

\medskip
Assumption~2 in Theorem~\ref{thm:5} implies that 
for all $\Delta \hat{\theta}$ satisfying~(\ref{AK}), we have 
\begin{equation}\label{assump2}
\left\langle\Delta \hat{\theta}, \nabla_{\theta \theta}^2 L(\bar{\varepsilon}, \bar{\theta}, \bar{\lambda}) \Delta \hat{\theta}\right\rangle>0.
\end{equation}
Moreover, substituting 
\[
A:=\nabla_{(\theta, \lambda)} f(\bar{\varepsilon}, \bar{\theta}, \bar{\lambda})
=\begin{bmatrix}
\nabla_\theta^2 L(\bar{\varepsilon}, \bar{\theta}, \bar{\lambda}) & \nabla_{\theta \lambda} L(\bar{\varepsilon}, \bar{\theta}, \bar{\lambda}) \\
-\nabla_{\theta \lambda} L(\bar{\varepsilon}, \bar{\theta}, \bar{\lambda}) & -\nabla_\lambda^2 L(\bar{\varepsilon}, \bar{\theta}, \bar{\lambda})
\end{bmatrix},
\] to (\ref{assump2}),

we obtain
\begin{equation}
\langle \Delta\hat\eta,\; A\Delta\hat\eta\rangle 
=\langle (\Delta\hat\theta,\Delta\hat\lambda),\, A(\Delta\hat\theta,\Delta\hat\lambda) \rangle
=
\langle \Delta\hat\theta,\, \nabla^2_{\theta\theta} L(\bar\varepsilon,\bar\theta,\bar\lambda)\,\Delta\hat\theta \rangle.
\end{equation}
Combining the two results yields 
\(\langle\Delta \hat{\eta}, A \Delta \hat{\eta}\rangle>0\).

\end{proof}

\begin{lemma}\label{lemma15}
Under the assumptions of Theorem~\ref{thm:5}, 
$\bar{s} := (A+N_K)^{-1}$ is everywhere single-valued.  
Equivalently, there exists a unique solution 
$\Delta\hat\eta$ to the auxiliary VI~(\ref{aux vi}) 
given $(\bar\varepsilon,\bar\theta,\bar \lambda)$ and $\Delta \hat{\varepsilon}$.
\end{lemma}

\begin{proof}
We prove that $G_0^{-1} := (A+N_K)^{-1}$ is single-valued everywhere.

Assume, for contradiction, that there exist 
two distinct solutions $\Delta\hat\eta_1$ and $\Delta\hat\eta_2$
such that $r=G_0^{-1}(\Delta \hat \eta_1)=G_0^{-1}(\Delta \hat \eta_2)$. Then,
\[
\Delta\hat\eta_1, \Delta\hat\eta_2 \in K, 
\qquad 
r - A\Delta\hat\eta_1 \in N_K(\Delta\hat\eta_1),
\qquad
r - A\Delta\hat\eta_2 \in N_K(\Delta\hat\eta_2).
\]

This implies, by the definition of the normal cone, that
\begin{equation}
\langle \Delta\hat\eta_1,\, r - A\Delta\hat\eta_2 \rangle \le 0,
\qquad 
\langle \Delta\hat\eta_2,\, r - A\Delta\hat\eta_1 \rangle \le 0.
\end{equation}

Note that \(K\) is the critical cone.  
By condition~(\ref{22}), we have
\begin{equation}
\begin{aligned}
\Delta \hat{\eta}_1 &\in K, \quad 
r - A \Delta \hat{\eta}_1 \in K^*, \quad
\langle \Delta \hat{\eta}_1,\; r - A \Delta \hat{\eta}_1 \rangle = 0, \\[2mm]
\Delta \hat{\eta}_2 &\in K, \quad 
r - A \Delta \hat{\eta}_2 \in K^*, \quad
\langle \Delta \hat{\eta}_2,\; r - A \Delta \hat{\eta}_2 \rangle = 0.
\end{aligned}
\end{equation}

By the definition of the polar cone $K^* = (K^-)^\perp$, 
We also have
\[
-A(\Delta\hat\eta_1-\Delta\hat\eta_2) \in K^* - K^* = (K^-)^\perp,
\qquad 
\Delta\hat\eta_1 - \Delta\hat\eta_2 \in K - K = K^+.
\]

Consider the following inner product:
\begin{align*}
\langle \Delta\hat\eta_1 - \Delta\hat\eta_2,\; A(\Delta\hat\eta_1-\Delta\hat\eta_2) \rangle
&= \langle \Delta\hat\eta_1 - \Delta\hat\eta_2,\; [r - A\Delta\hat\eta_2] - [r - A\Delta\hat\eta_1] \rangle \\
&= \langle \Delta\hat\eta_1,\; r - A\Delta\hat\eta_2 \rangle 
   - \langle \Delta\hat\eta_1,\; r - A\Delta\hat\eta_1 \rangle \\
&\quad - \langle \Delta\hat\eta_2,\; r - A\Delta\hat\eta_2 \rangle
   + \langle \Delta\hat\eta_2,\; r - A\Delta\hat\eta_1 \rangle \\
&\le 0.
\end{align*}

Following Lemma~\ref{lamma:assumption of tm5}, Assumption~2 of Theorem~\ref{thm:5} ensures that 
for any nonzero $\Delta\hat\eta\in K^+$ with 
$A\Delta\hat\eta\perp K^-$, we must have 
$\langle \Delta\hat\eta,\; A\Delta\hat\eta\rangle > 0$. 
This contradicts the above inequality unless 
$\Delta\hat\eta_1=\Delta\hat\eta_2$.

Therefore, the solution must be unique, which proves that 
$G_0^{-1}=(A+N_K)^{-1}$ is everywhere single-valued.
\end{proof}

\begin{remark}\label{remark:single_lip}
According to Dontchev and Rockafellar~\cite[Chapter~2E]{dontchev2009implicit}, 
Consider the generalized equation 
\(
G_0 = A + N_K
\),
where \(A\) is a linear operator and \(N_K\) is the normal cone mapping of a closed convex cone \(K\).
If the inverse mapping \(G_0^{-1}\) is \emph{everywhere single-valued},
then \(G_0^{-1}\) is (globally) Lipschitz continuous.
\end{remark}

\begin{lemma}[Relation between $G^{-1}$ and $G_0^{-1}$]\label{lem:G-vs-G0}

Then for all $\Delta\eta$ sufficiently close to $0$, 
the inverse mappings $G^{-1}$ and $G_0^{-1}$ satisfy the following relationship:
\begin{equation}
G^{-1}(\gamma)=G_0^{-1}(\gamma)+\bar{\eta},\qquad \gamma\ \text{near } 0.
\end{equation}
\end{lemma}

\begin{proof}
Consider any $\Delta\eta$ close to $0$.  
We first expand the generalized equation $G(\bar{\eta}+\Delta\eta)$ as
\begin{align*}
G(\bar{\eta}+\Delta\eta)
&= f(\bar{\varepsilon},\bar{\eta}) 
+ \nabla_{\eta} f(\bar{\varepsilon},\bar{\eta})\,\Delta\eta
+ N_E(\bar{\eta}+\Delta\eta)\\[1mm]
&= f(\bar{\varepsilon},\bar{\eta})
-\nabla_{\eta} f(\bar{\varepsilon},\bar{\eta})\,\bar{\eta}
+ \nabla_{\eta} f(\bar{\varepsilon},\bar{\eta})(\bar{\eta}+\Delta\eta)
+ N_E(\bar{\eta}+\Delta\eta).
\end{align*}

Let $A := \nabla_{\eta} f(\bar{\varepsilon},\bar{\eta})$.  
Then $\gamma \in G(\bar{\eta}+\Delta\eta)$ if and only if
\begin{equation}
\gamma \in f(\bar{\varepsilon},\bar{\eta})
- A\bar{\eta}
+ A(\bar{\eta}+\Delta\eta)
+ N_E(\bar{\eta}+\Delta\eta).
\end{equation}

Rearranging the terms, this is equivalent to
\begin{equation}
\gamma - f(\bar{\varepsilon},\bar{\eta}) + A\bar{\eta}
\in A(\bar{\eta}+\Delta\eta) + N_E(\bar{\eta}+\Delta\eta),
\end{equation}
which can be further rewritten as
\begin{equation}
\gamma - f(\bar{\varepsilon},\bar{\eta})
\in A\Delta\eta + N_E(\bar{\eta}+\Delta\eta).
\end{equation}

Note that $- f(\bar{\varepsilon},\bar{\eta}) \in N_E(\bar \eta)$. By the \emph{Reduction Lemma}~\ref{lem:reduction}, this holds if and only if
\begin{equation}
\gamma \in A\Delta\eta + N_K(\Delta\eta),
\end{equation}
which is exactly
\begin{equation}
\gamma \in G_0(\Delta\eta).
\end{equation}
\end{proof}

\begin{lemma}\label{G is first order approximation}
Define 
\[
\sigma(\varepsilon):=G^{-1}\!\big(-\nabla_\varepsilon f(\bar\varepsilon,\bar\eta)\,(\varepsilon-\bar\varepsilon)\big)=G^{-1}\!\big(-u(\varepsilon-\bar\varepsilon)\big).
\]
If $G^{-1}$ admits a single-valued Lipschitz localization around $0$ for $\bar\eta$,
then $\sigma(\varepsilon)$ is a first-order approximation of $S(\varepsilon)$ at $\bar\varepsilon$.
\end{lemma}

\begin{proof}
Since $f$ is strictly differentiable at $(\bar\varepsilon,\bar\eta)$, for any small perturbations $(\Delta\varepsilon,\Delta\eta)$ we have
\begin{equation}
f(\bar\varepsilon+\Delta\varepsilon,\bar\eta+\Delta\eta)
=
f(\bar\varepsilon,\bar\eta)
+\nabla_\varepsilon f(\bar\varepsilon,\bar\eta)\,\Delta\varepsilon
+\nabla_\eta f(\bar\varepsilon,\bar\eta)\,\Delta\eta
+o(\Delta\varepsilon,\Delta\eta),
\end{equation}
where $\|o(\Delta\varepsilon,\Delta\eta)\|$ denotes the higher-order remainder term 
satisfying 
$\|o(\Delta\varepsilon,\Delta\eta)\|=o(\|\Delta\varepsilon\|+\|\Delta\eta\|)$.

By (\ref{eta25}),
\begin{equation}
\begin{aligned}
S(\bar\varepsilon+\Delta\varepsilon)
&=\{\eta\mid f(\bar\varepsilon+\Delta\varepsilon,\eta)+N_E(\eta)\ni 0\} \\
&=\{\eta\mid  f(\bar{\varepsilon}, \bar{\eta})+\nabla_{\varepsilon} f(\bar{\varepsilon}, \bar{\eta}) \Delta \varepsilon+\nabla_\eta f(\bar{\varepsilon}, \bar{\eta})(\eta-\bar{\eta})+o(\Delta \varepsilon, \eta-\bar{\eta})+N_E(\eta)\ni 0\}.\\
\end{aligned}
\end{equation}

Then, \(\eta \in S(\bar{\varepsilon}+\Delta\varepsilon)\) satisfieså that
\begin{equation}
    f(\bar{\varepsilon}, \bar{\eta})
    +\nabla_{\varepsilon} f(\bar{\varepsilon}, \bar{\eta}) \,\Delta \varepsilon
    +\nabla_\eta f(\bar{\varepsilon}, \bar{\eta})(\eta-\bar{\eta})
    +o(\Delta \varepsilon, \eta-\bar{\eta})
    +N_E(\eta)
    \;\ni\; 0.
\end{equation}

Rearranging the terms, we obtain the equivalent inclusion
\begin{equation}
    -\nabla_{\varepsilon} f(\bar{\varepsilon}, \bar{\eta}) \,\Delta \varepsilon
    -o(\Delta \varepsilon, \eta-\bar{\eta})
    \;\in\;
    f(\bar{\varepsilon}, \bar{\eta})
    +\nabla_\eta f(\bar{\varepsilon}, \bar{\eta})(\eta-\bar{\eta})
    +N_E(\eta)
    \;=\; G(\eta),
\end{equation}
where \(G(\eta):=\nabla_\eta f(\bar{\varepsilon},\bar{\eta})\,(\eta-\bar{\eta})+f(\bar{\varepsilon},\bar{\eta})+N_E(\eta)\) 
denotes the linearized generalized equation in \(\eta\).

Therefore, any solution \(\eta\in S(\bar\varepsilon+\Delta\varepsilon)\) can be written as:
\begin{equation}
    \eta 
    \;=\; 
    G^{-1}\!\left(
    -\nabla_{\varepsilon} f(\bar{\varepsilon}, \bar{\eta}) \,\Delta \varepsilon
    -o(\Delta \varepsilon, \eta-\bar{\eta})
    \right),
\end{equation}
which expresses \(S(\bar\varepsilon+\Delta\varepsilon)\) implicitly via the inverse mapping \(G^{-1}\).

We use $\kappa$ to denote the Lipschitz constant of $G^{-1}$. Since $\bar\eta=G(0)$, we have:

we have
\begin{equation}\label{eq:lipschitz-eta2}
\begin{aligned}
\|\eta-\bar\eta\|&=\| G^{-1}\!\left(
    -\nabla_{\varepsilon} f(\bar{\varepsilon}, \bar{\eta}) \,\Delta \varepsilon
    -o(\Delta \varepsilon, \eta-\bar{\eta})
    \right)-G^{-1}(0)\|\\
&\le k\,\big\|-\nabla_{\varepsilon} f(\bar\varepsilon,\bar\eta)\,\Delta\varepsilon
- o(\Delta\varepsilon,\eta-\bar\eta)\big\|.
\end{aligned}
\end{equation}
Applying the triangle inequality gives
\begin{equation}
\|\eta-\bar\eta\|
\le k\,\|\nabla_{\varepsilon} f(\bar\varepsilon,\bar\eta)\|\,\|\Delta\varepsilon\|
+ k\,\|o(\Delta\varepsilon,\eta-\bar\eta)\|.
\end{equation}

Since $o(\Delta\varepsilon,\eta-\bar\eta)=o(\|\Delta\varepsilon\|+\|\eta-\bar\eta\|)$ as $(\Delta\varepsilon,\eta-\bar\eta)\to (0,0)$, 
for any $\delta>0$ there exists a neighborhood of $(0,0)$ such that 
\begin{equation}\label{eq:small-o}
\|o(\Delta\varepsilon,\eta-\bar\eta)\|
\le \delta\big(\|\Delta\varepsilon\|+\|\eta-\bar\eta\|\big).
\end{equation}

Combining (\ref{eq:lipschitz-eta2}) and (\ref{eq:small-o}) yields
\begin{equation}
\|\eta-\bar\eta\|
\le k\,\|\nabla_{\varepsilon} f(\bar\varepsilon,\bar\eta)\|\,\|\Delta\varepsilon\|
+ k\delta\,\|\Delta\varepsilon\|
+ k\delta\,\|\eta-\bar\eta\|.
\end{equation}
Rearranging terms gives
\begin{equation}\label{DIFerror1}
(1-k\delta)\,\|\eta-\bar\eta\|
\le k\big(\|\nabla_{\varepsilon} f(\bar\varepsilon,\bar\eta)\|+\delta\big)\,\|\Delta\varepsilon\|.
\end{equation}

Finally, choosing $\delta>0$ small enough such that $k\delta<\frac12$, we obtain

\begin{equation}\label{bigo}
\|\eta-\bar\eta\|
\le \frac{k}{\,1-k\delta\,}\big(\|\nabla_{\varepsilon} f(\bar\varepsilon,\bar\eta)\|+\delta\big)\,
\|\Delta\varepsilon\|
= O(\|\Delta\varepsilon\|),
\end{equation}

Since $G^{-1}$ is Lipschitz continuous,
\begin{equation}\label{smallo}
    \begin{aligned}
        S(\bar{\varepsilon}+\Delta{\varepsilon})-\sigma(\bar{\varepsilon}+\Delta{\varepsilon})&=\left\|G^{-1}(-\nabla_{\varepsilon} f(\bar{\varepsilon}, \bar{\eta}) \Delta \varepsilon-o(\Delta \varepsilon, \eta-\bar \eta))-G^{-1}(-\nabla_{\varepsilon} f(\bar{\varepsilon}, \bar{\eta}) \Delta \varepsilon)\right\| \\ &\leq k\|o(\Delta \varepsilon, \eta-\bar \eta)\|.
    \end{aligned}
\end{equation} 

Substituting (\ref{bigo}) to (\ref{smallo}) yields:
\begin{equation}\label{smallo}
    \begin{aligned}
        S(\bar{\varepsilon}+\Delta{\varepsilon})-\sigma(\bar{\varepsilon}+\Delta{\varepsilon})&\leq k\|o(\Delta \varepsilon, \eta-\bar \eta)\|=o(\|\Delta \varepsilon\|).
    \end{aligned}
\end{equation}

This proves that $\sigma(\varepsilon)$ is a first-order approximation of $S(\varepsilon)$ at $\bar\varepsilon$.
\end{proof}

\subsubsection{Proof of Theorem~\ref{thm:5}}\label{proof of thm5}

Following Lemma~\ref{lemma15} and Remark~\ref{remark:single_lip}, we establish that the mapping $G_0^{-1}$ is locally Lipschitz continuous.
Moreover, by Lemma~\ref{lem:G-vs-G0}, it holds that

$$
G^{-1}(v)=G_0^{-1}(v)+\bar{\eta}
$$

which implies that $G^{-1}$ inherits the Lipschitz continuity and single-valuedness of $G_0^{-1}$.
In this case, the Lipschitz continuity of $G^{-1}$ ensures the applicability of Lemma~\ref{G is first order approximation}, which shows that $\sigma(\varepsilon)$ serves as a first-order local approximation of the solution mapping $S(\varepsilon)$.

\medskip
Indeed, for small $\Delta\varepsilon$,
\begin{align}
\sigma(\bar\varepsilon+\Delta\bar\varepsilon)
&= G^{-1}\big(-\nabla_\varepsilon f(\bar\varepsilon,\bar\eta)\,\Delta\bar\varepsilon\big)\\[2mm]
&= G_0^{-1}\big(-\nabla_\varepsilon f(\bar\varepsilon,\bar\eta)\,\Delta\bar\varepsilon\big)+\bar\eta.
\end{align}
Therefore,
\begin{align}
S(\bar\varepsilon+\Delta\bar\varepsilon)-S(\bar\varepsilon)
&=\sigma(\bar\varepsilon+\Delta\bar\varepsilon)-\bar\eta + o(\|\Delta \varepsilon\|) \\
&=G_0^{-1}\!\big(-\nabla_\varepsilon f(\bar\varepsilon,\bar\eta)\,\Delta\bar\varepsilon\big)
+ o(\|\Delta\bar\varepsilon\|) \notag\\
&=\bar s \big(-\nabla_\varepsilon f(\bar\varepsilon,\bar\eta)\,\Delta\bar\varepsilon\big)
+ o(\|\Delta\bar\varepsilon\|).
\end{align}

We obtain
\begin{align}
\lim_{t\downarrow 0}\frac{S(\bar\varepsilon+t {\Delta\bar\varepsilon})-S(\bar\varepsilon)}{t}
&= \lim_{t\downarrow 0}\frac{\bar s \big(-\nabla_\varepsilon f(\bar\varepsilon,\bar\eta)\,t\Delta\bar\varepsilon\big)
+ o(\|\Delta\bar\varepsilon\|)}{t}\\[2mm]
\end{align}

Since $N_K$ is the normal cone mapping of the convex cone $K$, it satisfies $N_K(\alpha w)=\alpha N_K(w)$ for all $\alpha>0$. Together with the linearity of $A$, this gives $\left(A+N_K\right)(\alpha w)=\alpha\left(A+N_K\right)(w)$. Hence $\bar{s}=\left(A+N_K\right)^{-1}$ is positively homogeneous,
\begin{equation}
\lim_{t\downarrow 0}\frac{\bar s \big(-\nabla_\varepsilon f(\bar\varepsilon,\bar\eta)\,t\Delta\bar\varepsilon\big)
+ o(\|\Delta\bar\varepsilon\|)}{t}
= \bar{s}(-\nabla_\varepsilon f(\bar\varepsilon,\bar\eta)\,\Delta\bar\varepsilon).
\end{equation}

Hence, $S$ is directionally differentiable at $\bar\varepsilon$ with
\[
D S(\bar\varepsilon)(\Delta \bar\varepsilon) 
=\bar{s} \big(-\nabla_\varepsilon f(\bar\varepsilon,\bar\eta)\,\Delta \bar\varepsilon\big).
\]

Equation~(\ref{barssolution}) ensures that
$$
\bar{s}\left(-\nabla_{\varepsilon} f(\bar{\varepsilon}, \bar{\eta}) \Delta \bar{\varepsilon}\right)
$$
is the solution of the auxiliary VI~(\ref{aux vi}).
This implies that, under the assumptions of Theorem~\ref{thm:5},
the solution mapping $S$ is directionally differentiable at $\bar{\varepsilon}$,
and its directional derivative is given by the solution of the auxiliary VI~(\ref{aux vi}).
Since $S(\varepsilon)=(\theta, \lambda)$,
the DIF $D \theta(\bar{\varepsilon} ; \Delta \bar{\varepsilon})$ defined in Definition~\ref{def:DIF} corresponds to
the $\theta$-component of$
D S(\bar{\varepsilon})(\Delta \bar{\varepsilon}).$ This guarantees the existence of the DIF.

\subsubsection{Proof of Proposition~\ref{pro:aux vi}} \label{proof of propostion6}
\begin{proof}
Given $(\bar{\varepsilon}, \bar{\theta}, \bar{\lambda})$ satisfy the optimality condition~(\ref{new VI}), we have  
\begin{equation}\label{24}
\nabla_{\theta} L(\bar{\varepsilon}, \bar{\theta}, \bar{\lambda}) + N_{\mathbb{R}^d}(\bar{\theta}) \ni 0 ,
\quad
- \nabla_{\lambda} L(\bar{\varepsilon}, \bar{\theta}, \bar{\lambda}) + N_{\mathbb{R}^m_+}(\bar{\lambda}) \ni 0 .
\end{equation}

Following the definition of the normal cone (Def.~\ref{def:normal cone}), condition (\ref{24}) implies that:
\begin{equation}\label{25}
\theta \in \mathbb{R}^d , \nabla_{\theta} L(\bar{\varepsilon}, \bar{\theta}, \bar{\lambda}) = 0 ,
\end{equation}
\begin{equation}
\lambda \in \mathbb{R}^m_+ , -\nabla_{\lambda} L(\bar{\varepsilon}, \bar{\theta}, \bar{\lambda}) \le 0 .
\end{equation}
Note that ${R}^m_+$ is a closed, convex cone. Following condition~(\ref{22}), we have
\begin{equation}
\nabla_{\lambda} L(\bar{\varepsilon}, \bar{\theta}, \bar{\lambda}) \cdot \bar{\lambda} = 0 ,
\end{equation}
which is consistent with the complementary slackness condition in the KKT system.  

Now we derive the equivalent of VI~(\ref{vi-linearized}) and VI~(\ref{aux vi}). We assume the $\Delta \bar{\varepsilon}, \Delta \hat{\theta},\Delta \hat{\lambda}$ are sufficiently close to 0.

VI~(\ref{vi-linearized}) requires that

\begin{equation}\label{6a}
\begin{aligned}
&\nabla_{\theta} L(\bar{\varepsilon}, \bar{\theta}, \bar{\lambda})
+ \nabla_{\theta\varepsilon} L(\bar{\varepsilon}, \bar{\theta}, \bar{\lambda}) \Delta \bar{\varepsilon} 
+ \nabla_{\theta}^2 L(\bar{\varepsilon}, \bar{\theta}, \bar{\lambda}) \Delta \hat{\theta} 
+ \nabla_{\theta\lambda} L(\bar{\varepsilon}, \bar{\theta}, \bar{\lambda}) \Delta \hat{\lambda} + N_{\mathbb{R}^d}(\bar{\theta} + \Delta \hat{\theta}) \ni 0 \\[2ex]
&\Updownarrow \\[2ex]
&\bar{\theta} + \Delta \hat{\theta} \in \mathbb{R}^d,\quad
\nabla_{\theta} L(\bar{\varepsilon}, \bar{\theta}, \bar{\lambda})
+ \nabla_{\theta\varepsilon} L(\bar{\varepsilon}, \bar{\theta}, \bar{\lambda}) \Delta \bar{\varepsilon} + \nabla_{\theta}^2 L(\bar{\varepsilon}, \bar{\theta}, \bar{\lambda}) \Delta \hat{\theta}
+ \nabla_{\theta\lambda} L(\bar{\varepsilon}, \bar{\theta}, \bar{\lambda}) \Delta \hat{\lambda} = 0
\end{aligned}
\end{equation}

and
\begin{equation}\label{6b}
\begin{aligned}
&- \nabla_{\lambda} L(\bar{\varepsilon}, \bar{\theta}, \bar{\lambda}) 
- \nabla_{\lambda\varepsilon} L(\bar{\varepsilon}, \bar{\theta}, \bar{\lambda}) \Delta \bar{\varepsilon} 
- \nabla_{\lambda\theta} L(\bar{\varepsilon}, \bar{\theta}, \bar{\lambda}) \Delta \hat{\theta} 
- \nabla_{\lambda}^2 L(\bar{\varepsilon}, \bar{\theta}, \bar{\lambda}) \Delta \hat{\lambda} 
+ N_{\mathbb{R}^m_+} (\bar{\lambda} + \Delta \hat{\lambda}) \ni 0 
\\[2ex]
&\Updownarrow \\[2ex]
&\bar{\lambda} + \Delta \hat{\lambda} \in \mathbb{R}^m_+,\quad  \nabla_{\lambda} L(\bar{\varepsilon}, \bar{\theta}, \bar{\lambda}) 
+\nabla_{\lambda\varepsilon} L(\bar{\varepsilon}, \bar{\theta}, \bar{\lambda}) \Delta \bar{\varepsilon} 
+\nabla_{\lambda\theta} L(\bar{\varepsilon}, \bar{\theta}, \bar{\lambda}) \Delta \hat{\theta} 
+ \nabla_{\lambda}^2 L(\bar{\varepsilon}, \bar{\theta}, \bar{\lambda}) \Delta \hat{\lambda} \leq 0
\end{aligned}
\end{equation}

It is straightforward to verify that $\nabla_{\lambda}^2 L(\bar{\varepsilon}, \bar{\theta}, \bar{\lambda}) = 0$.  By equation~(\ref{25}), $ \nabla_{\theta} L(\bar{\varepsilon}, \bar{\theta}, \bar{\lambda})=0$. Substituting  $\nabla_{\theta} L(\bar{\varepsilon}, \bar{\theta}, \bar{\lambda})=0$ to (\ref{6a}) and $\nabla_{\lambda}^2 L(\bar{\varepsilon}, \bar{\theta}, \bar{\lambda}) = 0$ to (\ref{6b}) yields
\begin{equation}\label{6aa}
\bar{\theta} + \Delta \hat{\theta} \in \mathbb{R}^d,\quad \nabla_{\theta\varepsilon} L(\bar{\varepsilon}, \bar{\theta}, \bar{\lambda}) \Delta \bar{\varepsilon} 
+ \nabla_{\theta}^2 L(\bar{\varepsilon}, \bar{\theta}, \bar{\lambda}) \Delta \hat{\theta} +  \nabla_{\theta\lambda} L(\bar{\varepsilon}, \bar{\theta}, \bar{\lambda}) \Delta \hat{\lambda}=0,
\end{equation}
and 
\begin{equation}\label{6bb}
\bar{\lambda} + \Delta \hat{\lambda} \in \mathbb{R}^m_+,\quad  \nabla_{\lambda} L(\bar{\varepsilon}, \bar{\theta}, \bar{\lambda}) 
+\nabla_{\lambda\varepsilon} L(\bar{\varepsilon}, \bar{\theta}, \bar{\lambda}) \Delta \bar{\varepsilon} 
+\nabla_{\lambda\theta} L(\bar{\varepsilon}, \bar{\theta}, \bar{\lambda}) \Delta \hat{\theta}  \leq 0
\end{equation}

Since $\mathbb{R}_{+}^m$ is a closed convex cone, by~(\ref{22}), we have
\begin{equation}\label{6c}
\begin{aligned}
\big(
\nabla_{\lambda} L(\bar{\varepsilon}, \bar{\theta}, \bar{\lambda})
+ \nabla_{\lambda\varepsilon} L(\bar{\varepsilon}, \bar{\theta}, \bar{\lambda}) \, \Delta \bar{\varepsilon}
+ \nabla_{\lambda\theta} L(\bar{\varepsilon}, \bar{\theta}, \bar{\lambda}) \, \Delta \hat{\theta}
\big)
\cdot (\bar{\lambda} + \Delta \hat{\lambda})=0
\end{aligned}
\end{equation}

In summary, VI~(\ref{vi-linearized}) is equivalent to the system consisting of (\ref{6aa}), (\ref{6bb}), and (\ref{6c}). Note that $\lambda=[\lambda_1,\dots,\lambda_j]$, we now further discuss the formulation (\ref{6bb}) and (\ref{6c}) for $j \in I_{\text{Inactive}}, j \in I_{\text{Binding}}$, and $j \in I_{\text{Non-binding}}$.
  
\textbf{Case 1.} 
If $j \in I_{\text{Inactive}}$, then 
$\nabla_{\lambda_j} L(\bar{\varepsilon}, \bar{\theta}, \bar{\lambda}) < 0$ 
and $\bar{\lambda}_j = 0$. 
Since $\Delta \bar{\varepsilon}$, $\Delta \hat{\theta}$, and $\Delta \hat{\lambda}$ 
are all sufficiently close to $0$ and 
$\nabla_{\lambda_j} L(\bar{\varepsilon}, \bar{\theta}, \bar{\lambda}) < 0$, 
condition~(\ref{6bb})
\begin{equation*}
\nabla_{\lambda_j} L(\bar{\varepsilon}, \bar{\theta}, \bar{\lambda})
+ \nabla_{\lambda_j\varepsilon} L(\bar{\varepsilon}, \bar{\theta}, \bar{\lambda}) \Delta \bar{\varepsilon}
+ \nabla_{\lambda_j\theta} L(\bar{\varepsilon}, \bar{\theta}, \bar{\lambda}) \Delta \hat{\theta}
\le 0
\end{equation*}
is automatically satisfied. 
Therefore, the term 
\begin{equation}\label{eq:free-term}
\nabla_{\lambda_j\varepsilon} L(\bar{\varepsilon}, \bar{\theta}, \bar{\lambda}) \, \Delta \bar{\varepsilon}
+ \nabla_{\lambda_j\theta} L(\bar{\varepsilon}, \bar{\theta}, \bar{\lambda}) \, \Delta \hat{\theta}
\;\; \text{is free.}
\end{equation}
Moreover, condition~(\ref{6c}) implies that 
\begin{equation}
    \Delta \hat{\lambda}_j = 0 \quad \text{for } j \in I_{\text{Inactive}}.
\end{equation}

\textbf{Case 2.} 
If $j \in I_{\text{Non-binding}}$, then 
$\nabla_{\lambda_j} L(\bar{\varepsilon}, \bar{\theta}, \bar{\lambda}) = 0$ 
and $\bar{\lambda}_j = 0$.
Substituting $\nabla_{\lambda_j} L(\bar{\varepsilon}, \bar{\theta}, \bar{\lambda}) = 0$ and $\bar{\lambda}_j = 0$ into condition~(\ref{6bb}) yields
\begin{equation}\label{eq:nonbinding}
\begin{aligned}
\Delta \hat{\lambda}_j &\ge 0 
\quad &&\text{for } j \in I_{\text{Non-binding}}, \\
\nabla_{\lambda_j \varepsilon} L(\bar{\varepsilon}, \bar{\theta}, \bar{\lambda}) \, \Delta \bar{\varepsilon}
+ \nabla_{\lambda_j \theta} L(\bar{\varepsilon}, \bar{\theta}, \bar{\lambda}) \, \Delta \hat{\theta} &\le 0
\quad &&\text{for } j \in I_{\text{Non-binding}}.
\end{aligned}
\end{equation}

\textbf{Case 3.} 
If $j \in I_{\text{Binding}}$, then 
$\nabla_{\lambda_j} L(\bar{\varepsilon}, \bar{\theta}, \bar{\lambda}) = 0$ 
and $\bar{\lambda}_j \ge 0$.
Since $\Delta \hat{\lambda}$ is close to $0$, 
$\bar{\lambda}+\Delta \hat{\lambda} \in \mathbb{R}_{+}^m$ 
is automatically satisfied. 
To satisfy conditions~(\ref{6bb}) and~(\ref{6c}), we must have 
\[
\nabla_{\lambda_j \varepsilon} L(\bar{\varepsilon}, \bar{\theta}, \bar{\lambda}) \, \Delta \bar{\varepsilon}
+ \nabla_{\lambda_j \theta} L(\bar{\varepsilon}, \bar{\theta}, \bar{\lambda}) \, \Delta \hat{\theta} = 0.
\]
In summary, we obtain
\begin{equation}\label{eq:binding}
\begin{aligned}
\Delta \hat{\lambda}_j \;\; \text{is free.}
\quad &&\text{for } j \in I_{\text{Binding}}, \\
\nabla_{\lambda_j \varepsilon} L(\bar{\varepsilon}, \bar{\theta}, \bar{\lambda}) \, \Delta \bar{\varepsilon}
+ \nabla_{\lambda_j \theta} L(\bar{\varepsilon}, \bar{\theta}, \bar{\lambda}) \, \Delta \hat{\theta} = 0
\quad &&\text{for } j \in I_{\text{Binding}}.
\end{aligned}
\end{equation}

Taken together, Cases~1--3 show that 
VI~(\ref{vi-linearized}) is equivalent to the system consisting of (\ref{6aa}) and (\ref{eq:free-term}--\ref{eq:binding}).On the other hand, the following argument shows that VI~(\ref{aux vi}) is also equivalent to this system.

The VI~(\ref{aux vi}) implies that: 

\begin{equation}\label{6d}
\begin{aligned}
&\nabla_{\theta\varepsilon} L(\bar{\varepsilon}, \bar{\theta}, \bar{\lambda}) \Delta \bar{\varepsilon} 
+ \nabla_{\theta}^2 L(\bar{\varepsilon}, \bar{\theta}, \bar{\lambda}) \Delta \hat{\theta} 
+ \nabla_{\theta\lambda} L(\bar{\varepsilon}, \bar{\theta}, \bar{\lambda}) \Delta \hat{\lambda} + N_{\mathbb{R}^d}(\Delta \hat{\theta}) \ni 0 \\[2ex]
&\Updownarrow \\[2ex]
&\Delta \hat{\theta} \in \mathbb{R}^d,\quad
\nabla_{\theta\varepsilon} L(\bar{\varepsilon}, \bar{\theta}, \bar{\lambda}) \Delta \bar{\varepsilon} + \nabla_{\theta}^2 L(\bar{\varepsilon}, \bar{\theta}, \bar{\lambda}) \Delta \hat{\theta}
+ \nabla_{\theta\lambda} L(\bar{\varepsilon}, \bar{\theta}, \bar{\lambda}) \Delta \hat{\lambda} = 0
\end{aligned}
\end{equation}

and
\begin{equation}\label{6e}
\begin{aligned}
&- \nabla_{\lambda\varepsilon} L(\bar{\varepsilon}, \bar{\theta}, \bar{\lambda}) \Delta \bar{\varepsilon} 
- \nabla_{\lambda\theta} L(\bar{\varepsilon}, \bar{\theta}, \bar{\lambda}) \Delta \hat{\theta} 
- \nabla_{\lambda}^2 L(\bar{\varepsilon}, \bar{\theta}, \bar{\lambda}) \Delta \hat{\lambda} 
+ N_{D} (\Delta \hat{\lambda}) \ni 0 
\\[2ex]
&\Updownarrow \\[2ex]
&\Delta \hat{\lambda} \in D,\quad  
(-\nabla_{\lambda\varepsilon} L(\bar{\varepsilon}, \bar{\theta}, \bar{\lambda}) \Delta \bar{\varepsilon} 
-\nabla_{\lambda\theta} L(\bar{\varepsilon}, \bar{\theta}, \bar{\lambda}) \Delta \hat{\theta} 
-\nabla_{\lambda}^2 L(\bar{\varepsilon}, \bar{\theta}, \bar{\lambda}) \Delta \hat{\lambda})(\Delta \hat{\lambda}^\prime- \Delta\hat{\lambda}) \leq 0, \forall \Delta \hat{\lambda}^\prime \in D
\end{aligned}
\end{equation}
Since $D$ is a closed convex cone, by~(\ref{22}), we have
\begin{equation}\label{6f}
\begin{aligned}
\big(
\nabla_{\lambda\varepsilon} L(\bar{\varepsilon}, \bar{\theta}, \bar{\lambda}) \, \Delta \bar{\varepsilon}
+ \nabla_{\lambda\theta} L(\bar{\varepsilon}, \bar{\theta}, \bar{\lambda}) \, \Delta \hat{\theta}
\big)
\cdot \Delta \hat{\lambda}=0
\end{aligned}
\end{equation}

\textbf{Case 1.} If $j \in I_{\text{Inactive}}$, then by the definition of the space $D$ (see Proposition~\ref{pro:aux vi}), 
we have 
\begin{equation}\label{6ee}
\Delta \hat{\lambda}_j = 0 \;\;, \forall\, \Delta \hat{\lambda} \in D
\end{equation}
Thus, condition~(\ref{6e}) is automatically satisfied. 
Therefore, the term 
\begin{equation}
\nabla_{\lambda_j\varepsilon} L(\bar{\varepsilon}, \bar{\theta}, \bar{\lambda}) \, \Delta \bar{\varepsilon}
+ \nabla_{\lambda_j\theta} L(\bar{\varepsilon}, \bar{\theta}, \bar{\lambda}) \, \Delta \hat{\theta}
\;\; \text{is free.}
\end{equation}

\textbf{Case 2.} If $j \in I_{\text{Non-binding}}$, then by the definition of the space $D$ (see Proposition~\ref{pro:aux vi}), 
we have \begin{equation}
\Delta \hat{\lambda}_j \geq 0 ,\;\; \forall\, \Delta \hat{\lambda} \in D    
\end{equation}
By the condition~(\ref{6f}),
\begin{equation}
\nabla_{\lambda_j \varepsilon} L(\bar{\varepsilon}, \bar{\theta}, \bar{\lambda}) \Delta \bar{\varepsilon}+\nabla_{\lambda_j \theta} L(\bar{\varepsilon}, \bar{\theta}, \bar{\lambda}) \Delta \hat{\theta} \leq 0 \quad \text { for } j \in I_{\text {Non-binding }} .
\end{equation}

\textbf{Case 3.} If $j \in I_{\text{binding}}$, then by the definition of the space $D$ (see Proposition~\ref{pro:aux vi}), 
we have \begin{equation}
\Delta \hat{\lambda}_j \quad \text{is free}  
\end{equation}
By the condition~(\ref{6f}),
\begin{equation}\label{6ff}
\nabla_{\lambda_j \varepsilon} L(\bar{\varepsilon}, \bar{\theta}, \bar{\lambda}) \Delta \bar{\varepsilon}+\nabla_{\lambda_j \theta} L(\bar{\varepsilon}, \bar{\theta}, \bar{\lambda}) \Delta \hat{\theta}=0 \quad \text { for } j \in I_{\text {Binding }}
\end{equation}

Since (\ref{6d},\ref{6ee}--\ref{6ff}) and (\ref{6aa},\ref{eq:free-term}--\ref{eq:binding}) 
are equivalent formulations, 
VI~(\ref{aux vi}) is equivalent to VI~(\ref{vi-linearized}). \qed
\end{proof}

% \subsection{Proof 2}

% \begin{lemma}[Reduction Lemma; see 2E.4 in~\cite{dontchev2009implicit}]
% Let $C\subset\R^n$ be a polyhedral convex set, $\bar x\in C$, $\bar v\in N_C(\bar x)$, and let
% $K := K_C(\bar x,\bar v)$ be the critical cone.
% Then there is a neighborhood $O$ of $(0,0)$ such that
% \[
% O\cap\big[\operatorname{gph}N_C-(\bar x,\bar v)\big]
% \;=\;
% O\cap \operatorname{gph}N_K.
% \]
% Equivalently, for $(w,u)$ sufficiently close to $(0,0)$,
% \[
% \bar v+u \in N_C(\bar x+w)\quad \Longleftrightarrow\quad u\in N_K(w).
% \]
% \end{lemma}

% \subsection{Proof of corollary~\ref{thm:DIF-error}}\label{proofofDIF-error}

% By the inequality~(\ref{DIFerror1}), we have 

% \begin{equation}
% (1-k \delta)\|\eta-\bar{\eta}\| \leq k\left(\left\|\nabla_{\varepsilon} f(\bar{\varepsilon}, \bar{\eta})\right\|+\delta\right)\|\Delta \varepsilon\|
% \end{equation}

% Since 

% $\theta$

\subsubsection{Proof of corollary~\ref{corollary2}}\label{app:positivehomo}
\begin{proof}
Recall the directional derivative

$$
D \theta(\bar{\varepsilon} ; v):=\lim _{t \downarrow 0, v^{\prime} \rightarrow v} \frac{\theta\left(\bar{\varepsilon}+t v^{\prime}\right)-\theta(\bar{\varepsilon})}{t},
$$

whenever the limit exists and is finite.

Let $\alpha \geq 0$.
- If $\alpha=0$ : by the definition with $v \equiv 0$,

$$
D \theta(\bar{\varepsilon} ; 0)=\lim _{t \downarrow 0, v^{\prime} \rightarrow 0} \frac{\theta\left(\bar{\varepsilon}+t v^{\prime}\right)-\theta(\bar{\varepsilon})}{t}=0,
$$

hence $0 \cdot D \theta(\bar{\varepsilon} ; v)=D \theta(\bar{\varepsilon} ; 0)$.
- If $\alpha>0$ : using the definition with the direction $\alpha v$,

$$
D \theta(\bar{\varepsilon} ; \alpha v)=\lim _{t \downarrow 0, u^{\prime} \rightarrow \alpha v} \frac{\theta\left(\bar{\varepsilon}+t u^{\prime}\right)-\theta(\bar{\varepsilon})}{t} .
$$

Choose $u^{\prime}=\alpha v^{\prime}$ with $v^{\prime} \rightarrow v$. Then

$$
\begin{aligned}
D \theta(\bar{\varepsilon} ; \alpha v) & =\lim _{t \downarrow 0, v^{\prime} \rightarrow v} \frac{\theta\left(\bar{\varepsilon}+t\left(\alpha v^{\prime}\right)\right)-\theta(\bar{\varepsilon})}{t}=\lim _{t \downarrow 0, v^{\prime} \rightarrow v} \frac{\theta\left(\bar{\varepsilon}+(\alpha t) v^{\prime}\right)-\theta(\bar{\varepsilon})}{t} \\
& =\lim _{s \downarrow 0, v^{\prime} \rightarrow v} \frac{\theta\left(\bar{\varepsilon}+s v^{\prime}\right)-\theta(\bar{\varepsilon})}{s / \alpha} \quad(\text { set } s=\alpha t) \\
& =\alpha \lim _{s \downarrow 0, v^{\prime} \rightarrow v} \frac{\theta\left(\bar{\varepsilon}+s v^{\prime}\right)-\theta(\bar{\varepsilon})}{s}=\alpha D \theta(\bar{\varepsilon} ; v) .
\end{aligned}
$$

Combining the two cases yields

$$
\alpha D \theta(\bar{\varepsilon} ; \Delta \bar{\varepsilon})=D \theta(\bar{\varepsilon} ; \alpha \Delta \bar{\varepsilon}), \quad \forall \alpha \in \mathbb{R}^{+}
$$

\end{proof}
\subsubsection{Proof of proposition~\ref{proposition7}}\label{proofof7}

Section~\ref{proof of thm5} has proved that the directional derivative of the
solution mapping $S$ is characterized by the auxiliary VI~(\ref{aux vi}). Specifically,
\begin{equation}
D S(\bar\varepsilon)(\Delta\bar\varepsilon)=\Delta\hat\eta
\quad\text{where}\quad
\mu\,\Delta\bar\varepsilon + A\,\Delta\hat\eta + N_K(\Delta\hat\eta)\ni 0 .
\end{equation}

Since the DIF $D\theta(\bar\varepsilon;\Delta\bar\varepsilon)$ defined in
Definition~\ref{def:DIF} is the $\theta$-component of
$D S(\bar\varepsilon)(\Delta\bar\varepsilon)$, we have
\begin{equation}
\Delta\hat\theta \;=\; D\theta(\bar\varepsilon;\Delta\bar\varepsilon).
\end{equation}

\subsubsection{Proof of Proposition~\ref{thm:DIF-error}} \label{proof of DIF error}

\begin{proof}
Let $\Delta\hat\eta=[\Delta \hat\theta,\Delta\hat\lambda]$ denote the solution of Auxiliary VI~(\ref{aux vi}). By~(\ref{barssolution}),
\[
\Delta\hat\eta \;=\; \bar{s}\!\left(-u\,\Delta\bar\varepsilon\right).
\]
By definition,
\[
\Delta\eta \;=\; \eta-\bar\eta \;=\; S(\bar\varepsilon+\Delta\bar\varepsilon)-\bar\eta .
\]
Therefore,
\begin{equation}\label{eq:theta_comp_bound}
\|\Delta\theta-\Delta\hat\theta\|
\;\le\;
\|\Delta\eta-\Delta\hat\eta\|
\;=\;
\big\|S(\bar\varepsilon+\Delta\bar\varepsilon)-\bar\eta-\bar{s}(-u\,\Delta\bar\varepsilon)\big\|.
\end{equation}

By Lemma~\ref{G is first order approximation},
\[
\sigma(\varepsilon)
:=G^{-1}\!\Big(-\nabla_{\varepsilon} f(\bar{\varepsilon},\bar{\eta})\,(\varepsilon-\bar{\varepsilon})\Big)
\]
is a first-order approximation of $S(\varepsilon)$ and $\sigma(\bar\varepsilon)=S(\bar\varepsilon)=\bar\eta$.
Thus,
\begin{equation}\label{eq:first_order}
S(\bar\varepsilon+\Delta\bar\varepsilon)
=
\sigma(\bar\varepsilon+\Delta\bar\varepsilon)+o(\|\Delta\bar\varepsilon\|).
\end{equation}
By Lemma~\ref{lem:G-vs-G0},
\begin{equation}\label{9-96}
\sigma(\bar\varepsilon+\Delta\bar\varepsilon)
=G^{-1}(-u\,\Delta\bar\varepsilon)
=G_0^{-1}(-u\,\Delta\bar\varepsilon)+\bar{\eta}
=\bar{s}(-u\,\Delta\bar\varepsilon)+\bar{\eta}.
\end{equation}
Substituting (\ref{9-96}) and (\ref{eq:first_order}) into right side of (\ref{eq:theta_comp_bound}) yields
\[
\big\|S(\bar\varepsilon+\Delta\bar\varepsilon)-\bar\eta-\bar{s}(-u\,\Delta\bar\varepsilon)\big\|
= o(\|\Delta\bar\varepsilon\|),
\]
hence
\[
\|\Delta\theta-\Delta\hat\theta\| \leq o(\|\Delta\bar\varepsilon\|).
\]

By the little-$o$ definition, for every $\epsilon>0$ there exists $\delta(\epsilon)>0$
such that, whenever $0<\|\Delta\bar\varepsilon\|<\delta(\epsilon)$,
\begin{equation}\label{eq:ratio}
\frac{\|\Delta\theta-\Delta\hat\theta\|}{\|\Delta\bar\varepsilon\|} < \epsilon.
\end{equation}
Fix any $M>0$ and set $\epsilon:=M$. Then there exists $\delta:=\delta(M)>0$ such that,
for all $0<\|\Delta\bar\varepsilon\|<\delta$,
\begin{equation}\label{eq:local-Lip}
\|\Delta\theta-\Delta\hat\theta\| \le M\,\|\Delta\bar\varepsilon\|.
\end{equation}
\end{proof}

\subsubsection{Proof of Theorem~\ref{thm:Auxiliary problem}}\label{proofof8}
Define the Lagrangian of QP~(\ref{QP}) as
\begin{equation}
\small
\begin{aligned}
\mathcal{L}_{QP}(\omega,\zeta)
=&\; L(\bar{\varepsilon}, \bar{\theta}, \bar{\lambda}) 
+ \big\langle \nabla_{\theta \varepsilon} L(\bar{\varepsilon}, \bar{\theta}, \bar{\lambda})\!\cdot\!\Delta \bar{\varepsilon},\; \omega \big\rangle 
+ \frac{1}{2}\big\langle \omega,\; \nabla_{\theta\theta}^2 L(\bar{\varepsilon}, \bar{\theta}, \bar{\lambda})\,\omega \big\rangle\\
&\;+\sum_{j\in I_{\text{binding}}}\zeta_j\!\left[
\frac{1}{N_j}\sum_{i=1}^{N_j}\ell_j(z_i^{(j)},\bar\theta)
+\frac{1}{N_j}\sum_{i=1}^{N_j}\nabla_\theta \ell_j(z_i^{(j)},\bar\theta)\!\cdot\!\omega
+\sum_{z_i^{(j)}\in Z^r}\Delta \bar{\varepsilon}_j\,\ell_j(z_i^{(j)},\bar\theta)
- \tau_j\right]\\
&\;+\sum_{j\in I_{\text{non-binding}}}\zeta_j\!\left[
\frac{1}{N_j}\sum_{i=1}^{N_j}\ell_j(z_i^{(j)},\bar\theta)
+\frac{1}{N_j}\sum_{i=1}^{N_j}\nabla_\theta \ell_j(z_i^{(j)},\bar\theta)\!\cdot\!\omega
+\sum_{z_i^{(j)}\in Z^r}\Delta \bar{\varepsilon}_j\,\ell_j(z_i^{(j)},\bar\theta)
- \tau_j\right]\\
&\;+\sum_{j\in I_{\text{Inactive}}}\zeta_j\!\left[
\frac{1}{N_j}\sum_{i=1}^{N_j}\ell_j(z_i^{(j)},\bar\theta)
+\frac{1}{N_j}\sum_{i=1}^{N_j}\nabla_\theta \ell_j(z_i^{(j)},\bar\theta)\!\cdot\!\omega
+\sum_{z_i^{(j)}\in Z^r}\Delta \bar{\varepsilon}_j\,\ell_j(z_i^{(j)},\bar\theta)
- \tau_j\right].
\end{aligned}
\end{equation}

A pair \((\omega^\star,\zeta^\star)\) satisfies the Karush–Kuhn–Tucker (KKT) conditions if:

\text{(i) Primal feasibility.}
\begin{align}
&\frac{1}{N_j}\sum_{i=1}^{N_j}\ell_j(z_i^{(j)},\bar\theta)
+\frac{1}{N_j}\sum_{i=1}^{N_j}\nabla_\theta \ell_j(z_i^{(j)},\bar\theta)\!\cdot\!\omega^\star
+\sum_{z_i^{(j)}\in Z^r}\Delta \bar{\varepsilon}_j\,\ell_j(z_i^{(j)},\bar\theta)
- \tau_j = 0,\qquad &&\forall j\in I_{\text{binding}},\\[2pt]
&\frac{1}{N_j}\sum_{i=1}^{N_j}\ell_j(z_i^{(j)},\bar\theta)
+\frac{1}{N_j}\sum_{i=1}^{N_j}\nabla_\theta \ell_j(z_i^{(j)},\bar\theta)\!\cdot\!\omega^\star
+\sum_{z_i^{(j)}\in Z^r}\Delta \bar{\varepsilon}_j\,\ell_j(z_i^{(j)},\bar\theta)
- \tau_j \le 0,\qquad &&\forall j\in I_{\text{non-binding}}.
\end{align}
For \(j\in I_{\text{Inactive}}\), there is no constraint.

\text{(ii) Stationarity.}\label{Stationarity}
\begin{equation}
\small
\begin{aligned}
\nabla_{\theta \varepsilon} L(\bar{\varepsilon}, \bar{\theta}, \bar{\lambda})\!\cdot\! \Delta \bar{\varepsilon}
+\nabla_{\theta\theta}^2 L(\bar{\varepsilon}, \bar{\theta}, \bar{\lambda})\,\omega^\star
+\sum_{j\in I_{\text{binding}}\cup I_{\text{non-binding}}\cup I_{\text{Inactive}}}
\zeta_j^\star\!\left(\frac{1}{N_j}\sum_{i=1}^{N_j}\nabla_\theta \ell_j(z_i^{(j)},\bar\theta)\right)
=0.
\end{aligned}
\end{equation}

\text{(iii) Dual feasibility.}
\begin{equation}
\zeta_j^\star \ge 0,\qquad \forall j\in I_{\text{non-binding}}.
\end{equation}

\text{(iv) Complementary slackness.}
\begin{equation}
\zeta_j^\star\!\left[
\frac{1}{N_j}\sum_{i=1}^{N_j}\ell_j(z_i^{(j)},\bar\theta)
+\frac{1}{N_j}\sum_{i=1}^{N_j}\nabla_\theta \ell_j(z_i^{(j)},\bar\theta)\!\cdot\!\omega^\star
+\sum_{z_i^{(j)}\in Z^r}\Delta \bar{\varepsilon}_j\,\ell_j(z_i^{(j)},\bar\theta)
- \tau_j\right]=0,\qquad \forall j\in I_{\text{non-binding}}.
\end{equation}

We now show that any \((\omega^\star,\zeta^\star)\) satisfying the above KKT conditions also satisfies the auxiliary VI~(\ref{aux vi}). 
Recall that in Section~\ref{proof of propostion6}, we have proved that the auxiliary VI~(\ref{aux vi}) is equivalent to the system \((\ref{6d},\ref{6ee}--\ref{6ff})\).

\medskip\noindent
Note that
\[
\begin{aligned}
L(\varepsilon,\theta,\lambda)
&=\frac{1}{N_0}\sum_{i=1}^{N_0}\ell_0(z_i^{(0)},\theta)
+\varepsilon_0\!\!\sum_{z_i^{(0)}\in Z^r}\!\ell_0(z_i^{(0)},\theta)\\
&\quad+\sum_{j=1}^m\lambda_j\!\left[
\frac{1}{N_j}\sum_{i=1}^{N_j}\ell_j(z_i^{(j)},\theta)
+\varepsilon_j\!\!\sum_{z_i^{(j)}\in Z^r}\!\ell_j(z_i^{(j)},\theta)-\tau_j
\right],
\end{aligned}
\]
we have
\begin{equation}
\frac{1}{N_j}\sum_{i=1}^{N_j}\nabla_\theta\ell_j(z_i^{(j)},\bar\theta)
=\nabla_{\theta\lambda_j}L(\bar\varepsilon,\bar\theta,\bar\lambda).
\end{equation}

\medskip\noindent

The stationarity condition can be rewritten as
\begin{equation}\label{kkt1}
\nabla_{\theta\varepsilon}L(\bar\varepsilon,\bar\theta,\bar\lambda)\,\Delta\bar\varepsilon
+\nabla_{\theta}^2L(\bar\varepsilon,\bar\theta,\bar\lambda)\,\omega^\star
+\nabla_{\theta\lambda}L(\bar\varepsilon,\bar\theta,\bar\lambda)\,\zeta^\star=0.
\end{equation}
By the definition of the normal cone (Definition~\ref{def:normal cone}), this is equivalent to
\begin{equation}\label{kktvi1}
\nabla_{\theta\varepsilon}L(\bar\varepsilon,\bar\theta,\bar\lambda)\,\Delta\bar\varepsilon
+\nabla_{\theta}^2L(\bar\varepsilon,\bar\theta,\bar\lambda)\,\omega^\star
+\nabla_{\theta\lambda}L(\bar\varepsilon,\bar\theta,\bar\lambda)\,\zeta^\star
\in N_{\mathbb{R}^d}(\Delta\hat\theta).
\end{equation}

\medskip\noindent
We distinguish the three index sets.

\textbf{Case 1 (\(j\in I_{\text{Inactive}}\)).} 
By complementary slackness, \(\zeta_j^\star=0\).

\medskip
\textbf{Case 2 (\(j\in I_{\text{non-binding}}\)).} 
By complementary slackness, \(\zeta_j^\star\ge 0\).
Moreover, since 
\(\frac{1}{N_j}\sum_{i=1}^{N_j}\ell_j(z_i^{(j)},\bar\theta)=\tau_j\),
substituting 
\(\nabla_{\lambda_j\varepsilon}L(\bar\varepsilon,\bar\theta,\bar\lambda)\,\Delta\bar\varepsilon
=\sum_{z_i^{(j)}\in Z^r}\Delta\bar\varepsilon_j\ell_j(z_i^{(j)},\bar\theta)\)
and 
\(\nabla_{\lambda_j\theta}L(\bar\varepsilon,\bar\theta,\bar\lambda)\,\omega^\star
=\frac{1}{N_j}\sum_{i=1}^{N_j}\nabla_\theta\ell_j(z_i^{(j)},\bar\theta)\omega^\star\)
into the primal feasibility condition gives
\begin{equation}
\nabla_{\lambda_j\varepsilon}L(\bar\varepsilon,\bar\theta,\bar\lambda)\,\Delta\bar\varepsilon
+\nabla_{\lambda_j\theta}L(\bar\varepsilon,\bar\theta,\bar\lambda)\,\omega^*
\le 0.
\end{equation}

\medskip
\textbf{Case 3 (\(j\in I_{\text{binding}}\)).} 
Here \(\zeta_j^\star\) is free.
Since 
\(\frac{1}{N_j}\sum_{i=1}^{N_j}\ell_j(z_i^{(j)},\bar\theta)=\tau_j\),
using the same substitutions into  of case 3 into primal feasibility as above yields
\begin{equation}
\nabla_{\lambda_j\varepsilon}L(\bar\varepsilon,\bar\theta,\bar\lambda)\,\Delta\bar\varepsilon
+\nabla_{\lambda_j\theta}L(\bar\varepsilon,\bar\theta,\bar\lambda)\,\omega^*
=0.
\end{equation}

Combining all cases, \(\zeta^\star\) satisfies
\begin{equation}\label{106}
\nabla_{\lambda\varepsilon}L(\bar\varepsilon,\bar\theta,\bar\lambda)\,\Delta\bar\varepsilon
-\nabla_{\lambda\theta}L(\bar\varepsilon,\bar\theta,\bar\lambda)\,\omega^\star
-\nabla_{\lambda}^2L(\bar\varepsilon,\bar\theta,\bar\lambda)\,\zeta^\star
+N_{D^\prime}(\zeta^\star)\ni 0,   
\end{equation}

where
\[
D^\prime:=\Bigl\{\zeta^*\in\mathbb{R}^m\;\big|\;
\zeta^*_j\ge 0 \ (j\in I_{\text{non-binding}}),\ 
\zeta^*_j=0 \ (j\in I_{\text{Inactive}})
\Bigr\}.
\]

\medskip\noindent

(\ref{106}) together with~(\ref{kkt1}) shows that \((\omega^\star,\zeta^\star)\) satisfies
\[
0\in A(\omega^\star,\zeta^\star)+\mu\Delta\bar\varepsilon+N_K(\omega^\star,\zeta^\star),
\]
which is exactly the auxiliary VI~(\ref{aux vi}). \(\hfill\qed\)

\subsubsection{Proof of Proposition~\ref{prop:dif-if}}\label{proofofdif-if}

\begin{proof}
By Theorem~\ref{thm:Auxiliary problem} we have
$\omega^\star=\Delta\hat\theta$, where $\Delta\hat\theta$ is the
$\theta$-component of the solution to the auxiliary VI~(\ref{aux vi}).
By Proposition~\ref{proposition7},
$\Delta\hat\theta = D\theta(\bar\varepsilon;\Delta\bar\varepsilon)$. Therefore, $\omega^*$ is the DIF.

Assume that no constraints are active at $\bar\theta$. Then the KKT
conditions reduce to the unconstrained optimality system, and the QP~(\ref{QP})
has the unique solution $\omega^\star$ given by the Newton/IF direction:
\begin{equation}\label{eq:IF-newton}
  \omega^\star
  \;=\;
  -\big[\nabla_{\theta\theta}^2 L(\bar\varepsilon,\bar\theta,\bar\lambda)\big]^{-1}
   \nabla_{\theta\varepsilon} L(\bar\varepsilon,\bar\theta,\bar\lambda).
\end{equation}
Since “no active constraints’’ implies $\bar\lambda=0$, substituting
$\bar\lambda=0$ into (\ref{eq:IF-newton}) yields
\[
  \omega^\star
  =
  -\big[\nabla_{\theta\theta}^2 L(\bar\varepsilon,\bar\theta,0)\big]^{-1}
   \nabla_{\theta\varepsilon} L(\bar\varepsilon,\bar\theta,0)
  \;=\;
  \frac{d\,\theta(\varepsilon)}{d\varepsilon_0}\Big|_{\varepsilon_0=0},
\]
which is exactly the \emph{influence function} (IF) defined in
\citet{koh2017understanding}.

Therefore, in the inactive-constraint case, the QP solution $\omega^\star$ coincides with the IF, as claimed.
\end{proof}

\subsection{Details in Toy example}\label{appendixoftoy}

We consider the toy example~(\ref{eq:toy})
\begin{equation*}\label{eq:app-prob}
\min_{\theta\in\mathbb{R}^2}\;
\frac13(\theta^\top x_1-y_1)^2+\frac13(\theta^\top x_2-y_2)^2+\frac13(\theta^\top x_3-y_3)^2
+\varepsilon\,(\theta^\top x_2-y_2)^2
\quad\text{s.t.}\quad \|\theta\|_2^2\le 1,
\end{equation*}
with data \(x_1=(1,0),\,x_2=(1,0),\,x_3=(0,1)\) and
\(y_1=1,\,y_2=0,\,y_3=\tfrac12\).
At \(\varepsilon=0\), the constrained solution is 
\(\bar\theta=(0.5,\,0.5)\). The corresponding dual variable is $\bar{\lambda}=0$.

\paragraph{Gradients and Hessian.}
For a single squared loss \(\ell_i(\theta)=(\theta^\top x_i-y_i)^2\),
\[
\nabla_\theta \ell_i(\theta)=2(\theta^\top x_i-y_i)x_i,\qquad
\nabla^2_{\theta\theta}\ell_i(\theta)=2x_ix_i^\top.
\]
Hence, at \(\varepsilon=0\) the (unconstrained) Hessian of
\(\tfrac13\sum_{i=1}^3\ell_i\) is
\begin{equation}\label{eq:app-H}
H=\frac13\!\sum_{i=1}^3 2x_ix_i^\top
=\frac{2}{3}\big(x_1x_1^\top+x_2x_2^\top+x_3x_3^\top\big)
=\begin{pmatrix}\frac{4}{3}&0\\[2pt]0&\frac{2}{3}\end{pmatrix}.
\end{equation}
The mixed derivative of the perturbed part is
\begin{equation}\label{eq:app-dLdte}
\nabla_{\theta\varepsilon} L(\bar\varepsilon,\bar\theta)
=\nabla_\theta \ell_2(\bar\theta)
=2(\bar\theta^\top x_2-y_2)x_2
=2\cdot 0.5\cdot (1,0)
=(1,0).
\end{equation}

\subsubsection{Classical IF (ignoring constraints)}
Differentiating the stationarity condition
\(\nabla_\theta L(\varepsilon,\theta)=0\) w.r.t. \(\varepsilon\) at
\((\bar\varepsilon,\bar\theta)\) gives
\[
H\,\frac{d\theta}{d\varepsilon}
+\nabla_{\theta\varepsilon}L(\bar\varepsilon,\bar\theta)=0
\quad\Rightarrow\quad
\frac{d\theta}{d\varepsilon}=-H^{-1}\nabla_{\theta\varepsilon}L(\bar\varepsilon,\bar\theta).
\]
Using \eqref{eq:app-H}–\eqref{eq:app-dLdte},
\[
\frac{d\theta}{d\varepsilon}
=-\begin{pmatrix}\frac{3}{4}&0\\[2pt]0&\frac{3}{2}\end{pmatrix}(1,0)
=\Big(-\tfrac34,\,0\Big).
\]
Removing sample \(z_2\) corresponds to \(\Delta\varepsilon=-\tfrac13\), so the IF
estimate is
\begin{equation}\label{eq:app-IF-step}
\Delta\theta_{\mathrm{IF}}
\approx \frac{d\theta}{d\varepsilon}\,\Delta\varepsilon
=\Big(-\tfrac34,\,0\Big)\!\cdot\!\Big(-\tfrac13\Big)
=\Big(\tfrac14,\,0\Big).
\end{equation}
Then \(\bar\theta+\Delta\theta_{\mathrm{IF}}=(0.75,\,0.5)\) whose
\(\ell_1\)-norm equals \(1.25\), i.e., the IF step is infeasible.

\subsubsection{DIF (feasible, sensitivity on the linearized VI)}
DIF linearizes the KKT/VI system at \((\bar\varepsilon,\bar\theta)\)
and searches \(\Delta\theta\) in the tangent subspace.
The QP~(\ref{QP}) corresponding to the toy problem~(\ref{eq:toy}) is
\begin{equation}\label{eq:app-DIF-QP}
\min_{\Delta\theta\in\mathbb{R}^2}\;
\big\langle \nabla_{\theta\varepsilon}L(\bar\varepsilon,\bar\theta)\,\Delta\varepsilon,
\;\Delta\theta\big\rangle
+\frac12\,\Delta\theta^\top H\,\Delta\theta
\quad\text{s.t.}\quad
\Delta\theta_1+\Delta\theta_2=0 .
\end{equation}
With \(\Delta\varepsilon=-\tfrac13\) and
\(b:=\nabla_{\theta\varepsilon}L\,\Delta\varepsilon=(-\tfrac13,0)\),
(\ref{eq:app-DIF-QP}) is
\[
\min_{\Delta\theta}\; b^\top\Delta\theta+\tfrac12\,\Delta\theta^\top H\,\Delta\theta
\quad\text{s.t.}\quad \Delta\theta_1+\Delta\theta_2=0 .
\]
The KKT conditions (with multiplier \(\mu\)) are
\[
H\Delta\theta+b+\mu(1,1)^\top=0,\qquad
\Delta\theta_1+\Delta\theta_2=0 .
\]
Expanding coordinates yields
\[
\frac{4}{3}\Delta\theta_1-\frac13+\mu=0,\qquad
\frac{2}{3}\Delta\theta_2+\mu=0,\qquad
\Delta\theta_1+\Delta\theta_2=0 .
\]
From the second and third equations \(\mu=-\frac{2}{3}\Delta\theta_2\) and
\(\Delta\theta_2=-\Delta\theta_1\).
Substituting into the first gives
\( \frac{4}{3}\Delta\theta_1-\frac13-\frac{2}{3}\Delta\theta_2=0
\Rightarrow 6\Delta\theta_1=1\),
hence
\begin{equation}\label{eq:app-DIF-sol}
{\;\Delta\theta_{\mathrm{DIF}}=(\tfrac16,\,-\tfrac16)\;}
\end{equation}
which lies on the tangent subspace and thus closely preserves feasibility
to first order.

\subsection{Details in Constrained Linear Regression}\label{appendix:constrained linear regression}

% ==== Setup ====
We consider per-sample
$$
\ell_i(\theta)=\tfrac12\,(x_i^\top\theta-y_i)^2,
$$
The constrained optimum $\bar{\theta}$ solves
\begin{equation}\label{eq:constraintlRE}
\begin{aligned}
\min_{\theta}\quad & \tfrac{1}{n}\sum_{i=1}^n \ell_i(\theta)\\
\text{s.t.}\quad & A_{\mathrm{eq}}\theta=b_{\mathrm{eq}},\quad
A_{\mathrm{ineq}}\theta\le b_{\mathrm{ineq}} .
\end{aligned}
\end{equation}
Assume we remove the data sample $(x_i,y_i)$. We use
$$
H=\nabla^2 L(\bar{\theta})=\tfrac{1}{n}X^\top X,
\qquad
g_i=\nabla_\theta \ell_i(\bar{\theta})=(x_i^\top\bar{\theta}-y_i)\,x_i .
$$

% ==== IF ====
\noindent\textbf{Classical Influence Function (IF).}
Ignoring feasibility constraints, the first-order effect of removing sample $(x_i,y_i)$ is
\begin{equation}
\Delta\theta_{\mathrm{IF}} \;=\; -\frac{1}{n}\,H^{-1}\,\,g_i .
\end{equation}

% ==== Penalty IF ====
\noindent\textbf{Penalty-based IF.}
We approximate feasibility via a penalized surrogate
\begin{equation}
\tilde{L}(\theta)
=\tfrac{1}{2n}\|X\theta-y\|_2^2
+\tfrac{\rho}{2}\|A_{\mathrm{eq}}\theta-b_{\mathrm{eq}}\|_2^2
+\sum_{j=1}^{m} k\,\big((a_j^\top\theta-b_j)_+\big)^3 ,
\end{equation}
where $(t)_+=\max\{t,0\}$, $\rho,k>0$, and $a_j^\top$ is the $j$-th row of $A_{\mathrm{ineq}}$.
The Hessian at $\bar{\theta}$ is
\begin{equation}
H_{\mathrm{pen}}
= \tfrac{1}{n}X^\top X
+ \rho\,A_{\mathrm{eq}}^\top A_{\mathrm{eq}}
+ \sum_{j\in\mathcal{A}} 6k\,t_j\, a_j a_j^\top,
\qquad
t_j=a_j^\top\bar{\theta}-b_j,\quad
\mathcal{A}=\{j:\ t_j>0\}.
\end{equation}
The corresponding update is
\begin{equation}
\Delta\theta_{\mathrm{penIF}} \;=\; -\frac{1}{n}\,H_{\mathrm{pen}}^{-1}\,\,g_i .
\end{equation}

{
In the experiments shown in Figure~\ref{fig:a}, we set the penalty parameters to 
$\rho = 200$ and $k = 1000$. To justify this choice, we evaluated the discrepancy between Penalty-based IF estimation $\Delta\theta_{\mathrm{penIF}}$ and leave-one-out retraining $ \Delta\theta_{\mathrm{LOO}}$:
\begin{equation*}
\|\Delta\theta_{\mathrm{penIF}} - \Delta\theta_{\mathrm{LOO}}\|.
\end{equation*}
by varying both $k$ and $\rho$.} 

{
Recall that $\bar{\theta}$ denotes the optimum of the constrained linear regression 
problem~(\ref{eq:constraintlRE}). At $\bar{\theta}$, all inequality constraints are 
strictly satisfied (i.e., $a_j^\top \bar{\theta}-b_j \le 0$ for every $j$), which implies 
that the active set is empty: $\mathcal{A}=\varnothing$. Consequently, the cubic penalty 
term contributes no curvature to the local Hessian $H_{\mathrm{pen}}$, and therefore 
the value of $k$ does not influence the penalty-based IF update $\Delta\theta_{\mathrm{penIF}}$. 
This behavior is precisely reflected in Figure~\ref{fig:k_tuning}(b), where varying $k$ 
produces essentially identical estimates.}

{In contrast, Figure~\ref{fig:k_tuning}(a) shows that increasing the equality--constraint
penalty $\rho$ steadily improves the approximation accuracy. The improvement becomes 
marginal once $\rho \geq 10$, after which the curve nearly plateaus. Based on this 
behavior, we adopt $\rho = 200$ in our experiments to ensure both stability and 
sufficient accuracy.}

\begin{figure}[h]
    \centering
    \begin{subfigure}{0.45\linewidth}
        \centering
        \includegraphics[width=\linewidth]{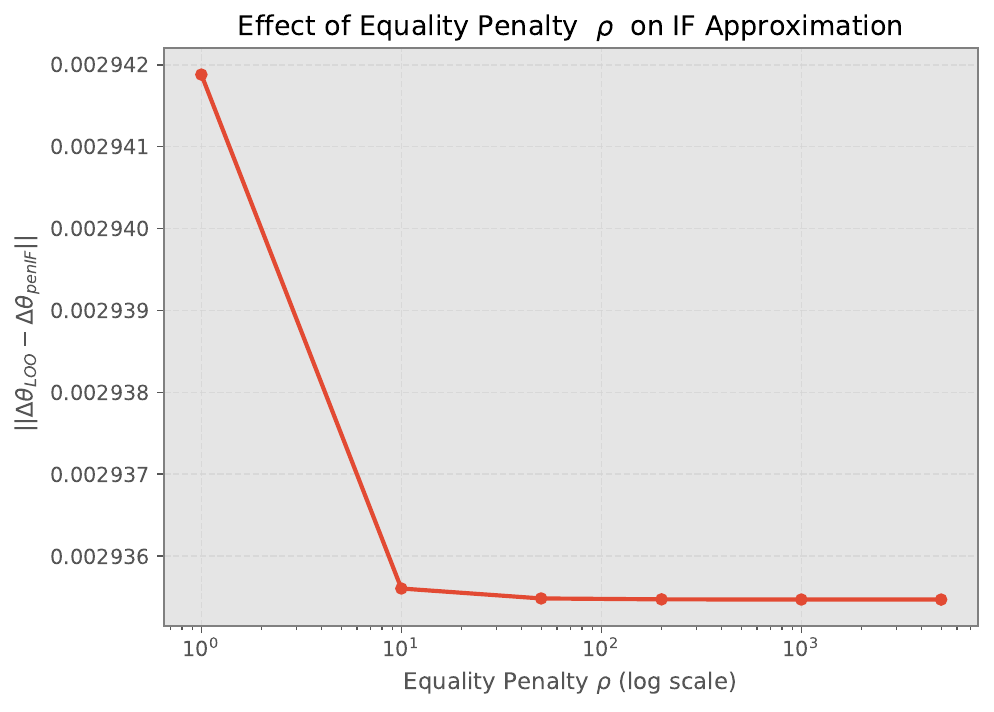}
        \caption{}
    \end{subfigure}\hfill
    \begin{subfigure}{0.45\linewidth}
        \centering
        \includegraphics[width=\linewidth]{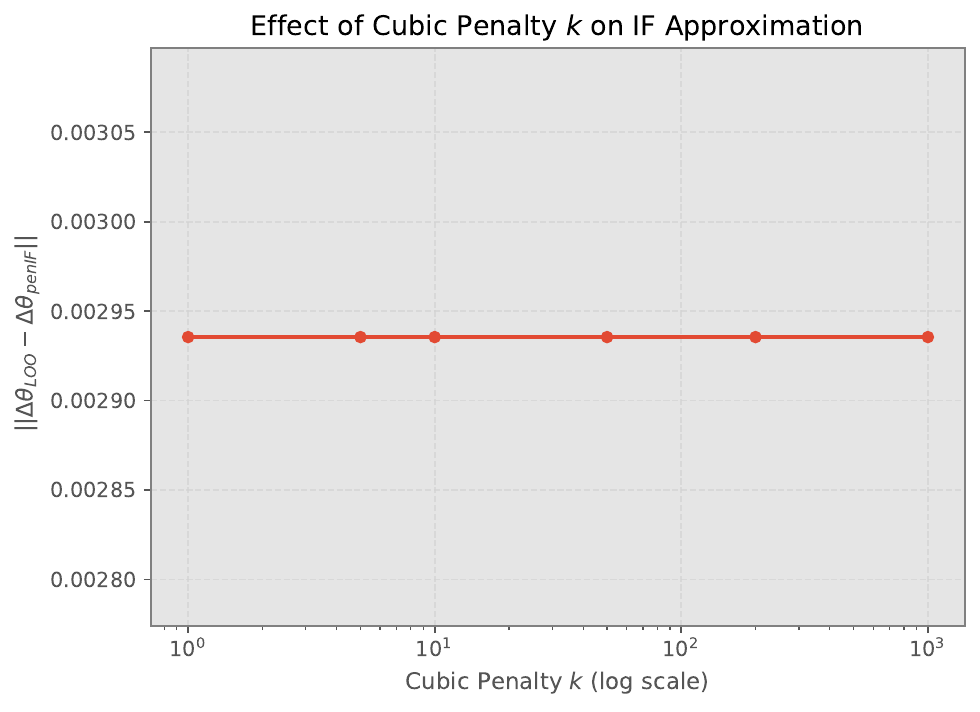}
        \caption{}
    \end{subfigure}
    \caption{
Hyper-parameter sensitivity of the penalty-based influence function.
(a) Effect of the equality penalty parameter $\rho$ on the approximation error 
$\|\Delta\theta_{\mathrm{penIF}} - \Delta\theta_{\mathrm{LOO}}\|$. Increasing $\rho$ improves the accuracy, 
but the gain becomes negligible once $\rho \ge 10$.
(b) Effect of the cubic penalty parameter $k$ on the cosine similarity between 
$\Delta\theta_{\mathrm{penIF}}$ and $\Delta\theta_{\mathrm{LOO}}$. Since the inequality constraint 
is non-active at the optimum $\bar{\theta}$, varying $k$ results in almost no change.
}
    \label{fig:k_tuning}
\end{figure}

% ==== DIF ====
\noindent\textbf{Directional Influence Function (DIF).}
We enforce feasibility with a one-step constrained QP (linearized KKT):
\begin{equation}
\begin{aligned}
\Delta\theta_{\mathrm{DIF}}
=\arg\min_{\Delta\theta\in\mathbb{R}^d}\quad
& \tfrac{1}{2}\,\Delta\theta^\top H\,\Delta\theta\;-\;\tfrac{1}{n}\,g_i^\top \Delta\theta\\
\text{s.t.}\quad
& A_{\mathrm{eq}}\Delta\theta=0,\\
& A_{\mathrm{ineq}}\Delta\theta= 0
\;\;(\text{active constraints only}).
\end{aligned}
\end{equation}

\subsection{Computational Complexity and Scalable Implementation of the QP Solver}\label{app:qp-complexity}

\subsubsection{Computational Complexity of the DIF QP}\label{app:ipm}
\label{app:qp-complexity}

We recall that the quadratic program used to compute DIF takes the form
\begin{align}
\min_{\omega}\quad &
b^\top \omega + \frac12\, \omega^\top H_{\theta\theta}\omega
\label{eq:app_difqp_obj}\\
\text{s.t.}\quad &
g_j(\omega):=a_j^\top \omega+\eta_j
\begin{cases}
=0, & j\in I_{\mathrm{binding}},\\[2pt]
\le 0, & j\in I_{\mathrm{non\mbox{-}binding}},
\end{cases}
\label{eq:app_difqp_cons}
\end{align}
where \(H_{\theta\theta}\) is the Hessian of the Lagrangian with respect to
\(\theta\) at \((\bar{\varepsilon},\bar{\theta},\bar{\lambda})\),
\[
b=\nabla_{\theta\varepsilon}L(\bar{\varepsilon},\bar{\theta},\bar{\lambda})
\,\Delta\bar{\varepsilon},
\]
and
\[
a_j
:=
\frac{1}{N_j}\sum_{i=1}^{N_j}
\nabla_\theta \ell_j(z_i^{(j)},\bar{\theta})
\]
is the linearized constraint direction. Let
\(N_c=|I_{\mathrm{binding}}\cup I_{\mathrm{non\mbox{-}binding}}|\)
denote the number of active or potentially active constraints. We collect the
constraint directions into \(A\in\mathbb{R}^{N_c\times d}\), whose \(j\)-th row
is \(a_j^\top\), and collect the constant terms into
\(\eta\in\mathbb{R}^{N_c}\).

\paragraph{Primal--dual interior-point formulation.}
For notational simplicity, we first describe the inequality part of the system;
binding equality constraints are handled by the same KKT system without slack
variables. Introducing slack variables
\(\zeta\in\mathbb{R}^{N_c}\) and dual variables
\(\xi\in\mathbb{R}^{N_c}\), the inequality constraints are written as
\[
A\omega+\zeta+\eta=0,
\qquad
\zeta>0,
\qquad
\xi>0.
\]
With a logarithmic barrier parameter \(\rho>0\), the primal--dual residuals are
\begin{align*}
r_{\mathrm{dual}} &= H_{\theta\theta}\omega+b+A^\top \xi,\\
r_{\mathrm{pri}} &= A\omega+\zeta+\eta,\\
r_{\mathrm{cent}} &= \Gamma\Lambda\mathbf{1}-\rho\mathbf{1},
\end{align*}
where \(\Gamma=\mathrm{diag}(\zeta)\) and
\(\Lambda=\mathrm{diag}(\xi)\).

\paragraph{Newton system.}
Each primal--dual iteration computes a Newton direction
\((\Delta\omega,\Delta\xi,\Delta\zeta)\) by solving
\begin{equation}
\label{eq:app_kkt_full}
\begin{bmatrix}
H_{\theta\theta} & A^\top & 0 \\
A                & 0      & I \\
0                & \Gamma & \Lambda
\end{bmatrix}
\begin{bmatrix}
\Delta\omega\\
\Delta\xi\\
\Delta\zeta
\end{bmatrix}
=
-
\begin{bmatrix}
r_{\mathrm{dual}}\\
r_{\mathrm{pri}}\\
r_{\mathrm{cent}}
\end{bmatrix}.
\end{equation}
Although the full KKT matrix has dimension \((d+2N_c)\times(d+2N_c)\),
the slack block can be eliminated efficiently because both \(\Gamma\) and
\(\Lambda\) are diagonal.

From the second and third block equations,
\[
A\Delta\omega+\Delta\zeta=-r_{\mathrm{pri}},
\qquad
\Gamma\Delta\xi+\Lambda\Delta\zeta=-r_{\mathrm{cent}},
\]
we obtain
\[
\Delta\zeta=-r_{\mathrm{pri}}-A\Delta\omega,
\]
and
\[
\Delta\xi
=
\Gamma^{-1}
\left(
\Lambda A\Delta\omega
-r_{\mathrm{cent}}
+\Lambda r_{\mathrm{pri}}
\right).
\]
Substituting this expression into the first KKT block yields the Schur
complement system
\begin{equation}
\label{eq:app_schur}
M_{\mathrm{sc}}\Delta\omega=b_{\omega},
\end{equation}
where
\[
M_{\mathrm{sc}}
=
H_{\theta\theta}
+
A^\top\Gamma^{-1}\Lambda A,
\]
and
\[
b_{\omega}
=
-r_{\mathrm{dual}}
-
A^\top\Gamma^{-1}
\left(
-r_{\mathrm{cent}}
+
\Lambda r_{\mathrm{pri}}
\right).
\]
After solving~(\ref{eq:app_schur}), the dual and slack directions are recovered
by back-substitution:
\[
\Delta\xi
=
\Gamma^{-1}
\left(
\Lambda A\Delta\omega
-r_{\mathrm{cent}}
+\Lambda r_{\mathrm{pri}}
\right),
\qquad
\Delta\zeta
=
-r_{\mathrm{pri}}-A\Delta\omega.
\]

\paragraph{Matrix-free implementation.}
We do not explicitly form either the Hessian \(H_{\theta\theta}\) or the Schur
matrix \(M_{\mathrm{sc}}\). Instead, we solve~(\ref{eq:app_schur}) using a
matrix-free Krylov method such as conjugate gradient. Each matrix--vector
product with \(M_{\mathrm{sc}}\) requires
\[
M_{\mathrm{sc}}v
=
H_{\theta\theta}v
+
A^\top\Gamma^{-1}\Lambda Av.
\]
The first term is computed by a Hessian--vector product (HVP), while the second
term only requires products with \(A\), \(A^\top\), and diagonal matrices.

\paragraph{Computational complexity.}
Let \(C_{\mathrm{hvp}}\) denote the cost of one HVP, which is comparable to a
small number of backpropagation passes. Each Krylov iteration costs
\[
\mathcal{O}\big(C_{\mathrm{hvp}}+N_c d\big),
\]
where the \(N_c d\) term comes from multiplying by \(A\) and \(A^\top\).
If \(K_{\mathrm{cg}}\) Krylov iterations are used inside each Newton step and
\(K_{\mathrm{pd}}\) primal--dual iterations are performed, the total cost is
\[
\mathcal{O}
\left(
K_{\mathrm{pd}}K_{\mathrm{cg}}
\big(C_{\mathrm{hvp}}+N_c d\big)
\right).
\]
Since the number of active constraints \(N_c\) is typically much smaller than
the parameter dimension \(d\), the method scales linearly with the model size
up to a small constraint-dependent factor. Importantly, the QP solver does not
require forming the full Hessian or the full KKT matrix, making it suitable for
large neural networks.

\subsubsection{Scalable QP Solver with Structured Curvature Approximation}
\label{app:scalable-qp}

For large neural networks, solving the DIF QP with exact Hessian-vector products (HVP) can still be expensive. 
A scalable alternative is to replace \(H_{\theta\theta}\) with a structured curvature approximation, such as K-FAC or EK-FAC, which enables efficient approximate IHVPs. 
We refer readers to~\citep{martens2015optimizing, grosse2023studying} for details on K-FAC/EK-FAC-based IHVP acceleration.

For large neural networks, solving the DIF QP with exact Hessian-vector products can still be expensive. 
To improve scalability, we replace the exact Hessian \(H_{\theta\theta}\) in the QP solver with a structured curvature approximation. 
In particular, K-FAC and EK-FAC provide tractable layer-wise approximations to the Fisher/Gauss--Newton curvature and enable efficient iHVP. This yields a scalable approximate solver for the DIF QP. 

Recall that, after eliminating the slack variables, the primal Newton direction is obtained from
\begin{equation}
\left(
H_{\theta\theta}
+
A^\top D A
\right)\Delta\omega
=
b_\omega,
\label{eq:scalable-schur}
\end{equation}
where \(D=\Gamma^{-1}\Lambda\), and \(A\in\mathbb{R}^{N_c\times d}\) contains the linearized constraint directions. 
We approximate \(H_{\theta\theta}\) by a structured curvature matrix \(G_{\mathrm{kfac}}\). 
Here, \(G_{\mathrm{kfac}}\) denotes the K-FAC approximation to the curvature of the Lagrangian loss with respect to the model parameters. 
Specifically, K-FAC approximates each layer-wise curvature block by a Kronecker product of the input-activation covariance and the output-gradient covariance, enabling efficient inverse-vector products without forming the full Hessian.
This gives, giving
\begin{equation}
\left(
G_{\mathrm{kfac}}
+
A^\top D A
\right)\Delta\omega
=
b_\omega.
\label{eq:kfac-schur}
\end{equation}

Using the Woodbury identity, the solution can be written as
\begin{equation}
\Delta\omega
=
G_{\mathrm{kfac}}^{-1}b_\omega
-
G_{\mathrm{kfac}}^{-1}A^\top
\left(
D^{-1}
+
A G_{\mathrm{kfac}}^{-1}A^\top
\right)^{-1}
A G_{\mathrm{kfac}}^{-1}b_\omega .
\label{eq:kfac-woodbury}
\end{equation}
This form avoids inverting a \(d\times d\) matrix. 
Instead, it only requires applying \(G_{\mathrm{kfac}}^{-1}\) to vectors and solving an \(N_c\times N_c\) linear system. 
Since \(N_c\ll d\) in our constrained learning settings, this substantially reduces the computational cost.

The dominant cost is forming
\[
A G_{\mathrm{kfac}}^{-1}A^\top .
\]
Because \(A^\top\) has \(N_c\) columns, this requires applying the K-FAC inverse to each constraint direction. 
Thus, one Newton update costs approximately \(N_c\) applications of a K-FAC inverse-vector product, plus the cost of solving a small \(N_c\times N_c\) system. 
If \(K_{\mathrm{pd}}\) primal-dual iterations are used, the overall cost is
\[
\mathcal{O}\!\left(K_{\mathrm{pd}}N_c\,C_{\mathrm{kfac}}^{-1}\right),
\]
where \(C_{\mathrm{kfac}}^{-1}\) denotes the cost of one K-FAC inverse-vector product. 
This provides a scalable alternative to exact HVP-based QP solving while preserving the constraint-adjusted structure of the DIF update.

\subsection{More experiments}\label{moreexp}
\subsubsection{Machine unlearning of Physical informed Neural Network (PINN)} \label{AppendixPINN}

We further evaluate DIF on a machine unlearning task for a Physics-Informed 
Neural Network (PINN) model, following the formulation in~\citep{shi2021physics}. 
The PINN is trained on the NGSIM vehicle-trajectory dataset~\citep{coifman2017critical} 
to estimate traffic velocity from observed trajectories, with a PDE constraint to enforce that the solution respects the traffic flow theory.

Let $u_\theta(x,t)$ be the neural-network 
approximation of the velocity field. The PINN training problem can be written as \footnote{For simplicity of presentation in Appendix~\ref{AppendixPINN}, we reuse the notation 
$(x,t,v)$ to denote the spatial coordinate, timestamp, and observed velocity 
in the NGSIM data. These symbols are independent of their usage in the main text.
\[
\min_\theta 
\quad 
\frac{1}{N}\sum_{i=1}^N 
\big( u_\theta(x_i,t_i) - v_i \big)^2
\]
subject to the PDE constraint (traffic-flow model)
\[
\frac{\partial u_\theta}{\partial t}(x,t)
+
u_\theta(x,t)\frac{\partial u_\theta}{\partial x}(x,t)
=
0,
\quad \forall (x,t)\in\Omega,
\]
where $(x_i,t_i,v_i)$ are the observed NGSIM trajectory data and $\Omega$ is the 
spatio-temporal domain.}

Figure~\ref{fig:three_images} illustrates (i) the raw vehicle trajectories, 
(ii) the removal of a subset of the trajectory data to simulate unlearning, 
and (iii) the corresponding velocity field estimated by the PINN. 

After removing 10\% training data, we use DIF to estimate the change in the solution and obtain the updated model parameters. Table~\ref{tab:pinnresults} presents the performance of three PINN models--Original, Retrained, and Unlearned--on the residual dataset after data removal. The Retrained Model achieves the best performance across all metrics, with the lowest data loss (MAE: 2.495 ), physics loss (MAE: $1.12 \times 10^{-2}$ ), and relative $L_2$ error ( $14.24 \%$ ). The Unlearned Model shows slightly higher errors than the Retrained Model, but remains significantly better than the Original Model.

\begin{figure}[htbp]
    \centering
    \begin{subfigure}[t]{0.32\textwidth}
        \includegraphics[width=\linewidth]{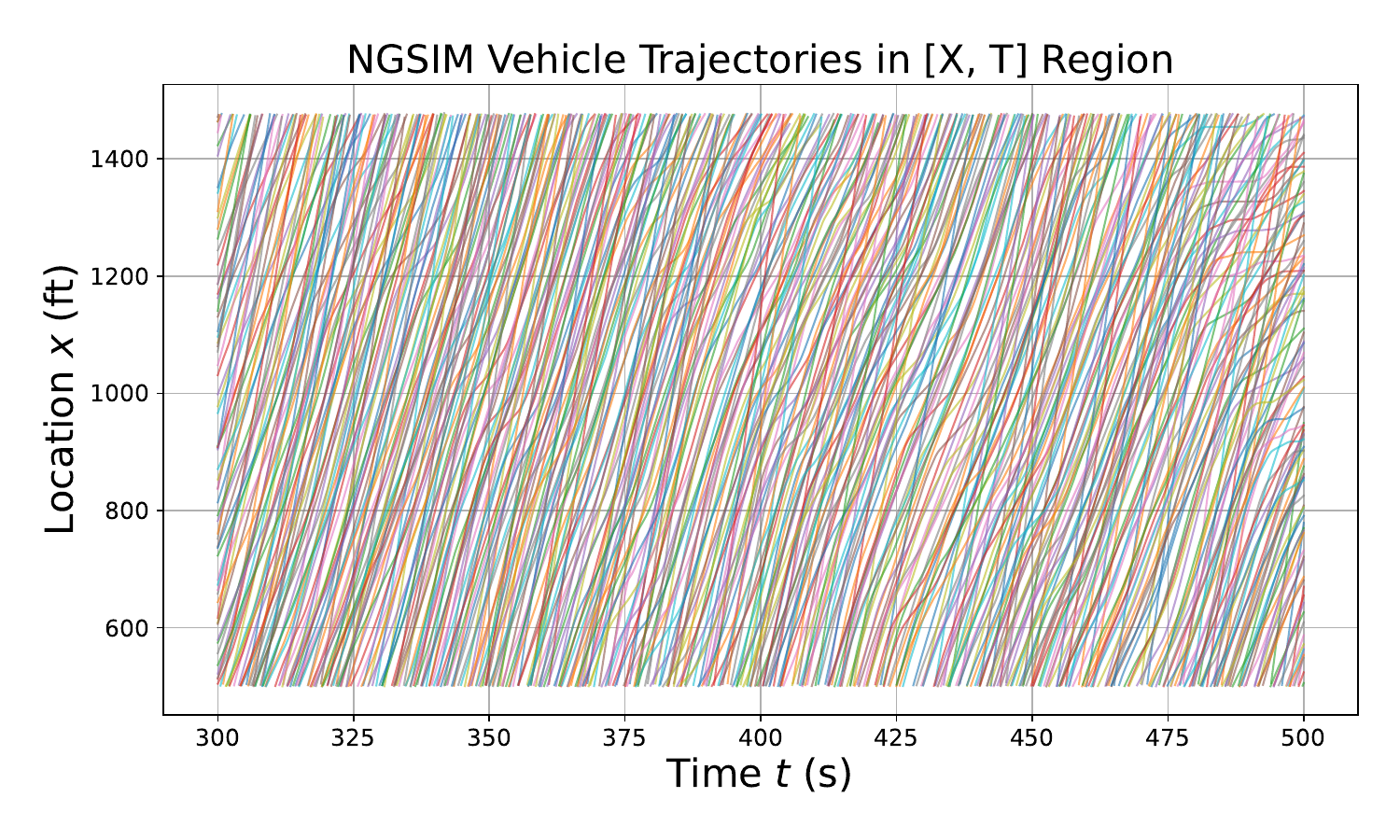}
        \caption{}
        \label{fig:sub1}
    \end{subfigure}
    \hfill
    \begin{subfigure}[t]{0.32\textwidth}
        \includegraphics[width=\linewidth]{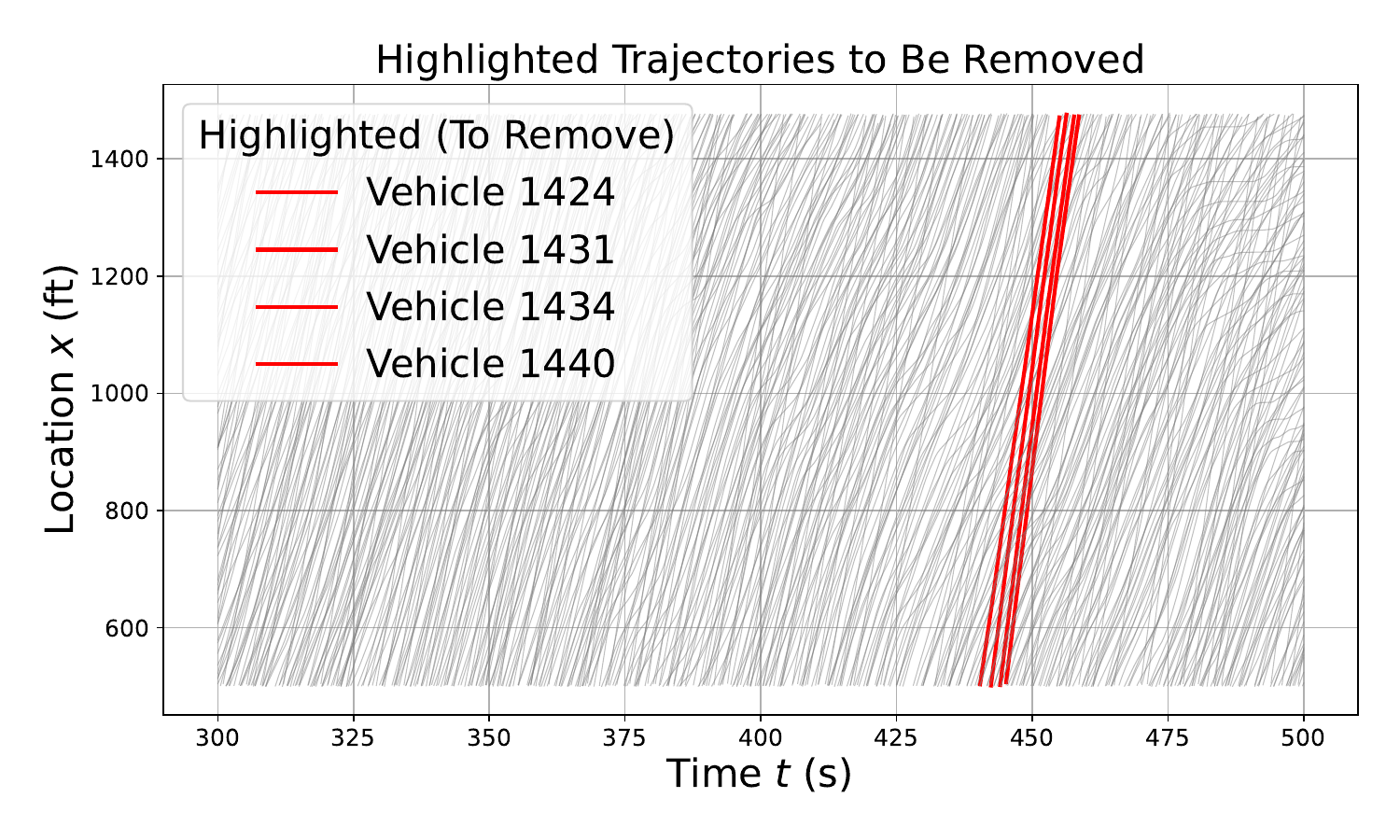}
        \caption{}
        \label{fig:sub2}
    \end{subfigure}
    \hfill
    \begin{subfigure}[t]{0.32\textwidth}
        \includegraphics[width=\linewidth]{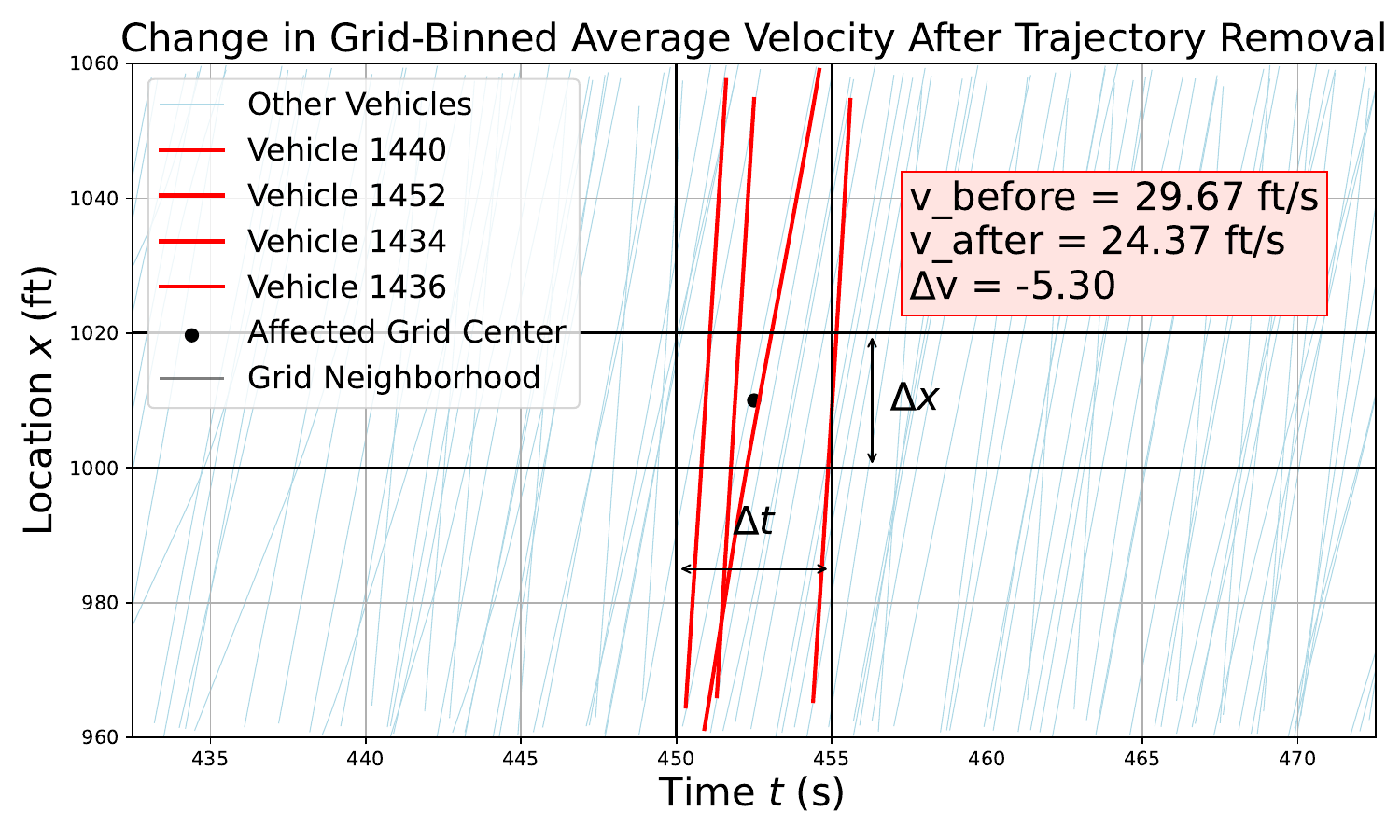}
        \caption{ }
        \label{fig:sub3}
    \end{subfigure}
    \caption{{Illustration of trajectory removal and its impact on observed velocity.  (a) All trajectories. (b) Subset of removed trajectories. (c) Change in average velocity in affected spatiotemporal bins due to trajectory removal.}}
    \label{fig:three_images}
\end{figure}

\begin{table}[t]
\centering
\caption{{Performance comparison of different PINN models on the residual dataset after removal.}}
\label{tab:pinnresults}

\begin{tabular}{lcccc}
\toprule
\textbf{Model} & 
\makecell{\textbf{MAE} \\ \textbf{(Data Loss)}} &
\makecell{\textbf{MAE} \\ \textbf{(Physics Loss)}} &
\textbf{Relative $L_2$ Error (\%)} &
\textbf{Training Time (s)} \\
\midrule
Original Model    & 4.156 & $47.1\times 10^{-2}$ & 21.65 & 327.06 \\
Retrained Model   & 2.495 & $1.12\times 10^{-2}$ & 14.24 & 331.50 \\
Unlearned Model   & 2.832 & $0.37\times 10^{-2}$ & 16.68 & 91.77 \\
\bottomrule
\end{tabular}

\end{table}

\subsubsection{Additional Large-Scale Validation on FairFace}
\label{app:fairface_large_scale}

We further evaluate DIF on a large-scale fairness-constrained learning problem.
Specifically, we consider gender prediction on the FairFace dataset~\citep{karkkainen2021fairface}, which contains 108,501 face images from seven racial groups.
We train a fairness-constrained CNN based on ResNet-50 and impose constraints to reduce risk disparity across racial groups.

We compare DIF with TRAK, TracIn, and random selection.
Since TRAK and TracIn are originally designed for unconstrained learning, we adapt them to the constrained setting using the same counterfactual retraining protocol as in Section~\ref{DIFviladation}.
For each method, we identify the top 10\% most harmful training samples, remove them from the training set, and retrain the model.
We then measure the loss change on a fixed misclassified test example after retraining.
A larger loss reduction indicates that the selected training samples have stronger harmful influence on the target test example.

\begin{table}[t]
\centering
\caption{
Loss change after removing the top 10\% harmful training samples on the FairFace fairness-constrained ResNet-50 experiment.
DIF identifies samples whose removal leads to the largest loss reduction on the target misclassified test example.
}
\label{tab:fairface_large_scale}
\begin{tabular}{lcccc}
\toprule
Method & DIF & TRAK & TracIn & Random \\
\midrule
Loss change & \textbf{0.004427} & 0.0037419 & 0.003235 & 0.0006043 \\
\bottomrule
\end{tabular}
\end{table}

As shown in Table~\ref{tab:fairface_large_scale}, DIF achieves the largest loss change after removing the selected harmful samples.
This result suggests that DIF better captures the interaction between data influence and constraint satisfaction.
In contrast, influence methods designed for unconstrained learning may overlook how a training sample affects the constrained solution through both the objective and the active fairness constraints, leading to less effective identification of harmful samples.

\subsection{Details of Counterfactual Retraining Evaluation}
\label{app:counterfactual-details}

We compare DIF with constrained TracIn and TRAK in the counterfactual retraining experiment. 
For each method, we compute an influence score for every training sample, remove the samples ranked as most harmful, and retrain the constrained CNN from the same training protocol. 
We evaluate the resulting loss change on the fixed misclassified test point used in the main experiment.

\paragraph{Constrained TracIn.}
Standard TracIn estimates the influence of a training point on a test point by accumulating gradient inner products along the training trajectory. 
In the constrained setting, we adapt TracIn by replacing the standard prediction loss with the Lagrangian loss used in the primal-dual training procedure. 
Specifically, during primal-dual optimization, we save checkpoints 
$\{\theta_{t_1},\theta_{t_2},\ldots,\theta_{t_k}\}$ 
and compute
\begin{equation}
\operatorname{TracInCP}(x_{\mathrm{train}},x_{\mathrm{test}})
=
\sum_{i=1}^{k}
\delta_i
\nabla_\theta \ell_{\mathrm{lag}}(\theta_{t_i},x_{\mathrm{train}})
\cdot
\nabla_\theta \ell_{\mathrm{lag}}(\theta_{t_i},x_{\mathrm{test}}),
\end{equation}
where $\ell_{\mathrm{lag}}$ denotes the Lagrangian loss, $\theta_{t_i}$ is the model checkpoint at iteration $t_i$, and $\delta_i$ is the step size between checkpoints $i-1$ and $i$. 
This modification incorporates the constraint-adjusted training direction into the TracIn score.

\paragraph{TRAK.}
For TRAK, we apply the standard TRAK procedure to a penalized objective, where the fairness constraints are incorporated through a squared-hinge penalty. 
This gives TRAK access to constraint information through the surrogate loss, but it does not explicitly model the KKT-based sensitivity structure.

\paragraph{Removal protocol.}
For DIF, constrained TracIn, and TRAK, we rank training samples by their estimated harmfulness to the fixed misclassified test point. 
We then remove the top \textbf{2\%} and \textbf{10\%} most harmful samples and retrain the constrained CNN. 
Random removal is included as a reference baseline. 
The absolute loss reduction after retraining is reported in Figure~\ref{fig:conter}.

\paragraph{Compute resources and societal impact.}
All experiments were conducted on a workstation equipped with an NVIDIA RTX 4090 GPU with 24GB memory. 
The constrained linear regression experiments were run on CPU, while the CNN, PINN, and FairFace experiments were run on GPU. 
DIF may improve interpretability, data debugging, and reliability analysis in constrained learning, especially when constraint violations are costly. 
DIF itself does not introduce additional high-risk capabilities beyond standard data-attribution methods. 
Its use should follow applicable privacy, data-governance requirements, and dataset terms of use.

\clearpage
\newpage
\section*{NeurIPS Paper Checklist}

\begin{enumerate}

\item {\bf Claims}
    \item[] Question: Do the main claims made in the abstract and introduction accurately reflect the paper's contributions and scope?
    \item[] Answer: \answerYes{} % Replace by \answerYes{}, \answerNo{}, or \answerNA{}.
    \item[] Justification:The abstract and introduction accurately summarize the main contributions, including the DIF formulation, the VI-based sensitivity analysis, and the empirical validation.

\item {\bf Limitations}
    \item[] Question: Does the paper discuss the limitations of the work performed by the authors?
    \item[] Answer: \answerYes{} % Replace by \answerYes{}, \answerNo{}, or \answerNA{}.
    \item[] Justification: The final section discusses DIF's limitations, including its local nature and the Hessian-related computational costs inherited from classical influence functions.

\item {\bf Theory assumptions and proofs}
    \item[] Question: For each theoretical result, does the paper provide the full set of assumptions and a complete (and correct) proof?
    \item[] Answer: \answerYes{} % Replace by \answerYes{}, \answerNo{}, or \answerNA{}.
    \item[] Justification: The main text explicitly states the regularity assumptions, including LICQ and SOSC, and the appendix provides complete proofs of all theoretical results.

    \item {\bf Experimental result reproducibility}
    \item[] Question: Does the paper fully disclose all the information needed to reproduce the main experimental results of the paper to the extent that it affects the main claims and/or conclusions of the paper (regardless of whether the code and data are provided or not)?
    \item[] Answer: \answerYes{} % Replace by \answerYes{}, \answerNo{}, or \answerNA{}.
    \item[] Justification: The appendix reports the hyperparameter choices, implementation details, and additional experimental settings needed to reproduce our results.

\item {\bf Open access to data and code}
    \item[] Question: Does the paper provide open access to the data and code, with sufficient instructions to faithfully reproduce the main experimental results, as described in supplemental material?
    \item[] Answer: \answerNo{} % Replace by \answerYes{}, \answerNo{}, or \answerNA{}.
    \item[] Justification: The paper uses only publicly available datasets. The code is not publicly released at the time of submission, but we plan to release the implementation of DIF upon camera-ready acceptance.

\item {\bf Experimental setting/details}
    \item[] Question: Does the paper specify all the training and test details (e.g., data splits, hyperparameters, how they were chosen, type of optimizer) necessary to understand the results?
    \item[] Answer: \answerYes{} % Replace by \answerYes{}, \answerNo{}, or \answerNA{}.
    \item[] Justification: The main text describes the experimental setup, and the appendix provides detailed DIF implementation settings, baselines' hyperparameter choices, and additional experimental details about influence functions.

\item {\bf Experiment statistical significance}
    \item[] Question: Does the paper report error bars suitably and correctly defined or other appropriate information about the statistical significance of the experiments?
    \item[] Answer: \answerYes{} % Replace by \answerYes{}, \answerNo{}, or \answerNA{}.
    \item[] Justification: Our experimental results report variability across repeated trials, including mean and standard deviation for counterfactual retraining experiments.

\item {\bf Experiments compute resources}
    \item[] Question: For each experiment, does the paper provide sufficient information on the computer resources (type of compute workers, memory, time of execution) needed to reproduce the experiments?
    \item[] Answer: \answerYes{} % Replace by \answerYes{}, \answerNo{}, or \answerNA{}.
    \item[] Justification: All experiments were run on machines equipped with an NVIDIA RTX 4090 GPU. 
    
\item {\bf Code of ethics}
    \item[] Question: Does the research conducted in the paper conform, in every respect, with the NeurIPS Code of Ethics \url{https://neurips.cc/public/EthicsGuidelines}?
    \item[] Answer: \answerYes{} % Replace by \answerYes{}, \answerNo{}, or \answerNA{}.
    \item[] Justification: {The research conforms to the NeurIPS Code of Ethics. It uses publicly available datasets and focuses on constrained data attribution without involving human-subject experiments, private data, or harmful deployment.}

\item {\bf Broader impacts}
    \item[] Question: Does the paper discuss both potential positive societal impacts and negative societal impacts of the work performed?
    \item[] Answer: \answerYes{} % Replace by \answerYes{}, \answerNo{}, or \answerNA{}.
    \item[] Justification: The paper discusses positive impacts, such as improving interpretability and reliability in constrained learning, as well as potential negative impacts, such as misuse for identifying sensitive or strategically important training samples.
    
\item {\bf Safeguards}
    \item[] Question: Does the paper describe safeguards that have been put in place for responsible release of data or models that have a high risk for misuse (e.g., pre-trained language models, image generators, or scraped datasets)?
    \item[] Answer: \answerNA{} % Replace by \answerYes{}, \answerNo{}, or \answerNA{}.
    \item[] Justification: The paper does not release high-risk datasets or models, such as scraped datasets, pre-trained language models, or image generators.

\item {\bf Licenses for existing assets}
    \item[] Question: Are the creators or original owners of assets (e.g., code, data, models), used in the paper, properly credited and are the license and terms of use explicitly mentioned and properly respected?
    \item[] Answer: \answerYes{} % Replace by \answerYes{}, \answerNo{}, or \answerNA{}.
    \item[] Justification: The original sources of all datasets, models, and code assets used in the paper are properly credited, and their licenses and terms of use are respected.

\item {\bf New assets}
    \item[] Question: Are new assets introduced in the paper well documented and is the documentation provided alongside the assets?
    \item[] Answer: \answerNA{} % Replace by \answerYes{}, \answerNo{}, or \answerNA{}.
    \item[] Justification: The paper does not introduce or release new assets such as datasets, model checkpoints, benchmarks, or software packages.

\item {\bf Crowdsourcing and research with human subjects}
    \item[] Question: For crowdsourcing experiments and research with human subjects, does the paper include the full text of instructions given to participants and screenshots, if applicable, as well as details about compensation (if any)? 
    \item[] Answer: \answerNA{} % Replace by \answerYes{}, \answerNo{}, or \answerNA{}.
    \item[] Justification: The paper does not involve crowdsourcing experiments or research with human subjects.

\item {\bf Institutional review board (IRB) approvals or equivalent for research with human subjects}
    \item[] Question: Does the paper describe potential risks incurred by study participants, whether such risks were disclosed to the subjects, and whether Institutional Review Board (IRB) approvals (or an equivalent approval/review based on the requirements of your country or institution) were obtained?
    \item[] Answer: \answerNA{} % Replace by \answerYes{}, \answerNo{}, or \answerNA{}.
    \item[] Justification: The paper does not involve study participants or human-subject research, so participant risks, disclosure, and IRB approval are not applicable.

\item {\bf Declaration of LLM usage}
    \item[] Question: Does the paper describe the usage of LLMs if it is an important, original, or non-standard component of the core methods in this research? Note that if the LLM is used only for writing, editing, or formatting purposes and does \emph{not} impact the core methodology, scientific rigor, or originality of the research, declaration is not required.
    %this research? 
    \item[] Answer: \answerNA{} % Replace by \answerYes{}, \answerNo{}, or \answerNA{}.
    \item[] Justification: LLMs were not used as an important, original, or non-standard component of the core methodology, experiments, or analysis. Any use of LLMs was limited to writing, editing, or formatting assistance and did not affect the scientific content.

\end{enumerate}

\end{document}